\documentclass{article}

\usepackage{microtype}
\usepackage{wrapfig}
\usepackage{graphicx}
\usepackage{subcaption}
\usepackage{booktabs} 

\usepackage{hyperref}
\usepackage{tcolorbox}

\usepackage[preprint]{icml2026}

\usepackage{amsmath}
\usepackage{amssymb}
\usepackage{mathtools}
\usepackage{amsthm}

\usepackage[capitalize,noabbrev]{cleveref}

\theoremstyle{plain}
\newtheorem{theorem}{Theorem}[section]

\newtheorem{lemma}[theorem]{Lemma}

\theoremstyle{definition}

\newtheorem{assumption}[theorem]{Assumption}
\theoremstyle{remark}
\newtheorem{remark}[theorem]{Remark}

\DeclareMathOperator*{\argmin}{argmin}

\usepackage[textsize=tiny]{todonotes}

\usepackage{xspace}
\newcommand{\name}{PD-VI\xspace}
\newcommand{\pname}{P$^2$D-VI\xspace}

\newcommand{\vecx}{\boldsymbol{x}}

\newcommand{\vecz}{\boldsymbol{z}}
\newcommand{\vecbeta}{\boldsymbol{\beta}}

\newcommand{\veczeta}{\boldsymbol{\zeta}}

\newcommand{\matA}{\mathbf{A}}

\icmltitlerunning{Scalable Mean-Field Variational Inference}

\begin{document}

\twocolumn[
  \icmltitle{Scalable Mean-Field Variational Inference via Preconditioned Primal–Dual Optimization}
  


  \icmlsetsymbol{equal}{*}

  \begin{icmlauthorlist}
    \icmlauthor{Jinhua Lyu}{iems}
    \icmlauthor{Tianmin Yu}{math}
    \icmlauthor{Ying Ma}{biology}
    \icmlauthor{Naichen Shi}{iems}
  \end{icmlauthorlist}

  \icmlaffiliation{iems}{Department of Industrial Engineering and Management Science, Northwestern University}
  \icmlaffiliation{math}{Department of Mathematics, Northwestern University}
  \icmlaffiliation{biology}{Department of Biostatistics, Brown University}
  
  \icmlcorrespondingauthor{Ying Ma}{ying\_ma@brown.edu}
  \icmlcorrespondingauthor{Naichen Shi}{naichen.shi@northwestern.edu}

  \icmlkeywords{Machine Learning, ICML}

  \vskip 0.3in
]



\printAffiliationsAndNotice{}  

\begin{abstract}
In this work, we investigate the large-scale mean-field variational inference (MFVI) problem from a mini-batch primal-dual perspective. By reformulating MFVI as a constrained finite-sum problem, we develop a novel primal--dual algorithm based on an augmented Lagrangian formulation, termed primal--dual variational inference (PD-VI). PD-VI jointly updates global and local variational parameters in the evidence lower bound in a scalable manner. To further account for heterogeneous loss geometry across different variational parameter blocks, we introduce a block-preconditioned extension, P$^2$D-VI, which adapts the primal–dual updates to the geometry of each parameter block and improves both numerical robustness and practical efficiency. We establish convergence guarantees for both PD-VI and P$^2$D-VI under properly chosen constant step size, without relying on conjugacy assumptions or explicit bounded-variance conditions. In particular, we prove $\mathcal{O}(1/T)$ convergence to a stationary point in general settings and linear convergence under strong convexity. Numerical experiments on synthetic data and a real large-scale spatial transcriptomics dataset demonstrate that our methods consistently outperform existing stochastic variational inference approaches in terms of convergence speed and solution quality.
\end{abstract}

\section{Introduction}
A fundamental problem in Bayesian inference is to characterize the posterior distribution of latent variables $\vecz$ given observations $\vecx$, i.e., $p(\vecz \mid \vecx)$. 
In many statistical applications, this posterior is \emph{non-conjugate}, rendering exact inference intractable. Furthermore, $\vecz$ is often \emph{high-dimensional}, creating additional challenges for inference. Variational inference (VI)~\cite{blei2017variational, zhang2018advances} has emerged as an alternative by reframing Bayesian inference as an optimization problem: approximating the posterior with a tractable variational distribution $q(\vecz)$ that minimizes the Kullback--Leibler divergence to $p(\vecz \mid \vecx)$. 
By leveraging optimization rather than sampling, VI often achieves orders-of-magnitude speedups over Markov chain Monte Carlo (MCMC)-based methods~\cite{hastings1970monte, geman1984stochastic,robert1999monte}, and has seen widespread success in applications such as Bayesian image segmentation~\cite{orbanz2008nonparametric, gao2023bayeseg}, hidden Markov model decoding~\cite{eddy1996hidden, mor2021systematic}, and visual embedding tokenization~\citep{roy2018theory}.

Among the many variational families, \emph{mean-field} (MF) variational inference remains one of the most widely adopted due to its simplicity and computational efficiency~\cite{cohn2010mean, tolle2021mean}. 
By assuming independence among latent variables, MF variational distributions factorize the joint posterior, enabling efficient optimization through coordinate ascent~\citep{blei2017variational}, gradient flow~\citep{yao2022mean}, or proximal-gradient methods~\citep{khan2015faster,baque2016principled,JMLR:v26:23-0573}. 
These algorithms are effective in moderate-scale settings where full-batch optimization is feasible.

Despite its popularity, two fundamental challenges arise when MFVI is deployed in modern large-scale Bayesian inference problems. 
The first challenge is \textit{scalability}. In many contemporary Bayesian models, latent variables are tied to individual observations, so the number of latent variables grows linearly with the dataset size. As a result, standard alternating minimization or full-batch gradient methods become impractical. 
A representative example is \emph{spatial transcriptomics}~\cite{zhao2021spatial,ma2024accurate}, where state-of-the-art datasets routinely contain hundreds of thousands of spatial locations, each associated with its own latent variables. Moreover, modeling spatial dependencies is critical in this setting, further increasing the computational and optimization complexity of MFVI. 
To address this issue, a number of large-scale MFVI algorithms have been proposed, among which stochastic variational inference (SVI)~\cite{hoffman2013stochastic} and its extensions~\cite{pmlr-v15-wang11a, hoffman2015structured, ranganath2013adaptive, khan2015faster, amari1998natural, martens2020new} are the most widely used. 
Despite their scalability, these methods suffer from inherent limitations: they often require diminishing step sizes to ensure convergence to stationary points, and they can become inefficient when mini-batches are non-i.i.d.
This limitation is illustrated in the middle panel of Figure~\ref{fig:distribution}, where SVI with a diminishing step size is applied to a Gaussian mixture model using poorly mixed (non-i.i.d.) mini-batches. The resulting reconstructed probability density deviates substantially from the ground truth, illustrating the suboptimal optimization performance. 

\begin{figure}[ht]
  \begin{center}
    \centerline{\includegraphics[width=\columnwidth]{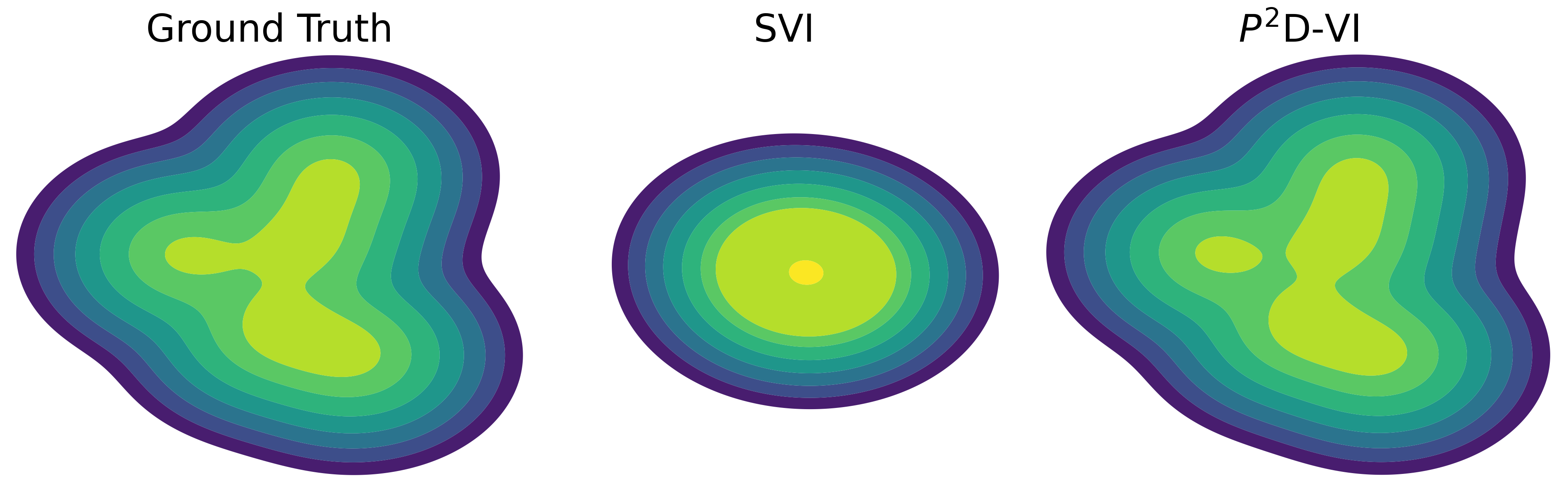}}
    \caption{Ground truth probability density of a Gaussian mixture model, and the probability density reconstructed by SVI with diminishing stepsizes, and by \pname (our algorithm). }
    \label{fig:distribution}
  \end{center}
  \vspace{-0.9cm}
\end{figure}

The second challenge is \emph{parameter heterogeneity}.
As models grow more complex, parameter blocks often operate at different scales and exhibit distinct curvatures.
Treating them uniformly can result in poor conditioning, slow convergence, and numerical instability, especially in large-scale variational inference.




To address the first challenge, we adopt a primal–dual perspective and reformulate mini-batch MFVI as a constrained finite-sum optimization problem.
This leads to a mini-batch primal–dual algorithm, termed primal–dual variational inference (\name).
Instead of relying on diminishing step size, \name implicitly adapts the effective constraint enforcement
through dual-variable updates, ensuring stable optimization under a constant step size.

 To tackle the second challenge of \textit{parameter heterogeneity}, we further introduce a block-preconditioned augmented Lagrangian that rescales updates for each parameter block.
This leads to preconditioned primal-dual variational inference (\pname), a variant that allows different parameter blocks to update at different scales, thus enabling faster and more robust numerical optimization. 

From a theoretical perspective, we show that \name and \pname converge to a stationary point under properly chosen constant step size in mini-batch settings, without requiring conjugacy assumptions or bounded gradient variance. 
To our knowledge, this is \textbf{the first primal--dual framework for large-scale MFVI} that handles global and local variables simultaneously.

Our main contributions are summarized as follows:
\begin{itemize}
\item We propose \name, a mini-batch primal--dual optimization framework for large-scale mean-field variational inference that jointly handles global and local latent variables. We further develop a preconditioned variant \pname that improves numerical efficiency and stability under parameter heterogeneity.
\item We establish convergence guarantees with \emph{constant step size}, proving $\mathcal{O}(1/T)$ convergence rates in nonconvex settings, and linear convergence in the strongly convex case without assuming conjugacy or bounded gradient variance. Technically, we introduce a novel error bound to handle the simultaneous update of global and local variables. 
\item We validate the proposed methods on a synthetic Gaussian mixture model
and on a non-conjugate real-world spatial transcriptomics dataset, demonstrating consistent improvements in both convergence speed and empirical performance over existing stochastic variational inference methods.
\end{itemize}

\begin{table*}[t]
\centering
\caption{Comparison of representative algorithms for VI. The first three columns show the convergence rate under the corresponding assumptions. 
Here, $\sigma^2$ denotes the variance of the stochastic gradient, and $\rho$ and $r$ are problem-dependent constants.The ``Local variables'' column indicates whether the algorithm can explicitly work with models with local latent variables associated with each sample. The “Constant step size” column indicates whether the method involves a constant step size.}
\label{tab:lmo_methods}
\small
\setlength{\tabcolsep}{8pt}
\renewcommand{\arraystretch}{1.2}
\begin{tabular}{p{4cm}  p{1.6cm} p{1.6cm} p{2.6cm} p{1.2cm}  p{1.2cm} }
\hline
Method  & General & Convex & Strongly convex & Local variables  &  Constant step size \\
\hline

SVI~\cite{hoffman2013stochastic}

& $O(\frac{1}{\sqrt{T}})$
& -
& -
& \checkmark
& $\times$ \\

PG-SVI~\cite{khan2015faster}

& $O(\frac{1}{T}+\sigma^2)$
& -
& -
& $\times$
& \checkmark\\
SNGVI~\cite{wu2024understanding}

& $O(\frac{1}{T})$
& -
& -
& $\times$
& $\times$ \\
CV~\citep{wang2024joint} & - & - & - & $\times$ & $\times$\\
Adam~\citep{zhang2022adam} & $O\left(\frac{\log T}{\sqrt{T}}\right)$ & - & - &\checkmark & $\times$\\
SGD~\cite{garrigos2023handbook}

& $O(\frac{1}{\sqrt{T}} + \sigma^2)$
& $O(\frac{1}{\sqrt{T}} + \sigma^2)$
& $O(\rho^T + \sigma^2)$, $\rho <1$
&\checkmark
& \checkmark \\

\name and \pname (Ours)

& $O(\frac{1}{T})$
& $O(\frac{1}{T})$
& $O(\frac{1}{r^{T}})$, $r>1$
&\checkmark
& \checkmark \\

\hline
\end{tabular}
\vspace{-0.5cm}
\end{table*}

\section{Related Works} \label{sec:related_work}

\paragraph{Variational inference} Variational inference~\cite{blei2017variational} is a widely used approach for approximating the posterior distribution of a Bayesian model by optimizing over a tractable family of distributions. A popular algorithm in VI is coordinate ascent variational inference (CAVI)~\cite{bhattacharya2023convergence, wang2025bayesian}, which iteratively updates each variational factor while holding the others fixed, yielding an elegant optimization algorithm with robust convergence performance~\cite{blei2017variational, wang2025bayesian}. However, CAVI must iterate over all data points sequentially each round, which becomes prohibitively slow on large datasets~\cite{hoffman2013stochastic}.

\textbf{Stochastic Variational Inference.}
To improve computational efficiency, extensive work has focused on stochastic or mini-batch VI. 
SVI~\cite{hoffman2013stochastic} is a seminal method that scales VI to massive datasets by combining stochastic optimization~\cite{robbins1951stochastic} with natural-gradient updates~\cite{amari1998natural} in conditionally conjugate models. 
In SVI, updates of the global variational parameters are equivalent to performing stochastic natural gradient descent, which replaces the Euclidean gradient with the natural gradient defined by the Fisher information geometry~\cite{amari1998natural, martens2020new}. 
Under standard Robbins--Monro conditions on the stepsize~\cite{robbins1951stochastic}, SVI is guaranteed to converge to a stationary point of the evidence lower bound (ELBO).

Several extensions have been proposed to improve the convergence behavior of stochastic VI, including parameter reweighting~\citep{dhaka2020robust}, alternative divergence measures~\citep{dhaka2021challenges,cai2024batch}, and adaptive or increasing mini-batch sizes~\citep{dinkel2025dynamic}. 
From a theoretical perspective, the best-known convergence rate for SVI-type algorithms is $\mathcal{O}(1/\sqrt{T})$, established for an asynchronous parallel variant of SVI~\cite{mohamad2019asynchronous}. 
A closely related method is stochastic natural gradient variational inference (SNGVI)~\cite{wu2024understanding}, which focuses exclusively on optimizing global variational parameters and does not explicitly introduce local variational variables. 
Despite these advances, both SVI and SNGVI typically rely on diminishing stepsizes to ensure convergence to stationary points. 
While this assumption is standard in stochastic optimization, diminishing stepsizes often lead to slow progress in the later stages of training and make performance sensitive to stepsize schedules, which can be difficult to tune and unstable in large-scale practice.

Another closely related approach is proximal-gradient SVI (PG-SVI)~\cite{khan2015faster}, which generalizes SVI and mirror-descent-based VI updates to handle non-conjugate models. 
By incorporating proximal operators, PG-SVI admits closed-form updates for a broader class of models than vanilla SVI. 
However, its theoretical guarantees establish $\mathcal{O}(1/T)$ convergence only to a neighborhood of the optimum, where the radius of this neighborhood depends explicitly on the variance of the mini-batch gradient estimator.

\textbf{Primal-dual algorithms}
Our proposed \name and \pname leverage primal-dual algorithms to solve finite-sum problems. Former algorithms such as~\citet{zhang2021fedpd, acar2021federated} share similar primal–dual formulations. However, in contrast to~\citet{zhang2021fedpd}, we develop a mini-batch variant that updates the model using only a randomly sampled subset at each iteration, thereby avoiding a full pass over all data points before performing a single global update. Compared with~\citet{acar2021federated}, our formulation explicitly introduces per-sample (or per-client) local variables, leading to a consensus-constrained problem with fundamentally different structure and update dynamics. Another line of work proposes control variates (CV) to reduce gradient variance in SVI~\citep{geffner2018using, geffner2020approximation, wang2024joint}.
Although control variates share conceptual similarities with dual variables, these methods do not explicitly incorporate local latent variables and generally do not provide convergence guarantees.
Our empirical comparisons indicate that \name and \pname achieve more stable convergence than CV-based approaches in large-scale settings.

\textbf{Preconditioned algorithms} 
In light of parameter heterogeneity, \pname employs preconditioning in the design of augmented Lagrangian. Preconditioning is a principled technique for accelerating first-order methods by reshaping the optimization geometry to mitigate ill-conditioning. 
A wide range of optimization algorithms incorporate preconditioning~\citep{davidon1991variable, liu1989limited,zhangadam}. In addition, adaptive gradient methods such as AdaGrad~\cite{duchi2011adaptive}, RMSProp~\cite{tieleman2012lecture}, and Adam~\cite{kingma2014adam} can also be interpreted as preconditioned stochastic gradient methods~\cite{ye2024preconditioning, scott2025designing}. 
A particularly practical and scalable strategy is diagonal preconditioning, which rescales decision variables using inexpensive coordinate-wise metrics. The work~\cite{pock2011diagonal} employs diagonal preconditioning for first-order primal–dual methods in convex saddle-point (min–max) optimization. Our work is motivated by the same principle that appropriate preconditioning can improve numerical stability and convergence speed, but we focus on designing and analyzing preconditioned updates for a mini-batch VI problem. To our best knowledge, \pname is the first algorithm that introduces block-diagonal preconditioning in the augmented Lagrangian construction for mini-batch MFVI. A detailed comparison of VI algorithms is presented in Table~\ref{tab:lmo_methods}.

\section{Methods}
\subsection{Notations}
Throughout the paper, we use $\|\cdot\|$ and $\|\cdot\|_2$ interchangeably to denote the Euclidean (i.e., $\ell_2$) norm and $\|\mathbf{v}\|_{\mathbf{A}}$ to define $\matA$-norm for positive semi-definite $\matA$ as $\|\mathbf{v}\|_{\mathbf{A}} = \langle\mathbf{v}, \mathbf{A}\mathbf{v} \rangle $.  Moreover, for any integer $n \ge 1$, we write $[n] := \{1,2,\ldots,n\}$. We use boldface letters (e.g., $\boldsymbol{x}$ and $\boldsymbol{A}$) to denote vectors and matrices, respectively, and use the corresponding non-bold letters (e.g., $\eta$) to denote scalars. For a distribution $p(\cdot)$, we use subscripts to indicate its parameters; for instance, $p_{\theta}(\cdot)$ denotes a distribution parameterized by $\theta$.
\subsection{Preliminary}
\label{sec:preliminary}
Let $\mathbf{x}_{1:n} := \{\mathbf{x}_1,\ldots,\mathbf{x}_n\}$ denote a collection of observations, where each $\mathbf{x}_i \in \mathbb{R}^{d_x}$ for all $i \in [n]$. We use $\veczeta_{1:n} := \{\veczeta_1,\ldots,\veczeta_n\}$ denote a collection of local latent random variables associated with each observations (e.g., assignment variables in Gaussian misture models), where each $\veczeta_i \in \mathbb{R}^{d_\zeta}$ for all $i \in [n]$. Let $\boldsymbol{\beta} \in \mathbb{R}^{d_\beta}$ denote a set of global latent random variables shared across all observations. For notational simplicity, we use $\mathbf{x}$ and $\veczeta$ to denote $\mathbf{x}_{1:n}$ and $\veczeta_{1:n}$, respectively, when no ambiguity arises. The latent variable $\vecz$ consists of both $\veczeta$ and $\vecbeta$, $\vecz=[\veczeta,\vecbeta]$.

In many hierarchical Bayesian models, the  latent variables $\veczeta_i$, and $\vecbeta$ are drawn from a prior distribution $p(\veczeta_i, \boldsymbol{\beta})$. These latent variables help govern the distribution of the data, specifically, each observation is assumed to follow the conditional distribution $\mathbf{x}_i \sim p(\cdot \mid \veczeta_i, \vecbeta)$. Conditioned on \(\{\veczeta_i\}_{i=1}^n\) and \(\vecbeta\), the observations $\{\mathbf{x}_i\}_{i=1}^n$ are assumed to be conditionally independent, i.e.,
\(
p(\mathbf{x}_{1:n}\mid \veczeta_{1:n},\vecbeta)=\prod_{i=1}^n p(\mathbf{x}_i\mid \veczeta_i,\vecbeta).
\)


We are interested in maximum a posteriori (MAP) estimation, which requires computing the posterior distribution $
p(\veczeta, \boldsymbol{\beta} \mid \mathbf{x})
=\frac{
p(\mathbf{x} \mid \veczeta, \boldsymbol{\beta}) \,
p(\veczeta, \boldsymbol{\beta})}{p(\mathbf{x})}$, 
where the denominator \(p(\mathbf{x})\) is the marginal likelihood (also known as the evidence). Computing this quantity requires marginalizing over all latent variables and is generally intractable for many statistical models~\citep{song2022learning}.

In block MFVI, we consider a variational family $\mathcal{Q}$ that factorizes across blocks of latent variables.
Specifically, we assume that the local latent variables are independent across data points, with each local variational factor $q(\veczeta_i)$ parameterized by a set of variational parameters $\boldsymbol{\phi}_i$.
The global variational factor $q(\boldsymbol{\beta})$ is parameterized by a set of variational parameters $\boldsymbol{\lambda}$. Additional factorization or independence assumptions among the components of $\boldsymbol{\beta}$ can be imposed and analyzed case by case.
Accordingly, the variational distribution takes the form
\[
q_{\boldsymbol{\phi},\boldsymbol{\lambda} }(\veczeta, \boldsymbol{\beta})
=
\left(\prod_{i=1}^n q_{\boldsymbol{\phi}_i}(\veczeta_i)\right)
q_{\boldsymbol{\lambda}}(\boldsymbol{\beta}).
\]
We denote the collection of local variational parameters by
\(
\boldsymbol{\phi}_{1:n} := \{\boldsymbol{\phi}_1,\ldots,\boldsymbol{\phi}_n\},
\)
and use $\boldsymbol{\phi}$ as a shorthand for $\boldsymbol{\phi}_{1:n}$ when the context is clear. 

The idea of MFVI is to identify a distribution within the variational family $\mathcal{Q}$ that best approximates the true posterior distribution.
Accordingly, MFVI can be formulated as the following optimization problem~\citep{blei2017variational}:
\begin{align*}
    \boldsymbol{\phi}^\ast, \boldsymbol{\lambda}^\ast = \argmin_{\boldsymbol{\phi}, \boldsymbol{\lambda}} \mathrm{KL}(q_{\boldsymbol{\phi},\boldsymbol{\lambda} }(\mathbf{z}, \boldsymbol{\beta}) ||p(\mathbf{z}, \boldsymbol{\beta} \mid \mathbf{x})).
\end{align*}
Using the identity
\(
\mathrm{KL}\!\left(
q_{\boldsymbol{\phi},\boldsymbol{\lambda}}(\mathbf{z}, \boldsymbol{\beta})
\;\middle\|\;
p(\mathbf{z}, \boldsymbol{\beta} \mid \mathbf{x})
\right)
=
\mathbb{E}_q\!\left[
\log q_{\boldsymbol{\phi},\boldsymbol{\lambda} }(\mathbf{z}, \boldsymbol{\beta})
-
\log p(\mathbf{z}, \boldsymbol{\beta}, \mathbf{x})
+
\log p(\mathbf{x})\right],
\)
and noting that $\log p(\mathbf{x})$ is independent of the variational parameters, the above problem is equivalent to $ \argmin_{\boldsymbol{\phi}, \boldsymbol{\lambda}}
\frac{1}{n}
\mathbb{E}_q\!\left[
\log q_{\boldsymbol{\phi},\boldsymbol{\lambda} }(\mathbf{z}, \boldsymbol{\beta})
-
\log p(\mathbf{z}, \boldsymbol{\beta}, \mathbf{x})
\right]$. 
By the conditional independence assumption, the joint log-density $\log p(\mathbf z,\boldsymbol{\beta},\mathbf x)$ decomposes across observations. Therefore, the objective admits a separable structure and can be equivalently written as
\begin{align} \label{eq:objective}
\min_{\boldsymbol{\phi},\, \boldsymbol{\lambda}}
\; \frac{1}{n}\sum_{i=1}^n f_i(\boldsymbol{\phi}_i, \boldsymbol{\lambda}),
\end{align}
where $f_i(\boldsymbol{\phi}_i, \boldsymbol{\lambda})
:=
\mathbb{E}_q[\log q_{\boldsymbol{\phi}_i}(\mathbf{z}_i)]
-
\mathbb{E}_q[\log p(\mathbf{x}_i \mid \veczeta_i, \vecbeta)]
-
\mathbb{E}_q[\log p_{\boldsymbol{\theta}}(\mathbf{z}_i)]
+
\frac{1}{n}
\mathbb{E}_q[\log q_{\boldsymbol{\lambda}}(\boldsymbol{\beta}) - \log p_{\boldsymbol{\theta}}(\boldsymbol{\beta})]$.
The derivation of this decomposition is provided in Appendix~\ref{section:decomposition_f}.

\subsection{Algorithms}
Our key observation is that the finite-sum optimization problem~\eqref{eq:objective} can be equivalently reformulated as the following constrained problem:
{
\small
\begin{equation}\label{eq:objective_constraint}
\begin{aligned}
\min_{\boldsymbol{\phi}_{1:n},\, \boldsymbol{\lambda}_{0:n}}
\; & \frac{1}{n}\sum_{i=1}^n f_i(\boldsymbol{\phi}_i, \boldsymbol{\lambda}_i)
\quad\text{s.t.}\; \boldsymbol{\lambda}_i = \boldsymbol{\lambda}_0,\; \forall i \in [n].
\end{aligned}
\end{equation}}
To solve the consensus-constrained problem in \eqref{eq:objective_constraint}, a standard approach is the \emph{methods of multipliers}. 
Specifically, consider the augmented Lagrangian (AL) function
{
\small
\begin{align} 
\nonumber
\mathcal{L}\!\left(\boldsymbol{\phi}_{1:n}, \boldsymbol{\lambda}_{0:n}, \boldsymbol{\mu}_{1:n}\right)
&=
\frac{1}{n}\sum_{i=1}^n
\Bigg[
f_i(\boldsymbol{\phi}_i,\boldsymbol{\lambda}_i)
+
\left\langle \boldsymbol{\mu}_i,\, \boldsymbol{\lambda}_i-\boldsymbol{\lambda}_0 \right\rangle \\
&\qquad
+
\frac{1}{2\eta}\left\|\boldsymbol{\lambda}_i-\boldsymbol{\lambda}_0\right\|_2^2
\Bigg],\label{eq:alm}
\end{align}}
where $\boldsymbol{\mu}_i$ denotes the dual variable associated with the consensus constraint
$\boldsymbol{\lambda}_i=\boldsymbol{\lambda}_0$, and $\eta>0$ is the penalty parameter.
For sufficiently small \(\eta\), minimizing \eqref{eq:alm} yields solutions that satisfy the KKT conditions of \eqref{eq:objective_constraint}
(see, e.g., Sec.~17.3 of \citet{wright1999numerical}).

The AL transforms constrained minimization~\eqref{eq:objective_constraint} into unconstrained minimization~\eqref{eq:alm}. The objective in \eqref{eq:alm} is separable over sample index $i$ except for the global variable $\boldsymbol{\lambda}_0$. This structure naturally motivates a mini-batch primal-dual Algorithm~\ref{alg:meta}.

At iteration $t$, we first randomly sample a subset of indices $S_t$ from $[n]$. The size of $S_t$ could be decided by memory constraints. For each sample in $S_t$, we perform a primal update step and a dual update step sequentially. These primal-dual updates could be carried out in parallel for all samples in $S_t$. The updated sample-wise primal variables $\boldsymbol{\lambda}_i^t$ are then aggregated to update shared $\boldsymbol{\lambda}_0^t$.  For samples not in $S_t$, we simply assume their primal and dual variables inherit their values from the previous iteration. The outline of the mini-batch algorithm is summarized in Algorithm~\ref{alg:meta}, while the primal-dual subroutines are detailed in Algorithm~\ref{alg:oracle_1} and~\ref{alg:oracle_2}.

\begin{algorithm}[tb]
  \caption{Mini-batch primal-dual variational inference}
  \label{alg:meta}
  \begin{algorithmic}
    \STATE {\bfseries Input:} $\boldsymbol{\lambda}_0^0,T$
    \STATE {\bfseries Initialize:} $\boldsymbol{\lambda}_i^0 = \boldsymbol{\lambda}_0^0$, $\boldsymbol{\mu}^0 = 0$, $\boldsymbol{h}^0 = 0$
    \FOR{$t = 1,\dots,T$}
      \STATE $S_t \gets (\text{subset of indices from [n]})$
      \FOR{each $i \in S_t$ {\bfseries in parallel}}
        \STATE {\bfseries Local Updates:}
        \vspace{1mm}
        \STATE
        $\begin{aligned}
            \boldsymbol{\phi}_i^{t},\;
            \boldsymbol{\lambda}_i^{t}, \; 
            \boldsymbol{\mu}_i^t, = \texttt{Oracle I}(f_i(\boldsymbol{\phi}_i, \boldsymbol{\lambda}_i), \boldsymbol{\lambda}_0^{t-1}, \boldsymbol{\mu}_i^{t-1})) 
        \end{aligned}$, or $\texttt{Oracle II}(f_i(\boldsymbol{\phi}_i, \boldsymbol{\lambda}_i), \boldsymbol{\lambda}_0^{t-1}, \boldsymbol{\mu}_i^{t-1}))$
      \ENDFOR
      \FOR{$i \notin S_t$ {\bfseries in parallel}}
        \STATE $\boldsymbol{\phi}_i^t = \boldsymbol{\phi}_i^{t-1}$, $\boldsymbol{\lambda}_i^t = \boldsymbol{\lambda}_i^{t-1}$, $\boldsymbol{\mu}_i^{t} = \boldsymbol{\mu}_i^{t-1}$
      \ENDFOR
      \STATE $\boldsymbol{h}^{t} = \boldsymbol{h}^{t-1} + \frac{1}{ n}\sum_{i\in S_t} \left( \boldsymbol{\lambda}_i^{t}- \boldsymbol{\lambda}_{0}^{t-1}\right)$
      \STATE $\boldsymbol{\lambda}_{0}^{t} = \frac{1}{|S_t|} \sum_{i \in S_t} \boldsymbol{\lambda}_i^{t} + \boldsymbol{h}^{t}$
    \ENDFOR
  \end{algorithmic}
\end{algorithm}

\begin{algorithm}[tb]
  \caption{\texttt{Oracle I}}
  \label{alg:oracle_1}
  \begin{algorithmic}
    \STATE {\bfseries Input:} $f_i(\boldsymbol{\phi}_i, \boldsymbol{\lambda}_i), 
    \boldsymbol{\lambda}_0^{t-1}, \boldsymbol{\mu}_i^{t-1}, \eta$
    \STATE $\boldsymbol{\phi}_i^{t},\;
            \boldsymbol{\lambda}_i^{t} = \argmin_{\boldsymbol{\phi}_i, \boldsymbol{\lambda}_i} f_i(\boldsymbol{\phi}_i, \boldsymbol{\lambda}_i)
+
\left\langle \boldsymbol{\mu}^{t-1}_i,\, \boldsymbol{\lambda}_i - \boldsymbol{\lambda}_0^{t-1} \right\rangle
+
\frac{1}{2\eta}\,\|\boldsymbol{\lambda}_i - \boldsymbol{\lambda}_0^{t-1}\|_2^2$
    \STATE $\boldsymbol{\mu}_i^{t}
            = \boldsymbol{\mu}_i^{t-1}
              + \frac{1}{\eta} \left(
                \boldsymbol{\lambda}_i^{t}
                - \boldsymbol{\lambda}_{0}^{t-1} \right)$
  \end{algorithmic}
\end{algorithm}

\begin{remark}
    The final step of Algorithm~\ref{alg:meta} is not a simple average of the \(\boldsymbol{\lambda}_i\) values in the selected subset, which may appear counterintuitive at first. Recall that under the alternating minimization framework, the final step updates \(\boldsymbol{\lambda}_0\) by minimizing \eqref{eq:alm} with all other variables fixed. 
Applying the first-order optimality condition to \eqref{eq:alm} yields $\boldsymbol{\lambda}_0 = \frac{1}{n}\sum_{i=1}^n (\boldsymbol{\lambda}_i + \eta \boldsymbol{\mu}_i)$, which motivates the last two steps in Algorithm~\ref{alg:meta}.
\end{remark}

The dual update in \texttt{Oracle I} ensures that $\boldsymbol{\mu}_i^t=-\nabla_{\boldsymbol{\lambda}_i}f_i(\boldsymbol{\phi}_i^t,\boldsymbol{\lambda}_i^t)$. This term mitigates the gradient differences across different samples and corrects consensus violations while still allowing meaningful updates of primal variables. The quadratic penalty term $\frac{1}{2\eta}\,\|\boldsymbol{\lambda}_i - \boldsymbol{\lambda}_0^{t-1}\|_2^2$ serves as a proximal regularizer that keeps local updates close to the global variable.

The primal minimization problem in Algorithm~\ref{alg:oracle_1} could be solved by a wide range of standard unconstrained methods can be applied. Based on our empirical results for MFVI, we recommend coordinate descent. 


\subsection{Preconditioned Augmented Lagrangian Algorithm}
As discussed, parameter heterogeneity poses a fundamental challenge in mean-field variational inference.
Both our theoretical analysis (Section~\ref{sec:analysis}) and empirical results (Section~\ref{sec:experiments}) demonstrate that pronounced heterogeneity in loss geometry across variational parameter blocks can lead to poor conditioning and optimization inefficiencies.
Such heterogeneity is intrinsic to MFVI, which typically involves optimizing a large number of parameters with distinct roles, scales, and optimization dynamics.

To address this challenge, we introduce a block-wise preconditioned augmented Lagrangian that adapts to the geometry of different variational parameter blocks. Specifically, we partition global variables $\boldsymbol{\lambda}$ into $B$ disjoint blocks and assign a customized quadratic penalty parameter to each block. The dimension of block $j$ is $d_j$. All coordinates within the same block share the same penalty parameter $\eta_j$, while different blocks may use different $\eta_j$'s. Formally, we define a diagonal matrix $
\mathbf{D}_\eta
:=
\mathrm{blkdiag}\bigl(\frac{1}{\eta_1}\mathbf{I}_{d_1},\ldots,\frac{1}{\eta_B}\mathbf{I}_{d_B}\bigr)$,
and formulate the \textit{block-wise preconditioned augmented Lagrangian} as
{
\small
$$
\mathcal{L}'_i(\boldsymbol{\phi}_i, \boldsymbol{\lambda}_i)
=
f_i(\boldsymbol{\phi}_i, \boldsymbol{\lambda}_i)
+
\left\langle \boldsymbol{\mu}_i,\, \boldsymbol{\lambda}_i - \boldsymbol{\lambda}_0 \right\rangle
+
\frac{1}{2}\|\boldsymbol{\lambda}_i - \boldsymbol{\lambda}_0\|_{\mathbf{D}_\eta}^2.
$$}

Accordingly, we introduce Algorithm~\ref{alg:oracle_2} to implement the corresponding oracle updates.
In practice, we recommend dividing these blocks by the scale of the Lipschitz constant, which can be approximated by the upper bound of second order derivatives. When the Lipschitz constants associated with different variables differ substantially in scale, grouping variables into blocks and applying the proposed diagonal preconditioning is particularly beneficial empirically, as it effectively balances the update magnitudes across blocks.

\begin{algorithm}[tb]
  \caption{\texttt{Oracle II}}
  \label{alg:oracle_2}
  \begin{algorithmic}
    \STATE {\bfseries Input:} $f_i(\boldsymbol{\phi}_i, \boldsymbol{\lambda}_i), 
    \boldsymbol{\lambda}_0^{t-1}, \boldsymbol{\mu}_i^{t-1}, \mathbf{D}_\eta$
    \STATE $\boldsymbol{\phi}_i^{t},\;
            \boldsymbol{\lambda}_i^{t} = \argmin_{\boldsymbol{\phi}_i, \boldsymbol{\lambda}_i} f_i(\boldsymbol{\phi}_i, \boldsymbol{\lambda}_i)
+
\left\langle \boldsymbol{\mu}_i^{t-1},\, \boldsymbol{\lambda}_i - \boldsymbol{\lambda}_0^{t-1} \right\rangle
+
\frac{1}{2}\,\|\boldsymbol{\lambda}_i - \boldsymbol{\lambda}_0^{t-1}\|_{\mathbf{D}_\eta}^2$
    \STATE $\boldsymbol{\mu}_i^{t}
            = \boldsymbol{\mu}_i^{t-1}
              + \mathbf{D}_\eta \left(
                \boldsymbol{\lambda}_i^{t}
                - \boldsymbol{\lambda}_{0}^{t-1} \right)$
  \end{algorithmic}
\end{algorithm}

\section{Theoretical Analysis} \label{sec:analysis}
In this section, we present two main theorems establishing the convergence rates of \name and \pname. We first state the assumptions used in the analysis. The first two assumptions concern Lipschitz continuity. Specifically, Assumption~\ref{ass:lsmooth} applies a single Lipschitz constant to all variables.
\begin{assumption}[Uniform Lipschitz Smoothness] \label{ass:lsmooth}
  Each $\nabla f_i(\cdot)$ is $L$-Lipschitz with $L\in (0,\infty), i.e., $ for all $i\in[n]$ and all
$\boldsymbol{x},\boldsymbol{y}$, $
      \|\nabla f_i(\boldsymbol{x}) - \nabla f_i(\boldsymbol{y})\| \leq L\|\boldsymbol{x}-\boldsymbol{y}\|.$
\end{assumption}
Assumption~\ref{ass:block_smooth_nonconvex} provides a more refined analysis by allowing block-wise Lipschitz constants, with each block of variables admitting its own Lipschitz constant.
\begin{assumption}[Block-wise Lipschitz Smoothness] \label{ass:block_smooth_nonconvex}
For each $i\in[n]$, the function $f_i(\boldsymbol{\phi},\boldsymbol{\lambda})$ is continuously differentiable.
Moreover, there exist constants $L_{\phi}\in(0,\infty)$ and $\{L_{\lambda,j}\}_{j=1}^B \subset (0,\infty)$ such that for all
$(\boldsymbol{\phi}_1,\boldsymbol{\lambda}_1),(\boldsymbol{\phi}_2,\boldsymbol{\lambda}_2)\in\mathbb{R}^{d_\phi}\times\mathbb{R}^{d_\lambda}$, and $j\in [B]$, $\left\|
\nabla_{\boldsymbol{\phi}}f_i(\boldsymbol{\phi}_1,\boldsymbol{\lambda}_1)
-
\nabla_{\boldsymbol{\phi}} f_i(\boldsymbol{\phi}_2,\boldsymbol{\lambda}_2)
\right\| \le
L_{\phi}
\|
(\boldsymbol{\phi}_1,\boldsymbol{\lambda}_1)
-
(\boldsymbol{\phi}_2,\boldsymbol{\lambda}_2)
\|$, and $\left\|
\bigl(\nabla_{\boldsymbol{\lambda}} f_i(\boldsymbol{\phi}_1,\boldsymbol{\lambda}_1)\bigr)_j
-
\bigl(\nabla_{\boldsymbol{\lambda}} f_i(\boldsymbol{\phi}_2,\boldsymbol{\lambda}_2)\bigr)_j
\right\|
\le
L_{\lambda,j}
\left\|
(\boldsymbol{\phi}_1,\boldsymbol{\lambda}_1)
-
(\boldsymbol{\phi}_2,\boldsymbol{\lambda}_2)
\right\|.$
Here, $(\cdot)_j$ denotes the $j$-th block of a vector under the partition of $\boldsymbol{\lambda}$.
\end{assumption}
 The next three assumptions further characterize the geometry of the objective $f_i$.  

\begin{assumption}[Convexity] \label{ass:convex}
    Each $f_i(\cdot)$ is convex, i.e., for all $i\in[n]$ and all
$\boldsymbol{x},\boldsymbol{y}$ in the domain, $f_i(\boldsymbol{y}) \geq f_i(\boldsymbol{x}) + \langle\nabla f_i(\boldsymbol{x}), \,  \boldsymbol{y}-\boldsymbol{x}\rangle .$
\end{assumption}
\begin{assumption}[Strong Convexity]\label{ass:strongly_convex}
Each $f_i(\cdot)$ is $\mu$-strongly convex. That is, for all $i\in[n]$ and all
$\boldsymbol{x},\boldsymbol{y}$ in the domain, $f_i(\boldsymbol{y})
\ge
f_i(\boldsymbol{x})
+
\left\langle \nabla f_i(\boldsymbol{x}),\, \boldsymbol{y}-\boldsymbol{x}\right\rangle
+
\frac{\mu}{2}\|\boldsymbol{y}-\boldsymbol{x}\|^2.$
\end{assumption}
\begin{assumption}[Strong Convexity in $\boldsymbol{\phi}$] \label{ass:phi_strong_convex_nonconvex}
For each $i\in[n]$, the function $f_i(\boldsymbol{\phi},\boldsymbol{\lambda})$ is twice continuously differentiable
and there exists
a constant $\mu\in(0,\infty)$ such that for all $i\in[n]$ and all $(\boldsymbol{\phi},\boldsymbol{\lambda})\in\mathbb{R}^{d_\phi}\times\mathbb{R}^{d_\lambda}$, $
\nabla^2_{\boldsymbol{\phi}\boldsymbol{\phi}} f_i(\boldsymbol{\phi},\boldsymbol{\lambda}) \succeq \mu \mathbf{I}_{d_\phi}.$
\end{assumption}

Assumption~\ref{ass:phi_strong_convex_nonconvex} is imposed to facilitate the nonconvex convergence analysis of Algorithm~\ref{alg:oracle_2}. 
In our numerical experiments, which focus on Gaussian mixture models with spatial correlation, this assumption is naturally satisfied. 
Specifically, for the Gaussian mixture model, the Hessian of the objective with respect to $\boldsymbol{\phi}$ takes the form $\nabla_{\boldsymbol{\phi}}^{2} = 1/\boldsymbol{\phi}$ (elementwise). 
Since $\boldsymbol{\phi}$ represents a probability vector, each component satisfies $\phi_k \le 1$, implying $1/\phi_k \ge 1$. 
Therefore, the Hessian is uniformly lower bounded by $1$, and the Assumption \ref{ass:phi_strong_convex_nonconvex} holds automatically.

To help theoretical analysis, we introduce the \emph{virtual variables}
$\{(\tilde{\boldsymbol{\phi}}_i^{t}, \tilde{\boldsymbol{\lambda}}_i^{t})\}$,
defined as the exact minimizers of the following local augmented subproblems: $
(\tilde{\boldsymbol{\phi}}_i^{t},\tilde{\boldsymbol{\lambda}}_i^{t})
=
\argmin_{\boldsymbol{\phi}_i,\boldsymbol{\lambda}_i}
\mathcal{L}_i(\boldsymbol{\phi}_i^{t-1}, \boldsymbol{\lambda}_i^{t-1}), \, \forall \, i\in [n].$ 
That is, $(\tilde{\boldsymbol{\phi}}_i^{t},\tilde{\boldsymbol{\lambda}}_i^{t})$ denotes a virtual full-update iterate at iteration $t$: it is obtained by hypothetically updating all local variables at iteration $t$.
These virtual iterates are introduced purely for theoretical analysis.

With these assumptions, we can state our worst case convergence rates, first for Algorithm \ref{alg:meta} and \ref{alg:oracle_1}, and then for Algorithm \ref{alg:meta} and \ref{alg:oracle_2}. The exact constants for the rates can be found in the proofs in Appendix \ref{app:proofs_section4}. Our rates are non-asymptotic and use big $O$ notation for brevity.
\begin{theorem}\label{thm:convergence_ours}
Assume that at each iteration $t$, the mini-batch $S_t$ is sampled
uniformly at random with fixed size $|S_t|=m$. $\boldsymbol{\phi}$ and $\boldsymbol{\lambda}$ represents local and global variables, respectively. Algorithm \ref{alg:meta} and \ref{alg:oracle_1} satisfies:\\
(i) (Strongly convex). Under Assumptions~\ref{ass:lsmooth} and~\ref{ass:strongly_convex}, for any $\eta = O(\frac{1}{L})$, we have $
\mathbb{E}\!\left[
f(\hat{\boldsymbol{\phi}}^{T}, \hat{\boldsymbol{\gamma}}^{T-1})
-
f^\ast
\right]
=
O\!\left(\frac{1}{r^{T-1}}\right).$

(ii) (Convex). Under Assumptions ~\ref{ass:lsmooth} and~\ref{ass:convex}, for any $\eta = O\!\left(
\frac{1}{L}
\right)$, we have $
\mathbb{E}\left[
f(\bar{\boldsymbol{\phi}}^{T},\bar{\boldsymbol{\gamma}}^{T-1})
-
f^\ast
\right]
=
O\left( \frac{1}{T}\right).$

(iii) (Nonconvex). Let $\kappa:=\frac{L^2}{\mu^2}$. Under Assumptions \ref{ass:lsmooth} and \ref{ass:phi_strong_convex_nonconvex}, for any \( 
\eta = O\left(\frac{\omega}{L(1+\kappa)}\right)\), we have $
\mathbb{E}\left[\frac{1}{T}\sum_{t=1}^T
\Bigl\|
\nabla_{\boldsymbol{\lambda}} f(\boldsymbol{\phi}^{t-1},\boldsymbol{\gamma}^{t-1})
\Bigr\|^2\right]
\;=\;
O\!\left(\frac{1}{T}\right).$

Here $\omega = \frac{m}{n}$, $r = 1 + \mu \eta$, $R = \sum_{t=1}^{T} r^{t-1}$, $\hat{\boldsymbol{\phi}}^{T}
:= \frac{1}{R}\sum_{t=1}^T r^{t-1}\tilde{\boldsymbol{\phi}}^{t}$, $\hat{\boldsymbol{\gamma}}^{T-1}
:= \frac{1}{R}\sum_{t=1}^T r^{t-1}\boldsymbol{\gamma}^{t-1}$, $\boldsymbol{\gamma}^t = \frac{1}{m}\sum_{i \in S_t} \boldsymbol{\lambda}_i^t$, 
\(
\bar{\boldsymbol{\phi}}^{T}
=
\frac{1}{T}\sum_{t=1}^T\tilde{\boldsymbol{\phi}}^{t},
\) \(
\bar{\boldsymbol{\gamma}}^{T-1}
=
\frac{1}{T}\sum_{t=1}^T\boldsymbol{\gamma}^{t-1}
\), 
and $f^* = f(\boldsymbol{\phi}^\ast, \boldsymbol{\lambda}^\ast)$ for convex and strongly convex $f$, where $(\boldsymbol{\phi}^\ast, \boldsymbol{\lambda}^\ast)$ represents the global optimal solution. The expectation is taken over the randomness of mini-batch sampling.
\end{theorem}
Theorem~\ref{thm:convergence_ours} establishes convergence guarantees for three classes of objective functions. In the strongly convex setting, the proposed method achieves a linear convergence rate to the global optimum. In the convex case, a sublinear $O\left(1/T\right)$ convergence rate is obtained. For nonconvex objectives, the algorithm converges sublinearly to a stationary point. 

The proof of Theorem~\ref{thm:convergence_ours} is technically nontrivial, as Algorithm~\ref{alg:meta} updates global and local variational parameters simultaneously in each iteration.
This simultaneous updating inevitably introduces additional error terms, since updates to local variables modify the objective function $f_i$.
The central technical challenge is therefore to control these induced local errors and ensure that they do not obstruct overall convergence.
Our main technical contribution is the development of an error bound for the local variables, enabled by the $\boldsymbol{\phi}$-strongly convex assumption ~\ref{ass:phi_strong_convex_nonconvex}. The full proof is relegated to Appendix~\ref{app:proofs_section4}. 

 For \pname, Theorem~\ref{thm:convergence_preconditioned} below establishes a sublinear convergence rate in the general setting, without requiring any convexity assumption on the global variables $\boldsymbol{\lambda}$.
\begin{theorem} \label{thm:convergence_preconditioned}
Assume that at each iteration $t$, the mini-batch $S_t$ is sampled
uniformly at random with fixed size $|S_t|=m$. Under assumptions~\ref{ass:block_smooth_nonconvex} and~\ref{ass:phi_strong_convex_nonconvex}, for any positive block-wise step size $\{\eta_j\}_{j=1}^B$ satisfying $\sum_{j=1}^B \eta_j^2 L_j^2 \le C$, where $C$ is a constant specified in~\eqref{eq:eta_conditions_all}, Algorithm \ref{alg:meta} and \ref{alg:oracle_2} satisfies:
{
\small
\begin{equation}
\label{eqn:pnameconverge}
\begin{aligned}
\sum_{j=1}^B \eta_j \mathbb{E}\!\left[\frac{1}{T}\sum_{t=1}^T
\Bigl\|
(\nabla_{\boldsymbol{\lambda}} f(\boldsymbol{\phi}^{t-1},\boldsymbol{\gamma}^{t-1}))_j
\Bigr\|^2\right]
\;=\;
O\!\left(
\frac{1}{T}\right).
\end{aligned}
\end{equation}}
Here, $(\cdot)_j$ denotes the $j$-th block of a vector under the partition of $\boldsymbol{\lambda}$.
\end{theorem}

 Our theoretical analysis shows that the block-wise step size must satisfy that
\(
\sum_{j=1}^B \eta_j^2 L_j^2 
\)
is upper bounded by a constant independent of the step size $\{\eta_j\}_{j=1}^B$ and the Lipschitz constants $\{L_j\}_{j=1}^B$. Here $L_j$ denotes the Lipschitz constant associated with the $j$-th block of the global variables. This condition naturally motivates a block-wise step size choice of the form $\eta_j \propto 1/L_j$, which balances the effective scales across blocks and avoids forcing all step size to be uniformly small due to the worst-conditioned block. Moreover, \eqref{eqn:pnameconverge} suggests that, within the admissible range, employing larger $\eta_j$'s yields a faster decrease of $\|\nabla_{\boldsymbol{\lambda}}f\|$. A formal statement and the full proof of Theorem~\ref{thm:convergence_preconditioned} is also relegated to Appendix~\ref{app:proofs_section4}.

\begin{wrapfigure}{r}{0.3\textwidth}
\vspace{-0.5cm}
  \begin{center}
    \centerline{\includegraphics[width=0.6\columnwidth]{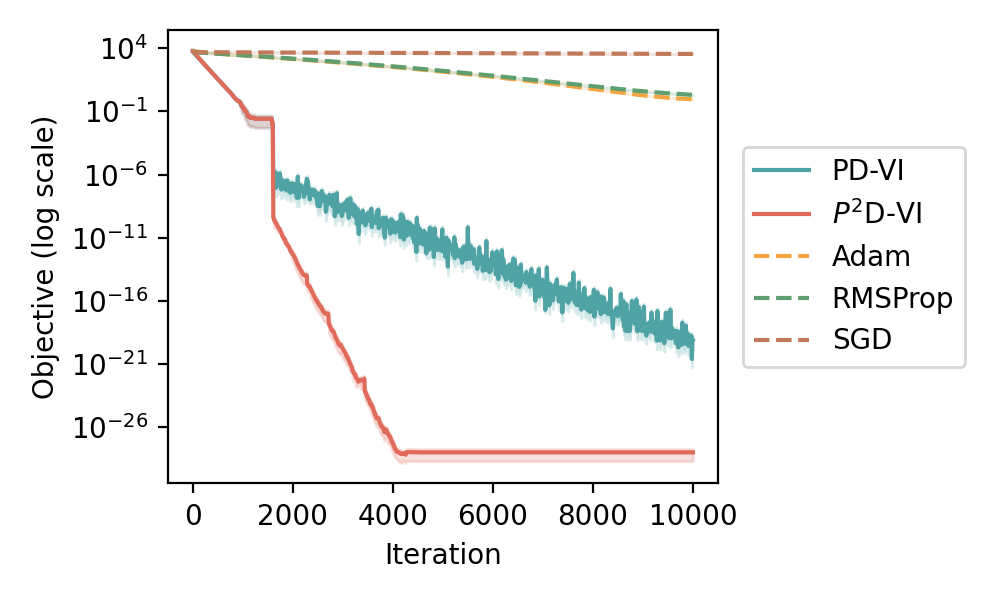}}
    \vspace{-0.5cm}
    \caption{Loss curves on a strongly convex quadratic problem.}
    \label{fig:quadratic}
  \end{center}
  \vspace{-0.7cm}
\end{wrapfigure}

We further corroborate our theoretical analysis on a strongly convex quadratic optimization problem, which serves as a controlled benchmark for illustrating the predicted convergence behavior of the proposed algorithms. The full experiment setup is relegated to Appendix \ref{app:exp_quadratic} while the loss curve is plotted in Figure~\ref{fig:quadratic}. The results highlight two key observations: (i) both \name and \pname exhibit linear convergence in the strongly convex problem; and (ii) preconditioning substantially accelerates convergence.

\section{Experiments} \label{sec:experiments}
\begin{wrapfigure}{r}{0.3\textwidth}
\vspace{-0.7cm}
  \centering
\includegraphics[width=0.6\columnwidth]{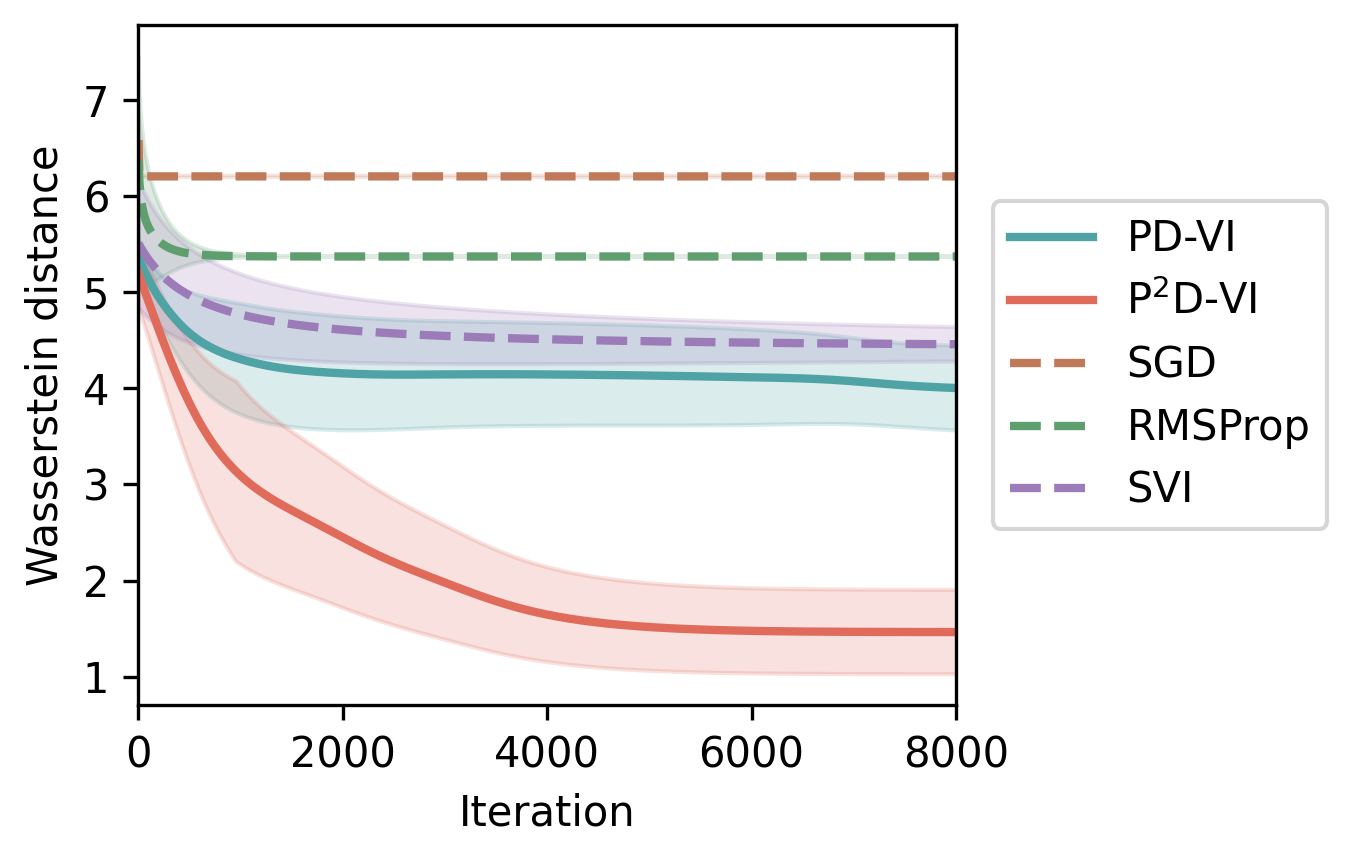}
  \caption{Wasserstein distance between the true and variational Gaussian mixture distributions on the synthetic dataset.}
  \label{fig:wasserstein}
  \vspace{-0.5cm}
\end{wrapfigure}
In this section, we evaluate both \name\ and \pname, along with several baseline methods, on synthetic and spatial transcriptomics datasets. We assess their performance in terms of computational efficiency and solution quality. Code to generate all numerical results is available in \url{https://github.com/JinhuaLyu/PD-VI_repo}.

\paragraph{Gaussian Mixture Model}

We also evaluate SVI algorithms on a synthetic Gaussian mixture model with five clusters and 100{,}000 data points in a 10-dimensional space. 
Detailed experimental settings are provided in Appendix~\ref{app:experiment_syn}. 
Figure~\ref{fig:wasserstein} plots the Wasserstein distance between true and variational Gaussian distributions, which shows \name and \pname achieve consistently lower Wasserstein distances and faster convergence than all baseline approaches, indicating superior solution quality and robustness under biased mini-batch sampling.

\paragraph{MOSTA Dataset} We use a spatial transcriptomics dataset from the MOSTA database, a large-scale resource profiling mouse embryonic development across multiple stages using Stereo-seq, which provides genome-wide gene expression measurements with spatial coordinates at cellular to subcellular resolution~\cite{chen2022spatiotemporal}. The dataset used in our experiments contains approximately 150{,}000 spatial locations, each associated with a high-dimensional gene expression profile spanning more than 20{,}000 genes. One of the key tasks in spatial transcriptomics analysis is spatial domain detection, which aims to identify tissue region with distinct molecular signatures. To address this, we model the data using a Bayesian hierarchical framework augmented with a Potts model to incorporate spatial dependencies, resulting in a non-conjugate inference problem. Details of the model are provided in Appendix~\ref{app:experiments_STD}. In this setting, the total number of variational variables is approximately three million. To enable scalable optimization, we partition the data and the corresponding variational parameters into around 30 fixed batches.

\begin{figure*}[h]
  \centering
  \includegraphics[width=0.99\textwidth]{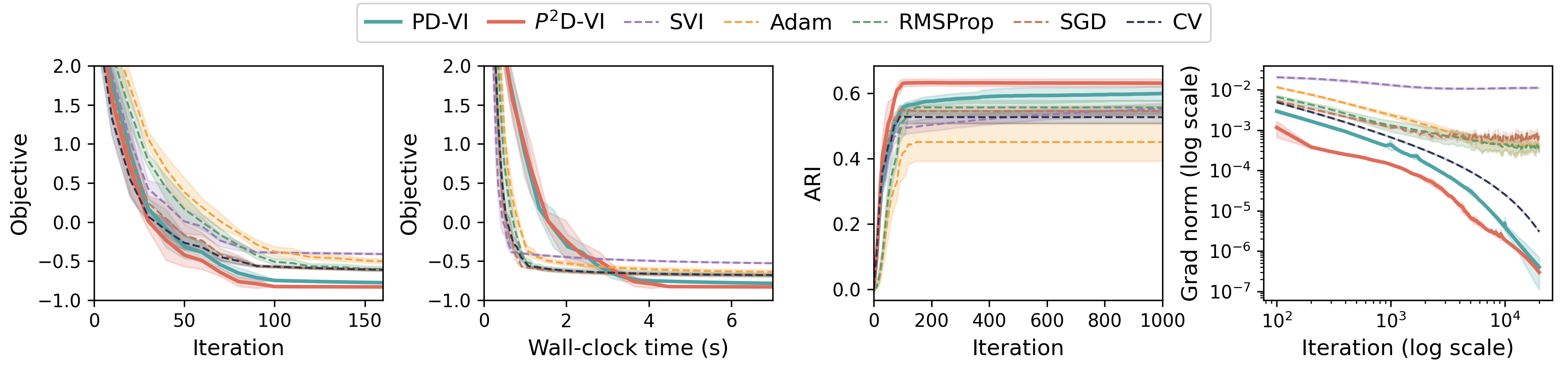}
  \caption{Convergence behavior and clustering performance on the MOSTA dataset.
From left to right, we show the evolution of the variational objective value (excluding constant terms) as a function of iteration count and wall-clock time, the adjusted Rand index (ARI) versus iteration, and the norm of the gradient with respect to the global variables versus iteration (displayed on a logarithmic scale). Solid lines represent the mean across runs, and shaded regions indicate one standard deviation over five random seeds. For clarity of visualization, the iteration range and wall-clock time shown correspond to a truncated segment of the full training process; extending the runs beyond the displayed range does not change the observed behavior.}
  \label{fig:combined}
\end{figure*}

\begin{figure*}[h]
  \centering
  \includegraphics[width=0.95\textwidth]{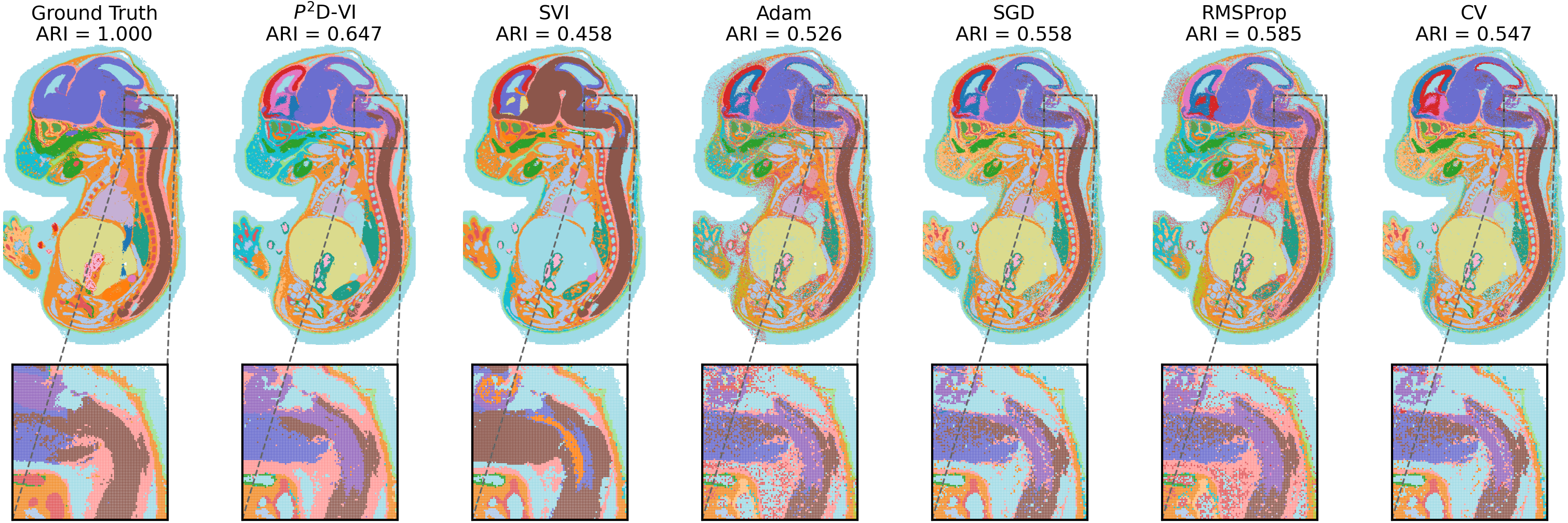}
  \caption{Spatial domains identified by \pname and competing methods across sagittal sections of mouse embryos. Ground-truth anatomical and tissue annotations are obtained from the MOSTA reference atlas generated in the original study (left). Clustering performance of different methods is quantified using the adjusted Rand index (ARI), where higher ARI values indicate better agreement with the reference annotations. }
  \label{fig:clusters-row}
\end{figure*}

We implement several baseline methods, including SVI~\cite{hoffman2013stochastic} with both constant and diminishing step size, SGD~\cite{ketkar2017stochastic}, RMSProp~\citep{tieleman2012lecture}, and Adam~\citep{kingma2014adam}. For a fair comparison, all methods are optimized with the same objective function, namely the ELBO, as defined in Section~\ref{sec:preliminary}. All baselines are carefully tuned and allow block-wise step sizes when applicable. 
Although PG-SVI and SNGVI are discussed in the introduction, both can be viewed as variants of SVI; under conjugate priors and without local latent variables, the three of them are equivalent. In our spatial transcriptomics experiments, the model includes local latent variables and a Potts prior to encode spatial dependencies, which breaks conjugacy and makes the original closed-form SVI updates inapplicable. We therefore follow the core SVI principle with a modified implementation: the global variational parameters are updated using a closed-form update equivalent to a single natural-gradient step, while the local variational variables are optimized at each global update. In practice, the local optimization converges in fewer than 20 iterations per global update, yielding an efficient SVI-style procedure tailored to our non-conjugate model with local latent variables.

Figure~\ref{fig:combined} compares the convergence behavior across methods.
The first two panels show that both \name\ and \pname\ converge faster in terms of iteration count than competing methods, with \pname\ achieving the fastest convergence, suggesting the effectiveness of the proposed preconditioning strategy.
Although \pname\ does not always exhibit the fastest initial progress in terms of wall-clock time, it consistently reaches lower objective values and remains competitive overall, demonstrating favorable practical efficiency.
While the adjusted rand index (ARI)~\cite{santos2009use} does not quantify convergence, it evaluates clustering accuracy relative to ground truth. According to this metric, \pname\ consistently attains higher and more stable values than the baselines.
We also observe that CV-based baselines exhibit relatively rapid initial decreases in the objective and gradient norm, which can be attributed to their variance reduction mechanisms.
Finally, the gradient norm versus iteration plot (shown on a logarithmic scale) provides further insight into the optimization dynamics: \pname\ exhibits the steepest decrease, indicating faster asymptotic convergence.

Figure~\ref{fig:clusters-row} presents the spatial domain identification results on the MOSTA dataset. Compared with baselines, \pname\ achieves the highest clustering accuracy. Visually, the clusters produced by \pname\ exhibit improved spatial coherence and sharper anatomical boundaries, particularly in complex embryonic regions where multiple tissue types are closely intertwined.
The zoomed-in views further highlight that \pname\ better preserves fine-grained spatial structures and reduces spurious fragmentation compared to other methods.


\section{Conclusion and Discussion}
We introduce \name and its block-preconditioned variant \pname for large-scale MFVI in mini-batch settings. 
We believe this primal--dual perspective opens new avenues for designing scalable optimization algorithms for modern VI.

\newpage


\bibliography{ref.bib}
\bibliographystyle{icml2026}

\newpage
\appendix
\onecolumn
In the Appendix, we will first present the derivations of the objective in Section 3, followed by the proof to the convergence theorems in Section 4. Then, we will exhibit the details in the numerical experiments in Section 5 and the additional numerical results.

\section{Proofs for Section 3 (Our Methods)}
\subsection{Decomposition of $f(\boldsymbol{\phi}, \boldsymbol{\lambda})$}
\label{section:decomposition_f}

\begin{align*}
f(\boldsymbol{\phi}, \boldsymbol{\lambda})
&:= \frac{1}{n}\mathbb{E}_q \bigl[
\log q_{\boldsymbol{\phi}}(\mathbf{z})
+ \log q_{\boldsymbol{\lambda}}(\boldsymbol{\beta})
- \log p(\mathbf{z}, \boldsymbol{\beta}, \mathbf{x})
\bigr].
\end{align*}

Using the mean-field factorization
\[
q_{\boldsymbol{\phi},\boldsymbol{\lambda}}(\mathbf{z}, \boldsymbol{\beta})
=
\Bigl(\prod_{i=1}^n q_{\boldsymbol{\phi}_i}(\mathbf{z}_i)\Bigr)\,
q_{\boldsymbol{\lambda}}(\boldsymbol{\beta}),
\]
and the joint distribution
\[
p(\mathbf{z}, \boldsymbol{\beta}, \mathbf{x})
=
\Bigl(\prod_{i=1}^n p(\mathbf{x}_i \mid \veczeta_i, \vecbeta)\Bigr)
\Bigl(\prod_{i=1}^n p_{\boldsymbol{\theta}}(\mathbf{z}_i)\Bigr)
p_{\boldsymbol{\theta}}(\boldsymbol{\beta}),
\]
where we assume that the local latent variables $\mathbf{z}_i$ are independent across data points $i\in[n]$, we obtain
\begin{align*}
f(\boldsymbol{\phi}, \boldsymbol{\lambda})
&=
\sum_{i=1}^n \mathbb{E}_q \!\left[ \log q_{\boldsymbol{\phi}_i}(\mathbf{z}_i) \right]
+ \mathbb{E}_q \!\left[ \log q_{\boldsymbol{\lambda}}(\boldsymbol{\beta}) \right]
- \sum_{i=1}^n \mathbb{E}_q \!\left[ \log p(\mathbf{x}_i \mid \veczeta_i, \vecbeta) \right]
- \sum_{i=1}^n \mathbb{E}_q \!\left[ \log p_{\boldsymbol{\theta}}(\mathbf{z}_i) \right]
- \mathbb{E}_q \!\left[ \log p_{\boldsymbol{\theta}}(\boldsymbol{\beta}) \right].
\end{align*}

Rearranging terms yields
\begin{align*}
f(\boldsymbol{\phi}, \boldsymbol{\lambda})
&=
\frac{1}{n}\sum_{i=1}^n
\Bigl(
\mathbb{E}_q[\log q_{\boldsymbol{\phi}_i}(\mathbf{z}_i)]
-
\mathbb{E}_q[\log p(\mathbf{x}_i \mid \veczeta_i, \vecbeta)]
-
\mathbb{E}_q[\log p_{\boldsymbol{\theta}}(\mathbf{z}_i)]
\Bigr)
+
\mathbb{E}_q[\log q_{\boldsymbol{\lambda}}(\boldsymbol{\beta})
-
\log p_{\boldsymbol{\theta}}(\boldsymbol{\beta})].
\end{align*}

Define
\[
f_i(\boldsymbol{\phi}_i, \boldsymbol{\lambda})
:=
\mathbb{E}_q[\log q_{\boldsymbol{\phi}_i}(\mathbf{z}_i)]
-
\mathbb{E}_q[\log p(\mathbf{x}_i \mid \veczeta_i, \vecbeta)]
-
\mathbb{E}_q[\log p_{\boldsymbol{\theta}}(\mathbf{z}_i)]
+
\frac{1}{n}
\mathbb{E}_q[\log q_{\boldsymbol{\lambda}}(\boldsymbol{\beta})
-
\log p_{\boldsymbol{\theta}}(\boldsymbol{\beta})],
\]
so we have
\begin{align*}
f(\boldsymbol{\phi}, \boldsymbol{\lambda})
=
\frac{1}{n}\sum_{i=1}^n f_i(\boldsymbol{\phi}_i, \boldsymbol{\lambda}).
\end{align*}
\section{Proofs for Section 4 (Analysis)} \label{app:proofs_section4}
The optimization problem is
\[
\min_{\boldsymbol{\phi},\,\boldsymbol{\lambda}}\; f(\boldsymbol{\phi},\boldsymbol{\lambda})
\;\;:=\;\; \frac{1}{n}\sum_{i=1}^n f_i(\boldsymbol{\phi}_i,\boldsymbol{\lambda}),
\]
where $\boldsymbol{\phi} := (\boldsymbol{\phi}_1,\ldots,\boldsymbol{\phi}_n)$ denotes the collection of local variables, with $\boldsymbol{\phi}_i$ specific to function $f_i$, and $\boldsymbol{\lambda}$ denotes the global variables shared across all functions.
For algorithm implementation, we introduce a local copy $\boldsymbol{\lambda}_i$ of the global variable $\boldsymbol{\lambda}$ for each function $f_i$. These local variables are driven to reach consensus through the algorithm's update rules. Throughout the following analysis, $f_i(\boldsymbol{\phi}_i,\boldsymbol{\lambda}_i)$ denotes the value of $f_i(\boldsymbol{\phi}_i,\boldsymbol{\lambda})$ evaluated at $\boldsymbol{\lambda}=\boldsymbol{\lambda}_i$. Moreover, with a slight abuse of notation, we write $\nabla_{\boldsymbol{\lambda}_i} f_i(\boldsymbol{\phi}_i,\boldsymbol{\lambda}_i)$ to denote the gradient of $f_i(\boldsymbol{\phi}_i,\boldsymbol{\lambda})$ with respect to $\boldsymbol{\lambda}$, evaluated at $\boldsymbol{\lambda}=\boldsymbol{\lambda}_i$. We write $\nabla f_i(\boldsymbol{\phi}_i,\boldsymbol{\lambda}_i)$ for the full gradient of $f_i$ evaluated at $(\boldsymbol{\phi}_i,\boldsymbol{\lambda}_i)$, i.e., the gradient with respect to both $\boldsymbol{\phi}_i$ and $\boldsymbol{\lambda}_i$.

We introduce the \emph{virtual variables}
\(\{(\tilde{\boldsymbol{\phi}}_i^{t}, \tilde{\boldsymbol{\lambda}}_i^{t})\}_{i=1}^n\),
defined as the exact minimizers of the following local augmented subproblems:
\begin{align*}
(\tilde{\boldsymbol{\phi}}_i^{t},\tilde{\boldsymbol{\lambda}}_i^{t})
&=
\argmin_{\boldsymbol{\phi}_i,\boldsymbol{\lambda}_i}
\;
f_i(\boldsymbol{\phi}_i,\boldsymbol{\lambda}_i)
+
\langle \boldsymbol{\mu}_i^{t-1},\,
\boldsymbol{\lambda}_i-\boldsymbol{\lambda}_0^{t-1}\rangle
+
\frac{1}{2\eta}
\|\boldsymbol{\lambda}_i-\boldsymbol{\lambda}_0^{t-1}\|_2^2 \\[0.3em]
&=
\argmin_{\boldsymbol{\phi}_i,\boldsymbol{\lambda}_i}
\;
f_i(\boldsymbol{\phi}_i,\boldsymbol{\lambda}_i)
+
\langle \boldsymbol{\mu}_i^{t-1},\,
\boldsymbol{\lambda}_i\rangle
+
\frac{1}{2\eta}
\|\boldsymbol{\lambda}_i-\boldsymbol{\lambda}_0^{t-1}\|_2^2,
\qquad \forall\, i \in [n].
\end{align*}
For any $i \in S_t$, the local update is exact, and thus
\(
(\boldsymbol{\phi}_i^{t},\boldsymbol{\lambda}_i^{t})
=
(\tilde{\boldsymbol{\phi}}_i^{t},\tilde{\boldsymbol{\lambda}}_i^{t})
\).
Define the averaged local variable
\begin{align}\label{eq:gamma_def}
  \boldsymbol{\gamma}^t
=
\frac{1}{|S_t|}\sum_{i\in S_t}\boldsymbol{\lambda}_i^{t},
\end{align}
which satisfies
\(
\boldsymbol{\gamma}^t
=
\boldsymbol{\lambda}_0^t-\boldsymbol{h}^t
\)
by construction. Moreover, define $m := |S_t|$ for notation simplicity.

\subsection{Convex Analysis}
We next introduce the following controlled quantities:
\begin{align*}
G_t
&=
\frac{1}{n}\sum_{i=1}^n
\mathbb{E}\!\left\|
\nabla f_i(\boldsymbol{\phi}_i^t,\boldsymbol{\lambda}_i^t)
-
\nabla f_i(\boldsymbol{\phi}_i^*,\boldsymbol{\lambda}_i^*)
\right\|_2^2,
\qquad
\epsilon_t
=
\frac{1}{n}\sum_{i=1}^n
\mathbb{E}\!\left\|
\tilde{\boldsymbol{\lambda}}_i^{t}
-
\boldsymbol{\gamma}^{t-1}
\right\|_2^2,
\end{align*}
when $G_t$ approaches zero, we know that the gradient of the whole function at each update is close to zero. 
When $\epsilon_t$ approaches zero, we know that the virtual variables are close to the averaged local variables from the previous iteration,
which indicates consensus among the variables.

Moreover, assume that the local updates are solved exactly,
so the first-order optimality conditions yield
\begin{align} \label{eq:dual_update}
\boldsymbol{\mu}_i^{t}
=
-\,\nabla_{\boldsymbol{\lambda}_i}
f_i(\boldsymbol{\phi}_i^{t},\boldsymbol{\lambda}_i^{t}), \quad \tilde{\boldsymbol{\mu}}_i^{t}
=
-\,\nabla_{\boldsymbol{\lambda}_i}
f_i(\tilde{\boldsymbol{\phi}}_i^{t},\tilde{\boldsymbol{\lambda}}_i^{t}).
\end{align}

Similarly, the optimality conditions for the virtual update are given by
\begin{align*}
\begin{cases}
\nabla_{\boldsymbol{\lambda}_i}
f_i(\tilde{\boldsymbol{\phi}}_i^{t},
\tilde{\boldsymbol{\lambda}}_i^{t})
+
\boldsymbol{\mu}_i^{t-1}
+
\frac{1}{\eta}
\bigl(\tilde{\boldsymbol{\lambda}}_i^{t}
-\boldsymbol{\lambda}_0^{t-1}\bigr)
= 0, \\[0.3em]
\nabla_{\boldsymbol{\phi}_i}
f_i(\tilde{\boldsymbol{\phi}}_i^{t},
\tilde{\boldsymbol{\lambda}}_i^{t})
= 0 .
\end{cases}
\end{align*}
Substituting \eqref{eq:dual_update} into the first condition above yields
\begin{align}\label{eq:true_equal_virtual}
\tilde{\boldsymbol{\lambda}}_i^{t}
-
\boldsymbol{\lambda}_0^{t-1}
=
\eta\Big(
\nabla_{\boldsymbol{\lambda}_i}
f_i(\boldsymbol{\phi}_i^{t-1},\boldsymbol{\lambda}_i^{t-1})
-
\nabla_{\boldsymbol{\lambda}_i}
f_i(\tilde{\boldsymbol{\phi}}_i^{t},
\tilde{\boldsymbol{\lambda}}_i^{t})
\Big).
\end{align}
In the convex case analysis, we use the following two assumptions. 
\begin{assumption}
  Each $\nabla f_i(\cdot)$ is $L$-Lipschitz with $L\in (0,\infty), i.e., $ for all $i\in[n]$ and all
$\boldsymbol{x},\boldsymbol{y}$ in the domain,
  \begin{align} 
      \|\nabla f_i(\boldsymbol{x}) - \nabla f_i(\boldsymbol{y})\| \leq L\|\boldsymbol{x}-\boldsymbol{y}\|.
  \end{align}
\end{assumption}
\begin{assumption} 
    Each $f_i(\cdot)$ is convex, i.e., for all $i\in[n]$ and all
$\boldsymbol{x},\boldsymbol{y}$ in the domain,
    \begin{align}
        f_i(\boldsymbol{y}) \geq f_i(\boldsymbol{x}) + \langle\nabla f_i(\boldsymbol{x}), \,  \boldsymbol{y}-\boldsymbol{x}\rangle
    \end{align}
\end{assumption}

\begin{theorem}
\label{thm:converge1.1}
Suppose that $\{f_i\}_{i=1}^n$ are convex and $L$-smooth.
For any step size $0 < \eta \leq \frac{1}{22L}$, Algorithm~\ref{alg:meta} and \ref{alg:oracle_1} satisfies
\begin{align}
\mathbb{E}\!\left[
f(\bar{\boldsymbol{\phi}}^{T},\bar{\boldsymbol{\gamma}}^{T-1})
-
f(\boldsymbol{\phi}^\ast,\boldsymbol{\lambda}^\ast)
\right]
=
\,
O\!\left( \frac{1}{T}\left(
\frac{1}{\eta}\|\boldsymbol{\lambda}^{0}-\boldsymbol{\lambda}^\ast\|^2
+
\frac{n}{m}\eta\,G_0
\right)\right).
\end{align}
where \(
\bar{\boldsymbol{\phi}}^{T}
=
\frac{1}{T}\sum_{t=1}^T\tilde{\boldsymbol{\phi}}^{t},
\) \(
\bar{\boldsymbol{\gamma}}^{T-1}
=
\frac{1}{T}\sum_{t=1}^T\boldsymbol{\gamma}^{t-1}
\), $G_0 = \frac{1}{n}\sum_{i=1}^n
\left\|
\nabla f_i(\boldsymbol{\phi}_i^*,\boldsymbol{\lambda}_i^*)
\right\|^2$.
\end{theorem}
\begin{proof}
From Lemma \ref{lemma:main_lemma}, for every $t\ge1$ we have
\begin{align*}
c_2\,
\mathbb{E}\!\left[
f(\tilde{\boldsymbol{\phi}}^{t},\boldsymbol{\gamma}^{t-1})
-
f(\boldsymbol{\phi}^\ast,\boldsymbol{\lambda}^\ast)
\right]
&\le
\Bigl(
\mathbb{E}\!\left\|\boldsymbol{\gamma}^{t-1}-\boldsymbol{\lambda}^\ast\right\|^2
+c_1\,G_{t-1}
\Bigr)
-
\Bigl(
\mathbb{E}\!\left\|\boldsymbol{\gamma}^{t}-\boldsymbol{\lambda}^\ast\right\|^2
+c_1\,G_{t}
\Bigr).
\end{align*}
Summing the above inequality over $t=1,\ldots,T$ yields the telescoping bound
\begin{align*}
c_2\sum_{t=1}^T
\mathbb{E}\!\left[
f(\tilde{\boldsymbol{\phi}}^{t},\boldsymbol{\gamma}^{t-1})
-
f(\boldsymbol{\phi}^\ast,\boldsymbol{\lambda}^\ast)
\right]
&\le
\Bigl(
\mathbb{E}\!\left\|\boldsymbol{\gamma}^{0}-\boldsymbol{\lambda}^\ast\right\|^2
+c_1\,G_{0}
\Bigr)
-
\Bigl(
\mathbb{E}\!\left\|\boldsymbol{\gamma}^{T}-\boldsymbol{\lambda}^\ast\right\|^2
+c_1\,G_{T}
\Bigr).
\end{align*}
When $0<\eta \leq\frac{1}{22L}$, $c_1$ and $c_2$ are positive, and clearly $G_T\ge0$. Hence
\begin{align*}
c_2\sum_{t=1}^T
\mathbb{E}\!\left[
f(\tilde{\boldsymbol{\phi}}^{t},\boldsymbol{\gamma}^{t-1})
-
f(\boldsymbol{\phi}^\ast,\boldsymbol{\lambda}^\ast)
\right]
&\le
\mathbb{E}\!\left\|\boldsymbol{\gamma}^{0}-\boldsymbol{\lambda}^\ast\right\|^2
+c_1\,G_{0}.
\end{align*}
Dividing both sides by $c_2T$ gives
\begin{align*}
\frac{1}{T}\sum_{t=1}^T
\mathbb{E}\!\left[
f(\tilde{\boldsymbol{\phi}}^{t},\boldsymbol{\gamma}^{t-1})
-
f(\boldsymbol{\phi}^\ast,\boldsymbol{\lambda}^\ast)
\right]
&\le
\frac{1}{c_2T}\Bigl(
\mathbb{E}\!\left\|\boldsymbol{\gamma}^{0}-\boldsymbol{\lambda}^\ast\right\|^2
+c_1\,G_{0}
\Bigr).
\end{align*}
Define that
\[
\bar{\boldsymbol{\phi}}^{T}
\;=\;
\frac{1}{T}\sum_{t=1}^T\tilde{\boldsymbol{\phi}}^{t},
\qquad
\bar{\boldsymbol{\gamma}}^{T-1}
\;=\;
\frac{1}{T}\sum_{t=1}^T\boldsymbol{\gamma}^{t-1}.
\]
Since $f$ is convex in $(\boldsymbol{\phi},\boldsymbol{\lambda})$, Jensen's inequality implies
\[
f(\bar{\boldsymbol{\phi}}^{T},\bar{\boldsymbol{\gamma}}^{T-1})
\;\le\;
\frac{1}{T}\sum_{t=1}^T f(\tilde{\boldsymbol{\phi}}^{t},\boldsymbol{\gamma}^{t-1}).
\]
Taking expectations and combining with the previous bound, we obtain
\begin{align*}
\mathbb{E}\!\left[
f(\bar{\boldsymbol{\phi}}^{T},\bar{\boldsymbol{\gamma}}^{T-1})
-
f(\boldsymbol{\phi}^\ast,\boldsymbol{\lambda}^\ast)
\right]
&\le
\frac{1}{T}\sum_{t=1}^T
\mathbb{E}\!\left[
f(\tilde{\boldsymbol{\phi}}^{t},\boldsymbol{\gamma}^{t-1})
-
f(\boldsymbol{\phi}^\ast,\boldsymbol{\lambda}^\ast)
\right] \\
&\le
\frac{1}{c_2T}\Bigl(
\mathbb{E}\!\left\|\boldsymbol{\gamma}^{0}-\boldsymbol{\lambda}^\ast\right\|^2
+c_1\,G_{0}
\Bigr),\\
&= \frac{1}{T}\Bigl(
\frac{1}{c_2}\!\left\|\boldsymbol{\lambda}_0^{0}-\boldsymbol{\lambda}^\ast\right\|^2
+\frac{c_1}{c_2}\,G_{0}
\Bigr).
\end{align*}
Plugging in the values of $c_1$ and $c_2$ in Lemma~\ref{lemma:main_lemma},
we have
\[
\frac{1}{c_2}
=
\frac{1-20\eta^2L^2}{2\eta\bigl(1-20\eta L-40\eta^2L^2\bigr)},
\qquad
\frac{c_1}{c_2}
=
\frac{n}{m}\cdot
\frac{4\eta(1+\eta L)}{1-20\eta L-40\eta^2L^2}.
\]
Therefore,
\begin{align*}
\mathbb{E}\!\left[
f(\bar{\boldsymbol{\phi}}^{T},\bar{\boldsymbol{\gamma}}^{T-1})
-
f(\boldsymbol{\phi}^\ast,\boldsymbol{\lambda}^\ast)
\right]
\le
\frac{1}{T}\left[
\frac{1-20\eta^2L^2}{2\eta\bigl(1-20\eta L-40\eta^2L^2\bigr)}\,
\|\boldsymbol{\lambda}^{0}-\boldsymbol{\lambda}^\ast\|^2
+
\frac{n}{m}\cdot\frac{4\eta(1+\eta L)}{1-20\eta L-40\eta^2L^2}\,G_0
\right].
\end{align*}
Let $D(\eta):=1-20\eta L-40\eta^2L^2$. Then the above can be equivalently written as
\[
\mathbb{E}\!\left[
f(\bar{\boldsymbol{\phi}}^{T},\bar{\boldsymbol{\gamma}}^{T-1})
-
f(\boldsymbol{\phi}^\ast,\boldsymbol{\lambda}^\ast)
\right]
\le
\frac{1}{T}\cdot\frac{1}{D(\eta)}
\left[
\frac{1-20\eta^2L^2}{2\eta}\,\|\boldsymbol{\lambda}^{0}-\boldsymbol{\lambda}^\ast\|^2
+
\frac{n}{m}\,4\eta(1+\eta L)\,G_0
\right].
\]
In particular, if $\eta L$ is sufficiently small so that $D(\eta)=\Theta(1)$ and
$1-20\eta^2L^2=\Theta(1)$, then
\[
\mathbb{E}\!\left[
f(\bar{\boldsymbol{\phi}}^{T},\bar{\boldsymbol{\gamma}}^{T-1})
-
f(\boldsymbol{\phi}^\ast,\boldsymbol{\lambda}^\ast)
\right]
=
\frac{1}{T}\,
O\!\left(
\frac{1}{\eta}\|\boldsymbol{\lambda}^{0}-\boldsymbol{\lambda}^\ast\|^2
+
\frac{n}{m}\eta\,G_0
\right).
\]
\end{proof}
We use the following lemmas to establish Lemma~\ref{lemma:main_lemma}, which in turn implies Theorem~\ref{thm:converge1.1}.
\begin{lemma} \label{lemma_lsmooth2}
If $f$ is convex and $L$--smooth, then for all $\boldsymbol{x},\boldsymbol{y},\boldsymbol{z}$,
\begin{align}
\label{eq:coercive-11}
-\langle \nabla f(\boldsymbol{x}),\, \boldsymbol{z}-\boldsymbol{y}\rangle
\;\le\;
-\,f(\boldsymbol{z})+f(\boldsymbol{y})+\frac{L}{2}\|\boldsymbol{z}-\boldsymbol{x}\|^2,\\
\label{eq:coercive-12}
\frac{1}{2L}\,\|\nabla f(\boldsymbol{y})-\nabla f(\boldsymbol{x})\|^2
\;\le\;
f(\boldsymbol{y})-f(\boldsymbol{x})-\langle \nabla f(\boldsymbol{x}),\boldsymbol{y}-\boldsymbol{x}\rangle .
\end{align}
\begin{proof}
By $L$--smoothness,
\[
f(\boldsymbol{y})
\;\le\;
f(\boldsymbol{x})+\langle \nabla f(\boldsymbol{x}),\,\boldsymbol{y}-\boldsymbol{x}\rangle
+\frac{L}{2}\|\boldsymbol{y}-\boldsymbol{x}\|^2 ,
\]
and
\[
f(\boldsymbol{z})
\;\le\;
f(\boldsymbol{x})+\langle \nabla f(\boldsymbol{x}),\,\boldsymbol{z}-\boldsymbol{x}\rangle
+\frac{L}{2}\|\boldsymbol{z}-\boldsymbol{x}\|^2 .
\]
Hence
\[
f(\boldsymbol{y})
\;\le\;
f(\boldsymbol{z})-\langle \nabla f(\boldsymbol{x}),\,\boldsymbol{z}-\boldsymbol{y}\rangle
+\frac{L}{2}\|\boldsymbol{z}-\boldsymbol{x}\|^2 ,
\]
which is equivalent to
\[
-\langle \nabla f(\boldsymbol{x}),\,\boldsymbol{z}-\boldsymbol{y}\rangle
\;\le\;
-\,f(\boldsymbol{z})+f(\boldsymbol{y})
+\frac{L}{2}\|\boldsymbol{z}-\boldsymbol{x}\|^2 .
\]

Next, to prove \eqref{eq:coercive-12}, fix $\boldsymbol{x},\boldsymbol{y}\in\mathbb{R}^d$. For any $\boldsymbol{z}\in\mathbb{R}^d$, by convexity of $f$ we have
\[
f(\boldsymbol{x})-f(\boldsymbol{z})\le \langle \nabla f(\boldsymbol{x}),\boldsymbol{x}-\boldsymbol{z}\rangle .
\]
Moreover, $L$–smoothness and the standard quadratic upper bound (descent lemma) give
\[
f(\boldsymbol{z})\le f(\boldsymbol{y})+\langle \nabla f(\boldsymbol{y}),\boldsymbol{z}-\boldsymbol{y}\rangle+\frac{L}{2}\|\boldsymbol{z}-\boldsymbol{y}\|^2 ,
\]
or equivalently,
\[
f(\boldsymbol{z})-f(\boldsymbol{y})\le \langle \nabla f(\boldsymbol{y}),\boldsymbol{z}-\boldsymbol{y}\rangle+\frac{L}{2}\|\boldsymbol{z}-\boldsymbol{y}\|^2 .
\]

Combining these inequalities, for every $\boldsymbol{z}$ we obtain
\begin{align}
f(\boldsymbol{x})-f(\boldsymbol{y})
&= f(\boldsymbol{x})-f(\boldsymbol{z})+f(\boldsymbol{z})-f(\boldsymbol{y}) \nonumber\\
&\le \langle \nabla f(\boldsymbol{x}),\boldsymbol{x}-\boldsymbol{z}\rangle
   +\langle \nabla f(\boldsymbol{y}),\boldsymbol{z}-\boldsymbol{y}\rangle
   +\frac{L}{2}\|\boldsymbol{z}-\boldsymbol{y}\|^2 .
\label{eq:RHS-z}
\end{align}

To obtain the tightest upper bound within this family, we minimize the
right–hand side of \eqref{eq:RHS-z} with respect to $\boldsymbol{z}$.
The objective is a convex quadratic function in $\boldsymbol{z}$; setting its gradient to
zero yields
\[
\boldsymbol{0} = -\nabla f(\boldsymbol{x}) + \nabla f(\boldsymbol{y}) + L(\boldsymbol{z}-\boldsymbol{y})
\quad\Longrightarrow\quad
\boldsymbol{z}^\star = \boldsymbol{y} - \frac{1}{L}\bigl(\nabla f(\boldsymbol{y})-\nabla f(\boldsymbol{x})\bigr).
\]

Substituting $\boldsymbol{z}=\boldsymbol{z}^\star$ into \eqref{eq:RHS-z} and reorganizing terms, we get
\begin{align*}
f(\boldsymbol{x})-f(\boldsymbol{y})
&\le
\bigl\langle \nabla f(\boldsymbol{x}),\boldsymbol{x}-\boldsymbol{z}^\star\bigr\rangle
+\bigl\langle \nabla f(\boldsymbol{y}),\boldsymbol{z}^\star-\boldsymbol{y}\bigr\rangle
+\frac{L}{2}\|\boldsymbol{z}^\star-\boldsymbol{y}\|^2 \\[2pt]
&=
\bigl\langle \nabla f(\boldsymbol{x}),\boldsymbol{x}-\boldsymbol{y}\bigr\rangle
-\frac{1}{L}\bigl\|\nabla f(\boldsymbol{y})-\nabla f(\boldsymbol{x})\bigr\|^2
+\frac{1}{2L}\bigl\|\nabla f(\boldsymbol{y})-\nabla f(\boldsymbol{x})\bigr\|^2 \\[2pt]
&=
\bigl\langle \nabla f(\boldsymbol{x}),\boldsymbol{x}-\boldsymbol{y}\bigr\rangle
-\frac{1}{2L}\bigl\|\nabla f(\boldsymbol{y})-\nabla f(\boldsymbol{x})\bigr\|^2 .
\end{align*}

Rearranging gives
\[
\frac{1}{2L}\,\|\nabla f(\boldsymbol{y})-\nabla f(\boldsymbol{x})\|^2
\le
f(\boldsymbol{y})-f(\boldsymbol{x})-\langle \nabla f(\boldsymbol{x}),\boldsymbol{y}-\boldsymbol{x}\rangle,
\]
which is exactly \eqref{eq:coercive-12}.
\end{proof}
\end{lemma}

\begin{lemma} \label{lemma_lsum}
\begin{align}
    \frac{1}{2Ln}\sum_{i=1}^n
\bigl\|\nabla f_i(\boldsymbol{\phi}_i,\boldsymbol{\lambda})
    - \nabla f_i(\boldsymbol{\phi}_i^{\ast},\boldsymbol{\lambda}^{\ast})\bigr\|^2 \leq
    f(\boldsymbol{\phi},\boldsymbol{\lambda}) - f(\boldsymbol{\phi}^{\ast},\boldsymbol{\lambda}^{\ast}), \quad \forall \boldsymbol{\phi}, \boldsymbol{\lambda}
\end{align}
\begin{proof}
According to \eqref{eq:coercive-12} and the fact that $(\boldsymbol{\phi}^{\ast},\boldsymbol{\lambda}^{\ast})$ is the minimizer of $f(\boldsymbol{\phi},\boldsymbol{\lambda})$, we have
\begin{align*}
\frac{1}{2L}\,\bigl\|\nabla f_i(\boldsymbol{\phi}_i,\boldsymbol{\lambda})
    - \nabla f_i(\boldsymbol{\phi}_i^{\ast},\boldsymbol{\lambda}^{\ast})\bigr\|^2
\le
 f_i(\boldsymbol{\phi}_i,\boldsymbol{\lambda})
    - f_i(\boldsymbol{\phi}_i^{\ast},\boldsymbol{\lambda}^{\ast}), \quad \forall \, i \in [n]. 
\end{align*}
Sum them up, we have
\begin{align*}
    \frac{1}{2Ln}\sum_{i=1}^n
\bigl\|\nabla f_i(\boldsymbol{\phi}_i,\boldsymbol{\lambda})
    - \nabla f_i(\boldsymbol{\phi}_i^{\ast},\boldsymbol{\lambda}^{\ast})\bigr\|^2 \leq  f(\boldsymbol{\phi},\boldsymbol{\lambda})
    - f(\boldsymbol{\phi}^{\ast},\boldsymbol{\lambda}^{\ast}), \; \forall \boldsymbol{\phi}, \boldsymbol{\lambda}.
\end{align*}
\end{proof}
\end{lemma}

\begin{lemma} \label{lemma:ht}
Algorithm~\ref{alg:meta} and \ref{alg:oracle_1} satisfies
\begin{align}
\boldsymbol{h}^t
=
\frac{1}{n}
\sum_{i=1}^n
\eta\,\boldsymbol{u}_i^{t},
\quad
\mathbb{E}\bigl\|\boldsymbol{h}^{t}\bigr\|^2
\;\le\;
\eta^2\,G_t.
\end{align}
\begin{proof}
By Algorithm~\ref{alg:meta} and \ref{alg:oracle_1}, we have
\begin{align*}
\boldsymbol{h}^t
&=
\boldsymbol{h}^{t-1}
+
\frac{1}{|S_t|}
\sum_{i\in S_t}
\bigl(\boldsymbol{\lambda}_i^t-\boldsymbol{\lambda}_0^{t-1}\bigr),\\
&=
\boldsymbol{h}^{t-1}
+
\frac{1}{|S_t|}
\sum_{i\in S_t}
\eta\,\bigl(\boldsymbol{u}_i^{t}-\boldsymbol{u}_i^{t-1}\bigr).
\end{align*}
Since
\(
\boldsymbol{h}^0=\frac{1}{n}\sum_{i=1}^n\boldsymbol{\mu}_i^0
\)
and
\(
\boldsymbol{u}_i^{0}
=
-\,\nabla_{\boldsymbol{\lambda}_i}
f_i(\boldsymbol{\phi}_i^{0},\boldsymbol{\lambda}_i^{0})
\)
for all $i$, we obtain the first equality.
Next, to prove the second inequality, we have
\begin{align*}
\mathbb{E}\bigl\|\boldsymbol{h}^{t}\bigr\|^2
&=
\mathbb{E}
\Bigl\|
\frac{1}{n}
\sum_{i=1}^n
\eta\,\boldsymbol{u}_i^{t}
\Bigr\|^2,\\
&=
\eta^2\,
\mathbb{E}
\Bigl\|
\frac{1}{n}
\sum_{i=1}^n
\Bigl(
\nabla_{\boldsymbol{\lambda}_i}
f_i(\boldsymbol{\phi}_i^{t},\boldsymbol{\lambda}_i^{t})
-
\nabla_{\boldsymbol{\lambda}_i}
f_i(\boldsymbol{\phi}_i^{\ast},\boldsymbol{\lambda}^{\ast})
\Bigr)
\Bigr\|^2,\\[4pt]
&\le
\frac{\eta^2}{n}
\sum_{i=1}^n
\mathbb{E}
\Bigl\|
\nabla_{\boldsymbol{\lambda}_i}
f_i(\boldsymbol{\phi}_i^{t},\boldsymbol{\lambda}_i^{t})
-
\nabla_{\boldsymbol{\lambda}_i}
f_i(\boldsymbol{\phi}_i^{\ast},\boldsymbol{\lambda}^{\ast})
\Bigr\|^2,\\[4pt]
&\le
\eta^2\,G_t.
\end{align*}
\end{proof}
\end{lemma}
\begin{lemma}[Unbiasedness of Uniform Sampling]
\label{lemma:unbiased_sampling}
Let $\mathcal{F}_{t-1}$ denote the filtration generated by all randomness up to iteration $t-1$.
Assume that at iteration $t$, the index set $S_t \subset [n]$ is sampled uniformly at random
from all subsets of $[n]$ with fixed cardinality $|S_t|=m$, independently of $\mathcal{F}_{t-1}$.
Let $\{X_i^t\}_{i=1}^n$ be a collection of random vectors that are $\mathcal{F}_{t-1}$-measurable.
Then,
\begin{align}
\mathbb{E}\!\left[
\frac{1}{|S_t|}
\sum_{i\in S_t} X_i^t
\,\middle|\, \mathcal{F}_{t-1}
\right]
=
\frac{1}{n}\sum_{i=1}^n X_i^t,
\end{align}
and consequently,
\begin{align}
\mathbb{E}\!\left[
\frac{1}{|S_t|}
\sum_{i\in S_t} X_i^t
\right]
=
\frac{1}{n}\sum_{i=1}^n \mathbb{E}[X_i^t].
\end{align}
\begin{proof}
Conditioned on $\mathcal{F}_{t-1}$, the vectors $\{X_i^t\}_{i=1}^n$ are deterministic.
Since $S_t$ is sampled uniformly at random from all subsets of $[n]$ of size $m$,
each index $i\in[n]$ is included in $S_t$ with probability $m/n$.
Therefore,
\[
\mathbb{E}\!\left[
\frac{1}{|S_t|}
\sum_{i\in S_t} X_i^t
\,\middle|\, \mathcal{F}_{t-1}
\right]
=
\frac{1}{m}
\sum_{i=1}^n
\mathbb{P}(i\in S_t)\,X_i^t
=
\frac{1}{n}\sum_{i=1}^n X_i^t.
\]
Taking expectation over $\mathcal{F}_{t-1}$ yields the desired result.
\end{proof}
\end{lemma}

\begin{lemma} \label{lemma:gamma_expectation}
    Algorithm~\ref{alg:meta} and \ref{alg:oracle_1} satisfies
\begin{align}
    \mathbb{E}\!\left[\,\boldsymbol{\gamma}^{t}-\boldsymbol{\gamma}^{t-1}\right]=
    \frac{\eta}{n}\sum_{i=1}^n
\mathbb{E}\!\left[
-\nabla_{\!\boldsymbol{\lambda}_i}
f_i\!\left(\tilde{\boldsymbol{\phi}}_i^{t}, \tilde{\boldsymbol{\lambda}}_i^{t}\right)
\right],
\end{align}
\begin{proof}
\begin{align*}
\mathbb{E}\!\left[\,\boldsymbol{\gamma}^{t}-\boldsymbol{\gamma}^{t-1}\right]
&=
\mathbb{E}\!\left[
\left(
\frac{1}{m}\sum_{i\in\mathcal{S}_t}\boldsymbol{\lambda}_i^{t}
\right)
-\boldsymbol{\lambda}_0^{t-1}
+\boldsymbol{h}^{t-1}
\right], \quad \bigl(\text{by \eqref{eq:gamma_def}}\bigr)\\[4pt]
&=
\mathbb{E}\!\left[
\frac{\eta}{m}\sum_{i\in\mathcal{S}_t}
\left(
\boldsymbol{\mu}_i^{t}
-\boldsymbol{\mu}_i^{t-1}
+\frac{1}{\eta}\boldsymbol{h}^{t-1}
\right)
\right],\\[4pt]
&=
\frac{\eta}{m}\mathbb{E}\!\left[
\sum_{i\in\mathcal{S}_t}
\left(
\nabla_{\!\boldsymbol{\lambda}_i} f_i(\boldsymbol{\phi}_i^{t-1},\boldsymbol{\lambda}_i^{t-1})
-
\nabla_{\!\boldsymbol{\lambda}_i} f_i(\boldsymbol{\phi}_i^{t},\boldsymbol{\lambda}_i^{t})
+\frac{1}{\eta}\boldsymbol{h}^{t-1}
\right)
\right], \quad \bigl(\text{by \eqref{eq:dual_update}}\bigr)\\[4pt]
&=
\frac{\eta}{m}\mathbb{E}\!\left[
\sum_{i\in\mathcal{S}_t}
\left(
\nabla_{\!\boldsymbol{\lambda}_i} f_i(\boldsymbol{\phi}_i^{t-1},\boldsymbol{\lambda}_i^{t-1})
-
\nabla_{\!\boldsymbol{\lambda}_i} f_i(\tilde{\boldsymbol{\phi}}_i^{t},\tilde{\boldsymbol{\lambda}}_i^{t})
+\frac{1}{\eta}\boldsymbol{h}^{t-1}
\right)
\right], \quad \bigl(\text{by \eqref{eq:true_equal_virtual}}\bigr)\\[4pt]
&=
\frac{\eta}{n}\mathbb{E}\!\left[
\sum_{i=1}^n
\left(
\nabla_{\!\boldsymbol{\lambda}_i} f_i(\boldsymbol{\phi}_i^{t-1},\boldsymbol{\lambda}_i^{t})
-
\nabla_{\!\boldsymbol{\lambda}_i} f_i(\tilde{\boldsymbol{\phi}}_i^{t},\tilde{\boldsymbol{\lambda}}_i^{t-1})
+\frac{1}{\eta}\boldsymbol{h}^{t-1}
\right)
\right],\quad \bigl(\text{by Lemma \ref{lemma:unbiased_sampling}}\bigr)\\[4pt]
&=
\frac{\eta}{n}\sum_{i=1}^n
\mathbb{E}\!\left[
-\nabla_{\!\boldsymbol{\lambda}_i}
f_i\!\left(\tilde{\boldsymbol{\phi}}_i^{t}, \tilde{\boldsymbol{\lambda}}_i^{t}\right)
\right]. \quad \bigl(\text{by Lemma \ref{lemma:ht} and \eqref{eq:dual_update}}\bigr)\\
\end{align*}
\end{proof}
\end{lemma}

\begin{lemma} \label{lemma:gamma_bound}
    Algorithm~\ref{alg:meta} and \ref{alg:oracle_1} satisfies 
\begin{align} \label{eq:gamma_bound}
  \mathbb{E}\!\left\|\,\boldsymbol{\gamma}^{t}-\boldsymbol{\gamma}^{t-1}\right\|^{2} \leq \epsilon_t 
\end{align}

\begin{proof}
    \begin{align*}
        \mathbb{E}\!\left\|\,\boldsymbol{\gamma}^{t}-\boldsymbol{\gamma}^{t-1}\right\|^{2}
& =
\mathbb{E}\;\Bigl\|
\frac{1}{|S_t|}\sum_{i\in S_{t}}
\left(\boldsymbol{\lambda}_{i}^{t}-\boldsymbol{\gamma}^{t-1}\right)
\Bigr\|^{2},\\
&\le
\frac{1}{|S_t|}\,
\mathbb{E}\!\left[
\sum_{i\in S_{t}}
\left\|\boldsymbol{\lambda}_{i}^{t}-\boldsymbol{\gamma}^{t-1}\right\|^{2}
\right], \quad \bigl(\text{by Jensen's inequality}\bigr)\\
&=
\frac{1}{|S_t|}\,
\mathbb{E}\!\left[
\sum_{i\in S_{t}}
\left\|\tilde{\boldsymbol{\lambda}}_{i}^{t}-\boldsymbol{\gamma}^{t-1}\right\|^{2}
\right]\\
& =
\frac{1}{n}
\sum_{i=1}^n
\mathbb{E}\!\left\|
\tilde{\boldsymbol{\lambda}}_{i}^{t}-\boldsymbol{\gamma}^{t-1}
\right\|^{2},  \quad \bigl(\text{by Lemma \ref{lemma:unbiased_sampling}}\bigr)\\
&=
\epsilon_{t}.
    \end{align*}
\end{proof}
\end{lemma}


\begin{lemma}\label{lemma:sum_norm_bound}
For any vectors $\{\boldsymbol{x}_k\}_{k=1}^m$ in a Euclidean space,
\begin{align}
\Bigl\|\sum_{k=1}^m \boldsymbol{x}_k\Bigr\|^2
\;\le\;
m\sum_{k=1}^m \|\boldsymbol{x}_k\|^2.
\end{align}
\end{lemma}
\begin{proof}
By the Cauchy--Schwarz inequality,
\begin{align*}
\Bigl\|\sum_{k=1}^m \boldsymbol{x}_k\Bigr\|
&\le
\sum_{k=1}^m \|\boldsymbol{x}_k\|
\;=\;
\boldsymbol{1}^{\top}
\begin{bmatrix}
\|\boldsymbol{x}_1\|\\
\vdots\\
\|\boldsymbol{x}_m\|
\end{bmatrix}
\;\le\;
\|\boldsymbol{1}\|\,
\left\|
\begin{bmatrix}
\|\boldsymbol{x}_1\|\\
\vdots\\
\|\boldsymbol{x}_m\|
\end{bmatrix}
\right\|
\;=\;
\sqrt{m}\left(\sum_{k=1}^m \|\boldsymbol{x}_k\|^2\right)^{1/2}.
\end{align*}
Squaring both sides yields
\[
\Bigl\|\sum_{k=1}^m \boldsymbol{x}_k\Bigr\|^2
\;\le\;
m\sum_{k=1}^m \|\boldsymbol{x}_k\|^2,
\]
which completes the proof.
\end{proof}

\begin{lemma}
\begin{align} \label{eq:epsilon_bound}
(1-4L^2\eta^2)\epsilon_t
\;\le\;
8\eta^2 G_{t-1}
+
8L\eta^2
\Bigl(
f(\tilde{\boldsymbol{\phi}}^{t},\boldsymbol{\gamma}^{t-1})
-
f(\boldsymbol{\phi}^\ast,\boldsymbol{\lambda}^\ast)
\Bigr)
.
\end{align}
\begin{proof}
\begin{align*}
\epsilon_t
&=
\frac{1}{n}\sum_{i=1}^n
\mathbb{E}\!\left\|
\tilde{\boldsymbol{\lambda}}_i^{t}
-
\boldsymbol{\gamma}^{t-1}
\right\|^2
=
\frac{1}{n}\sum_{i=1}^n
\mathbb{E}\!\left\|
\tilde{\boldsymbol{\lambda}}_i^{t}
-
\boldsymbol{\lambda}_0^{t-1}
+
\boldsymbol{h}^{t-1}
\right\|^2,\\
&=
\frac{\eta^2}{n}
\sum_{i=1}^n
\mathbb{E}\!\left\|
\nabla_{\boldsymbol{\lambda}_i}
f_i(\boldsymbol{\phi}_i^{t-1},\boldsymbol{\lambda}_i^{t-1})
-
\nabla_{\boldsymbol{\lambda}_i}
f_i(\tilde{\boldsymbol{\phi}}_i^{t},\tilde{\boldsymbol{\lambda}}_i^{t})
+
\frac{\boldsymbol{h}^{t-1}}{\eta}
\right\|^2, \quad \bigl(\text{by \eqref{eq:true_equal_virtual}}\bigr)\\
&=
\frac{\eta^2}{n}
\sum_{i=1}^n
\mathbb{E}\Bigl\|
\nabla_{\boldsymbol{\lambda}_i}
f_i(\boldsymbol{\phi}_i^{t-1},\boldsymbol{\lambda}_i^{t-1})
-
\nabla_{\boldsymbol{\lambda}_i}
f_i(\boldsymbol{\phi}_i^\ast,\boldsymbol{\lambda}^\ast)
\\[-2pt]
&\qquad\qquad\qquad
+
\nabla_{\boldsymbol{\lambda}_i}
f_i(\boldsymbol{\phi}_i^\ast,\boldsymbol{\lambda}^\ast)
-
\nabla_{\boldsymbol{\lambda}_i}
f_i(\tilde{\boldsymbol{\phi}}_i^{t},\boldsymbol{\gamma}^{t-1})
\\[-2pt]
&\qquad\qquad\qquad
+
\nabla_{\boldsymbol{\lambda}_i}
f_i(\tilde{\boldsymbol{\phi}}_i^{t},\boldsymbol{\gamma}^{t-1})
-
\nabla_{\boldsymbol{\lambda}_i}
f_i(\tilde{\boldsymbol{\phi}}_i^{t},\tilde{\boldsymbol{\lambda}}_i^{t})
+
\frac{\boldsymbol{h}^{t-1}}{\eta}
\Bigr\|^2.\\
\end{align*}
By Lemma \ref{lemma:sum_norm_bound}, we have
\begin{align*}
\epsilon_t
&\le
\frac{4\eta^2}{n}
\sum_{i=1}^n
\mathbb{E}\!\left\|
\nabla_{\boldsymbol{\lambda}_i}
f_i(\boldsymbol{\phi}_i^{t-1},\boldsymbol{\lambda}_i^{t-1})
-
\nabla_{\boldsymbol{\lambda}_i}
f_i(\boldsymbol{\phi}_i^\ast,\boldsymbol{\lambda}^\ast)
\right\|^2\\
&\quad+
\frac{4\eta^2}{n}
\sum_{i=1}^n
\mathbb{E}\!\left\|
\nabla_{\boldsymbol{\lambda}_i}
f_i(\boldsymbol{\phi}_i^\ast,\boldsymbol{\lambda}^\ast)
-
\nabla_{\boldsymbol{\lambda}_i}
f_i(\tilde{\boldsymbol{\phi}}_i^{t},\boldsymbol{\gamma}^{t-1})
\right\|^2\\
&\quad+
\frac{4\eta^2}{n}
\sum_{i=1}^n
\mathbb{E}\!\left\|
\nabla_{\boldsymbol{\lambda}_i}
f_i(\tilde{\boldsymbol{\phi}}_i^{t},\boldsymbol{\gamma}^{t-1})
-
\nabla_{\boldsymbol{\lambda}_i}
f_i(\tilde{\boldsymbol{\phi}}_i^{t},\tilde{\boldsymbol{\lambda}}_i^{t})
\right\|^2
+
4\|\boldsymbol{h}^{t-1}\|^2, \quad \left(\text{by Lemma~\ref{lemma:sum_norm_bound}} \right)\\
&\le
\frac{4\eta^2}{n}
\sum_{i=1}^n
\mathbb{E}\!\left\|
\nabla_{\boldsymbol{\lambda}_i}
f_i(\boldsymbol{\phi}_i^{t-1},\boldsymbol{\lambda}_i^{t-1})
-
\nabla_{\boldsymbol{\lambda}_i}
f_i(\boldsymbol{\phi}_i^\ast,\boldsymbol{\lambda}^\ast)
\right\|^2\\
&\quad+
\frac{4\eta^2}{n}
\sum_{i=1}^n
\mathbb{E}\!\left\|
\nabla_{\boldsymbol{\lambda}_i}
f_i(\boldsymbol{\phi}_i^\ast,\boldsymbol{\lambda}^\ast)
-
\nabla_{\boldsymbol{\lambda}_i}
f_i(\tilde{\boldsymbol{\phi}}_i^{t},\boldsymbol{\gamma}^{t-1})
\right\|^2\\
&\quad+
\frac{4\eta^2}{n}
\sum_{i=1}^n
\mathbb{E}\!\left\|
\nabla_{\boldsymbol{\lambda}_i}
f_i(\tilde{\boldsymbol{\phi}}_i^{t},\boldsymbol{\gamma}^{t-1})
-
\nabla_{\boldsymbol{\lambda}_i}
f_i(\tilde{\boldsymbol{\phi}}_i^{t},\tilde{\boldsymbol{\lambda}}_i^{t})
\right\|^2
+
4\eta^2\,G_{t-1}, \quad \bigl(\text{by Lemma \ref{lemma:ht}}\bigr)\\
&\le
8\eta^2\,G_{t-1}
+
4L^2\eta^2\,\epsilon_t
+
8L\eta^2
\Bigl(
f(\tilde{\boldsymbol{\phi}}^{t},\boldsymbol{\gamma}^{t-1})
-
f(\boldsymbol{\phi}^\ast,\boldsymbol{\lambda}^\ast)
\Bigr). \quad \bigl(\text{by Lemma \ref{lemma_lsum}}\bigr)
\end{align*}
\end{proof}
\end{lemma}

\begin{lemma}
\begin{align} \label{eq:G_bound}
    G_t \leq \Bigl(1-\frac{m}{n}\Bigr)\,G_{t-1}
+
\frac{2L^2m}{n}\,\epsilon_t
+
\frac{4Lm}{n}\,
\mathbb{E}\Bigl[
f(\tilde{\boldsymbol{\phi}}^t,\boldsymbol{\gamma}^{t-1})
-
f(\boldsymbol{\phi}^\ast,\boldsymbol{\lambda}^\ast)
\Bigr].
\end{align}
\begin{proof}
\begin{align*}
    G_{t}
&=
\frac{1}{n}\sum_{i=1}^n
\mathbb{E}
\bigl\|
\nabla
f_i(\boldsymbol{\phi}_i^{t},\boldsymbol{\lambda}_i^{t})
-
\nabla
f_i(\boldsymbol{\phi}_i^\ast,\boldsymbol{\lambda}^\ast)
\bigr\|^2,
\quad \bigl(\text{by definition of } G_{t}\bigr)\\
&=
\Bigl(1-\frac{m}{n}\Bigr)
\cdot
\frac{1}{n}\sum_{i=1}^n
\mathbb{E}
\bigl\|
\nabla f_i(\boldsymbol{\phi}_i^{t-1},\boldsymbol{\lambda}^{t-1})
-
\nabla f_i(\boldsymbol{\phi}_i^\ast,\boldsymbol{\lambda}^\ast)
\bigr\|^2\\
&\quad+
\frac{m}{n}\cdot
\frac{1}{n}\sum_{i=1}^n
\mathbb{E}
\bigl\|
\nabla f_i(\tilde{\boldsymbol{\phi}}_i^{t},\tilde{\boldsymbol{\lambda}}_i^{t})
-
\nabla
f_i(\boldsymbol{\phi}_i^\ast,\boldsymbol{\lambda}^\ast)
\bigr\|^2,
\quad \bigl(\text{by uniform sampling of } S_t,\ |S_t|=m\bigr)\\
&=
\Bigl(1-\frac{m}{n}\Bigr)\,G_{t-1}
+
\frac{m}{n}\cdot
\frac{1}{n}\sum_{i=1}^n
\mathbb{E}
\bigl\|
\nabla f_i(\tilde{\boldsymbol{\phi}}_i^{t},\tilde{\boldsymbol{\lambda}}_i^{t})
-
\nabla f_i(\tilde{\boldsymbol{\phi}}_i^{t},\boldsymbol{\gamma}^{t-1})
+
\nabla f_i(\tilde{\boldsymbol{\phi}}_i^{t},\boldsymbol{\gamma}^{t-1})
-
\nabla
f_i(\boldsymbol{\phi}_i^\ast,\boldsymbol{\lambda}^\ast)
\bigr\|^2,\\
&\le
\Bigl(1-\frac{m}{n}\Bigr)\,G_{t-1}
+
\frac{2m}{n^2}
\sum_{i=1}^n
\mathbb{E}
\bigl\|
\nabla
f_i(\tilde{\boldsymbol{\phi}}_i^{t},\tilde{\boldsymbol{\lambda}}_i^{t})
-
\nabla
f_i(\tilde{\boldsymbol{\phi}}_i^{t},\boldsymbol{\gamma}^{t-1})
\bigr\|^2\\
&\quad+
\frac{2m}{n^2}
\sum_{i=1}^n
\mathbb{E}
\bigl\|
\nabla
f_i(\tilde{\boldsymbol{\phi}}_i^{t},\boldsymbol{\gamma}^{t-1})
-
\nabla
f_i(\boldsymbol{\phi}_i^\ast,\boldsymbol{\lambda}^\ast)
\bigr\|^2,
\quad \bigl(\text{by Lemma \ref{lemma:sum_norm_bound}}\bigr)\\
&\le
\Bigl(1-\frac{m}{n}\Bigr)\,G_{t-1}
+
\frac{2L^2m}{n^2}
\sum_{i=1}^n
\mathbb{E}
\bigl\|
\tilde{\boldsymbol{\lambda}}_i^{t}
-
\boldsymbol{\gamma}^{t-1}
\bigr\|^2
+
\frac{4Lm}{n}
\mathbb{E}\Bigl[
f(\tilde{\boldsymbol{\phi}}^t,\boldsymbol{\gamma}^{t-1})
-
f(\boldsymbol{\phi}^\ast,\boldsymbol{\lambda}^\ast)
\Bigr],
\quad \bigl(\text{by Lemma \ref{lemma_lsum}}\bigr)\\
&=
\Bigl(1-\frac{m}{n}\Bigr)\,G_{t-1}
+
\frac{2L^2m}{n}\,\epsilon_t
+
\frac{4Lm}{n}
\mathbb{E}\Bigl[
f(\tilde{\boldsymbol{\phi}}^t,\boldsymbol{\gamma}^{t-1})
-
f(\boldsymbol{\phi}^\ast,\boldsymbol{\lambda}^\ast)
\Bigr].
\end{align*}
\end{proof}
\end{lemma}

\begin{lemma} \label{lemma:gamma_lambda}
\begin{align}
  \label{eq:gamma_lambda}
\mathbb{E}\!\left\|\boldsymbol{\gamma}^t-\boldsymbol{\lambda}^\ast\right\|^2    \leq \mathbb{E}\!\left\|\boldsymbol{\gamma}^{t-1}-\boldsymbol{\lambda}^\ast\right\|^2
-
2\eta\,
\mathbb{E}\!\left[
f(\tilde{\boldsymbol{\phi}}^{t},\boldsymbol{\gamma}^{t-1})
-
f(\boldsymbol{\phi}^\ast,\boldsymbol{\lambda}^\ast)
\right]
+
\eta L\,\epsilon_t
+
\mathbb{E}\!\left\|\boldsymbol{\gamma}^t-\boldsymbol{\gamma}^{t-1}\right\|^2.
\end{align}
\begin{proof}
\begin{align*}
\mathbb{E}\!\left\|\boldsymbol{\gamma}^t-\boldsymbol{\lambda}^\ast\right\|^2
&=
\mathbb{E}
\left\|
\boldsymbol{\gamma}^{t-1}-\boldsymbol{\lambda}^\ast
+
\boldsymbol{\gamma}^t-\boldsymbol{\gamma}^{t-1}
\right\|^2, \\[4pt]
&=
\mathbb{E}\!\left\|\boldsymbol{\gamma}^{t-1}-\boldsymbol{\lambda}^\ast\right\|^2
+
2\,\mathbb{E}
\big\langle
\boldsymbol{\gamma}^{t-1}-\boldsymbol{\lambda}^\ast,\;
\boldsymbol{\gamma}^t-\boldsymbol{\gamma}^{t-1}
\big\rangle
+
\mathbb{E}\!\left\|\boldsymbol{\gamma}^t-\boldsymbol{\gamma}^{t-1}\right\|^2, \\[4pt]
&=
\mathbb{E}\!\left\|\boldsymbol{\gamma}^{t-1}-\boldsymbol{\lambda}^\ast\right\|^2
+
\frac{2\eta}{n}
\sum_{i=1}^n
\mathbb{E}
\Big[
\big\langle
\boldsymbol{\gamma}^{t-1}-\boldsymbol{\lambda}^\ast,\;
-\nabla_{\boldsymbol{\lambda}_i}
f_i(\tilde{\boldsymbol{\phi}}_i^{t},\tilde{\boldsymbol{\lambda}}_i^{t})
\big\rangle
\Big]
+
\mathbb{E}\!\left\|\boldsymbol{\gamma}^t-\boldsymbol{\gamma}^{t-1}\right\|^2, \quad \bigl(\text{by Lemma \ref{lemma:gamma_expectation}}\bigr)\\[4pt]
&=
\mathbb{E}\!\left\|\boldsymbol{\gamma}^{t-1}-\boldsymbol{\lambda}^\ast\right\|^2
+
\frac{2\eta}{n}
\sum_{i=1}^n
\mathbb{E} \Big[
\big\langle
\boldsymbol{\gamma}^{t-1}-\boldsymbol{\lambda}^\ast,\;
-\nabla_{\boldsymbol{\lambda}_i}
f_i(\tilde{\boldsymbol{\phi}}_i^{t},\tilde{\boldsymbol{\lambda}}_i^{t})
\big\rangle + 
\big\langle
\tilde{\boldsymbol{\phi}}_i^{t}-\boldsymbol{\phi}_i^\ast,\;
-\nabla_{\boldsymbol{\phi}_i}
f_i(\tilde{\boldsymbol{\phi}}_i^{t},\tilde{\boldsymbol{\lambda}}_i^{t})
\big\rangle
\Big]
\\[4pt]
&+
\mathbb{E}\!\left\|\boldsymbol{\gamma}^t-\boldsymbol{\gamma}^{t-1}\right\|^2, \quad \bigl(\text{since }\boldsymbol{\phi}_i \text{ is minimized}\bigr)\\[4pt]
&\le
\mathbb{E}\!\left\|\boldsymbol{\gamma}^{t-1}-\boldsymbol{\lambda}^\ast\right\|^2
+
\frac{2\eta}{n}
\sum_{i=1}^n
\mathbb{E}
\Big[
f_i(\boldsymbol{\phi}_i^\ast,\boldsymbol{\lambda}^\ast)
-
f_i(\tilde{\boldsymbol{\phi}}_i^{t},\boldsymbol{\gamma}^{t-1})
+
\frac{L}{2}
\big\|
\tilde{\boldsymbol{\lambda}}_i^{t}
-
\boldsymbol{\gamma}^{t-1}
\big\|^2
\Big]
+
\mathbb{E}\!\left\|\boldsymbol{\gamma}^t-\boldsymbol{\gamma}^{t-1}\right\|^2, \quad \bigl(\text{by \eqref{eq:coercive-11}}\bigr)\\[4pt]
&=
\mathbb{E}\!\left\|\boldsymbol{\gamma}^{t-1}-\boldsymbol{\lambda}^\ast\right\|^2
-
2\eta\,
\mathbb{E}\!\left[
f(\tilde{\boldsymbol{\phi}}^{t},\boldsymbol{\gamma}^{t-1})
-
f(\boldsymbol{\phi}^\ast,\boldsymbol{\lambda}^\ast)
\right]
+
\eta L\,\epsilon_t
+
\mathbb{E}\!\left\|\boldsymbol{\gamma}^t-\boldsymbol{\gamma}^{t-1}\right\|^2 .
\end{align*}
\end{proof}
\end{lemma}


\begin{lemma} \label{lemma:main_lemma}
For convex and $L$-smooth $\{f_i(\cdot)\}_{i=1}^n$ functions, Algorithm~\ref{alg:meta} and \ref{alg:oracle_1} satisfies
\begin{align}
\mathbb{E}\!\left\|\boldsymbol{\gamma}^t-\boldsymbol{\lambda}^\ast\right\|^2
+c_1\,G_t
\leq
\mathbb{E}\!\left\|\boldsymbol{\gamma}^{t-1}-\boldsymbol{\lambda}^\ast\right\|^2
+c_1\,G_{t-1}
-c_2\,
\mathbb{E}\!\left[
f(\tilde{\boldsymbol{\phi}}^t,\boldsymbol{\gamma}^{t-1})
-
f(\boldsymbol{\phi}^\ast,\boldsymbol{\lambda}^\ast)
\right].
\end{align}
where $c_1 = \frac{n}{m}\Bigl(\frac{8\eta^2(1+\eta L)}{1-20\eta^2L^2}\Bigr)$ and $c_2 = \frac{2\eta - 40\eta^2L - 80\eta^3L^2}{1-20\eta^2L^2}.$

\begin{proof}
By adding \eqref{eq:gamma_lambda}, \eqref{eq:gamma_bound}, rescaled \eqref{eq:epsilon_bound} and rescaled \eqref{eq:G_bound}, 
we can end up with the desired result.
\end{proof}
\end{lemma}


\subsection{Strongly Convex Anlaysis}
In the strongly convex case, we make the following two assumptions.
\begin{assumption} 
  Each $\nabla f_i(\cdot)$ is $L$-Lipschitz with $L\in (0,\infty), i.e., $ for all $i\in[n]$ and all
$\boldsymbol{x},\boldsymbol{y}$ in the domain,
  \begin{align} 
      \|\nabla f_i(\boldsymbol{x}) - \nabla f_i(\boldsymbol{y})\| \leq L\|\boldsymbol{x}-\boldsymbol{y}\|.
  \end{align}
\end{assumption}
\begin{assumption}
Each $f_i(\cdot)$ is $\mu$-strongly convex. That is, for all $i\in[n]$ and all
$\boldsymbol{x},\boldsymbol{y}$ in the domain,
\begin{align}
f_i(\boldsymbol{y})
\ge
f_i(\boldsymbol{x})
+
\left\langle \nabla f_i(\boldsymbol{x}),\, \boldsymbol{y}-\boldsymbol{x}\right\rangle
+
\frac{\mu}{2}\|\boldsymbol{y}-\boldsymbol{x}\|^2.
\end{align}
\end{assumption}

\begin{theorem}
\label{thm:converge1.2}
For $\mu$-strongly convex and $L$-smooth functions $\{f_i(\cdot)\}_{i=1}^n$.
If we set the step-size $\eta \leq \min\left\{
\frac{1}{200L},\;
\frac{1}{2\mu},\;
\frac{m}{8\mu(n-m)}
\right\}$ in Algorithm~\ref{alg:meta} and \ref{alg:oracle_1}, then the following convergence rate holds.

Define $r = 1+\mu\eta$ and
\[
\hat{\boldsymbol{\phi}}^{T}
:= \frac{1}{R}\sum_{t=1}^T r^{t-1}\tilde{\boldsymbol{\phi}}^{t},
\qquad
\hat{\boldsymbol{\gamma}}^{T-1}
:= \frac{1}{R}\sum_{t=1}^T r^{t-1}\boldsymbol{\gamma}^{t-1}.
\]
Then we have
\begin{align}
\mathbb{E}\!\left[
f\!\left(
\frac{1}{R}\sum_{t=1}^T r^{t-1}
(\tilde{\boldsymbol{\phi}}^{t},\boldsymbol{\gamma}^{t-1})
\right)
-
f(\boldsymbol{\phi}^\ast,\boldsymbol{\lambda}^\ast)
\right]
=
\,
O\!\left(\frac{1}{r^{T-1}}\left(
\frac{1}{\eta}\,\|\boldsymbol{\gamma}^{0}-\boldsymbol{\lambda}^\ast\|^2
+
\frac{n}{m}\,\eta\,G_0
\right)\right).
\end{align}
where $c_1$ and $c_2$ are defined in Lemma~\ref{lemma:main_strong}, and $G_0 = \frac{1}{n}\sum_{i=1}^n
\left\|
\nabla f_i(\boldsymbol{\phi}_i^*,\boldsymbol{\lambda}_i^*)
\right\|^2$.
\begin{proof}
Define $r = (1 + \mu\eta )$ and $R = \sum_{t=0}^{T-1} r^t $.
From Lemma \ref{lemma:main_strong}, for every $t\ge1$ we have
\begin{align*}
c_2 r^{t-1}\,
\mathbb{E}\!\left[
f(\tilde{\boldsymbol{\phi}}^{t},\boldsymbol{\gamma}^{t-1})
-
f(\boldsymbol{\phi}^\ast,\boldsymbol{\lambda}^\ast)
\right]
&\le r^{t-1}
\Bigl(
\mathbb{E}\!\left\|\boldsymbol{\gamma}^{t-1}-\boldsymbol{\lambda}^\ast\right\|^2
+c_1\,G_{t-1}
\Bigr)
- r^{t}
\Bigl(
\mathbb{E}\!\left\|\boldsymbol{\gamma}^{t}-\boldsymbol{\lambda}^\ast\right\|^2
+c_1\,G_{t}
\Bigr).
\end{align*}
Summing the above inequality over $t=1,\ldots,T$ yields the telescoping bound
\begin{align*}
c_2\sum_{t=1}^T r^{t-1}
\mathbb{E}\!\left[
f(\tilde{\boldsymbol{\phi}}^{t},\boldsymbol{\gamma}^{t-1})
-
f(\boldsymbol{\phi}^\ast,\boldsymbol{\lambda}^\ast)
\right]
&\le
\Bigl(
\mathbb{E}\!\left\|\boldsymbol{\gamma}^{0}-\boldsymbol{\lambda}^\ast\right\|^2
+c_1\,G_{0}
\Bigr)
- r^{T}
\Bigl(
\mathbb{E}\!\left\|\boldsymbol{\gamma}^{T}-\boldsymbol{\lambda}^\ast\right\|^2
+c_1\,G_{T}
\Bigr).
\end{align*}
When $\eta \leq \min\left\{
\frac{1}{200L},\;
\frac{1}{2\mu},\;
\frac{m}{8\mu(n-m)}
\right\}$ 
, $c_1$ and $c_2$ are positive, and clearly $G_T\ge0$. Hence
\begin{align*}
\frac{c_2}{R}\sum_{t=1}^T r^{t-1}
\mathbb{E}\!\left[
f(\tilde{\boldsymbol{\phi}}^{t},\boldsymbol{\gamma}^{t-1})
-
f(\boldsymbol{\phi}^\ast,\boldsymbol{\lambda}^\ast)
\right]
&\le \frac{1}{R}\Bigl(
\mathbb{E}\!\left\|\boldsymbol{\gamma}^{0}-\boldsymbol{\lambda}^\ast\right\|^2
+c_1\,G_{0} \Bigr).
\end{align*}
Dividing both sides by $c_2$ gives
\begin{align*}
\frac{1}{R}\sum_{t=1}^T r^{t-1}
\mathbb{E}\!\left[
f(\tilde{\boldsymbol{\phi}}^{t},\boldsymbol{\gamma}^{t-1})
-
f(\boldsymbol{\phi}^\ast,\boldsymbol{\lambda}^\ast)
\right]
&\le
\frac{1}{c_2R}\Bigl(
\mathbb{E}\!\left\|\boldsymbol{\gamma}^{0}-\boldsymbol{\lambda}^\ast\right\|^2
+c_1\,G_{0}
\Bigr).
\end{align*}
By Jensen's inequality and the convexity of $f(\boldsymbol{\phi},\boldsymbol{\lambda})$, we obtain
\begin{align*}
\mathbb{E}\!\left[
f\!\left(
\frac{1}{R}\sum_{t=1}^T r^{t-1}
(\tilde{\boldsymbol{\phi}}^{t},\boldsymbol{\gamma}^{t-1})
\right)
-
f(\boldsymbol{\phi}^\ast,\boldsymbol{\lambda}^\ast)
\right]
&\le
\frac{1}{R}\sum_{t=1}^T r^{t-1}
\mathbb{E}\!\left[
f(\tilde{\boldsymbol{\phi}}^{t},\boldsymbol{\gamma}^{t-1})
-
f(\boldsymbol{\phi}^\ast,\boldsymbol{\lambda}^\ast)
\right] \\
&\le
\frac{1}{c_2R}\Bigl(
\mathbb{E}\!\left\|\boldsymbol{\gamma}^{0}-\boldsymbol{\lambda}^\ast\right\|^2
+
c_1\,G_0
\Bigr).
\end{align*}

Define the weighted averaged iterates
\[
\hat{\boldsymbol{\phi}}^{T}
\;:=\;
\frac{1}{R}\sum_{t=1}^T r^{t-1}\tilde{\boldsymbol{\phi}}^{t},
\qquad
\hat{\boldsymbol{\gamma}}^{T-1}
\;:=\;
\frac{1}{R}\sum_{t=1}^T r^{t-1}\boldsymbol{\gamma}^{t-1}.
\]
Then the above inequality can be written as
\begin{align*}
\mathbb{E}\!\left[
f(\hat{\boldsymbol{\phi}}^{T},\hat{\boldsymbol{\gamma}}^{T-1})
-
f(\boldsymbol{\phi}^\ast,\boldsymbol{\lambda}^\ast)
\right]
\le
\frac{1}{c_2R}\Bigl(
\mathbb{E}\!\left\|\boldsymbol{\gamma}^{0}-\boldsymbol{\lambda}^\ast\right\|^2
+
c_1\,G_0
\Bigr) \leq \frac{1}{c_2 r^{T-1}}\Bigl(
\mathbb{E}\!\left\|\boldsymbol{\gamma}^{0}-\boldsymbol{\lambda}^\ast\right\|^2
+
c_1\,G_0
\Bigr).
\end{align*}
Plug in the value of $c_1$ and $c_2$ in Lemma~\ref{lemma:main_strong}, we have
\begin{align*}
\mathbb{E}\!\left[
f\!\left(
\frac{1}{R}\sum_{t=1}^T r^{t-1}
(\tilde{\boldsymbol{\phi}}^{t},\boldsymbol{\gamma}^{t-1})
\right)
-
f(\boldsymbol{\phi}^\ast,\boldsymbol{\lambda}^\ast)
\right]
=
\frac{1}{r^{T-1}}\,
O\!\left(
\frac{1}{\eta}\,\|\boldsymbol{\gamma}^{0}-\boldsymbol{\lambda}^\ast\|^2
+
\frac{n}{m}\,\eta\,G_0
\right).
\end{align*}
\end{proof}

\end{theorem}

We use the following lemmas to establish Lemma~\ref{lemma:main_strong}, which in turn implies Theorem~\ref{thm:converge1.2}.


\begin{lemma} \label{lemma:gamma_lambda_strong}
\begin{align}
-\frac{1}{n}\sum_{i=1}^n
\mathbb{E}\bigl\|
(\tilde{\boldsymbol{\phi}}_i^{t},\tilde{\boldsymbol{\lambda}}_i^{t})
-
(\boldsymbol{\phi}_i^\ast,\boldsymbol{\lambda}^\ast)
\bigr\|^2
\;\le\;
-\mathbb{E}\bigl\|
\boldsymbol{\gamma}^{t}-\boldsymbol{\lambda}^\ast
\bigr\|^2 .
\end{align}
\end{lemma}

\begin{proof}
We have
\begin{align*}
\mathbb{E}\bigl\|
\boldsymbol{\gamma}^{t}-\boldsymbol{\lambda}^\ast
\bigr\|^2
&=
\mathbb{E}\Bigl\|
\frac{1}{m}\sum_{i\in S_t}
\bigl(\boldsymbol{\lambda}_i^{t}-\boldsymbol{\lambda}^\ast\bigr)
\Bigr\|^2,  \qquad
\bigl(\text{by the definition of $\gamma^t$}\bigr)\\[4pt]
&\le
\mathbb{E}\!\left[
\frac{1}{m}
\sum_{i\in S_t}
\bigl\|
\boldsymbol{\lambda}_i^{t}-\boldsymbol{\lambda}^\ast
\bigr\|^2
\right],
\qquad
\bigl(\text{by Jensen's inequality}\bigr) \\[6pt]
&\le
\frac{1}{m}\,
\mathbb{E}\!\left[
\sum_{i\in S_t}
\bigl(
\|\tilde{\boldsymbol{\lambda}}_i^{t}-\boldsymbol{\lambda}^\ast\|^2
+
\|\tilde{\boldsymbol{\phi}}_i^{t}-\boldsymbol{\phi}_i^\ast\|^2
\bigr)
\right], \\[6pt]
&=
\frac{1}{n}\,
\mathbb{E}\!\left[
\sum_{i=1}^n
\bigl\|
(\tilde{\boldsymbol{\phi}}_i^{t},\tilde{\boldsymbol{\lambda}}_i^{t})
-
(\boldsymbol{\phi}_i^\ast,\boldsymbol{\lambda}^\ast)
\bigr\|^2
\right]. \qquad
\bigl(\text{by Lemma \ref{lemma:unbiased_sampling}}\bigr)
\end{align*}
This completes the proof.
\end{proof}


\begin{lemma} \label{lemma:main_strong}
  For $\mu$-strongly convex and $L_i$-smooth $\{f_i(\cdot)\}_{i=1}^n$ functions, if $\eta \leq \min\left\{
\frac{1}{200L},\;
\frac{1}{2\mu},\;
\frac{m}{8\mu(n-m)}
\right\}$, 
  Algorithm~\ref{alg:meta} and \ref{alg:oracle_1} satisfies,
  \begin{align}
    (1+\mu \eta) \bigl(\mathbb{E}\!\left\|\boldsymbol{\gamma}^t-\boldsymbol{\lambda}^\ast\right\|^2
+c_1\,G_t \bigr)
\leq
\mathbb{E}\!\left\|\boldsymbol{\gamma}^{t-1}-\boldsymbol{\lambda}^\ast\right\|^2
+c_1\,G_{t-1}
-c_2\,
\mathbb{E}\!\left[
f(\tilde{\boldsymbol{\phi}}^t,\boldsymbol{\gamma}^{t-1})
-
f(\boldsymbol{\phi}^\ast,\boldsymbol{\lambda}^\ast)
\right].
  \end{align}
  where $
c_1
=
\frac{
8\eta^2(1+\eta L)
}{
(1-4L^2\eta^2)\Bigl(\frac{m}{n}-\mu\eta(1-\frac{m}{n})\Bigr)
-
16L^2\eta^2\frac{m}{n}(1+\mu\eta)
}$, $c_2 = \frac{2\eta(1-4L^2\eta^2)
-
8L\eta^2(1+\eta L)
-
(1+\mu\eta)c_1\cdot\frac{4Lm}{n}}{1-4L^2\eta^2}$.
\end{lemma}

\begin{proof}
By generalizing Lemma \ref{lemma_lsmooth2} to $\mu$-strongly convex and $L$-smooth function $f_i(\cdot)$, we have
\begin{align} \label{eq:lipschitz_strong}
-\langle \nabla f_i(\boldsymbol{x}),\, \boldsymbol{z}-\boldsymbol{y}\rangle
\;\le\;
-\,f_i(\boldsymbol{z})+f_i(\boldsymbol{y})+\frac{L}{2}\|\boldsymbol{z}-\boldsymbol{x}\|^2 - \frac{\mu}{2}\|\boldsymbol{x} - \boldsymbol{y}\|^2 .
\end{align}
Similar to the proof of Lemma \ref{lemma:gamma_lambda}, since the strongly convex function is also convex, we have
\begin{align*}
\mathbb{E}\!\left\|\boldsymbol{\gamma}^t-\boldsymbol{\lambda}^\ast\right\|^2
&=
\mathbb{E}
\left\|
\boldsymbol{\gamma}^{t-1}-\boldsymbol{\lambda}^\ast
+
\boldsymbol{\gamma}^t-\boldsymbol{\gamma}^{t-1}
\right\|^2, \\[4pt]
&=
\mathbb{E}\!\left\|\boldsymbol{\gamma}^{t-1}-\boldsymbol{\lambda}^\ast\right\|^2
+
2\,\mathbb{E}
\big\langle
\boldsymbol{\gamma}^{t-1}-\boldsymbol{\lambda}^\ast,\;
\boldsymbol{\gamma}^t-\boldsymbol{\gamma}^{t-1}
\big\rangle
+
\mathbb{E}\!\left\|\boldsymbol{\gamma}^t-\boldsymbol{\gamma}^{t-1}\right\|^2, \\[4pt]
&=
\mathbb{E}\!\left\|\boldsymbol{\gamma}^{t-1}-\boldsymbol{\lambda}^\ast\right\|^2
+
\frac{2\eta}{n}
\sum_{i=1}^n
\mathbb{E}
\Big[
\big\langle
\boldsymbol{\gamma}^{t-1}-\boldsymbol{\lambda}^\ast,\;
-\nabla_{\boldsymbol{\lambda}_i}
f_i(\tilde{\boldsymbol{\phi}}_i^{t},\tilde{\boldsymbol{\lambda}}_i^{t})
\big\rangle
\Big]
+
\mathbb{E}\!\left\|\boldsymbol{\gamma}^t-\boldsymbol{\gamma}^{t-1}\right\|^2, \quad \bigl(\text{by Lemma \ref{lemma:gamma_expectation}}\bigr)\\[4pt]
&=
\mathbb{E}\!\left\|\boldsymbol{\gamma}^{t-1}-\boldsymbol{\lambda}^\ast\right\|^2
+
\frac{2\eta}{n}
\sum_{i=1}^n
\mathbb{E} \Big[
\big\langle
\boldsymbol{\gamma}^{t-1}-\boldsymbol{\lambda}^\ast,\;
-\nabla_{\boldsymbol{\lambda}_i}
f_i(\tilde{\boldsymbol{\phi}}_i^{t},\tilde{\boldsymbol{\lambda}}_i^{t})
\big\rangle + 
\big\langle
\tilde{\boldsymbol{\phi}}_i^{t}-\boldsymbol{\phi}_i^\ast,\;
-\nabla_{\boldsymbol{\phi}_i}
f_i(\tilde{\boldsymbol{\phi}}_i^{t},\tilde{\boldsymbol{\lambda}}_i^{t})
\big\rangle
\Big]
\\[4pt]
&\quad +
\mathbb{E}\!\left\|\boldsymbol{\gamma}^t-\boldsymbol{\gamma}^{t-1}\right\|^2, \quad \bigl(\text{since }\boldsymbol{\phi}_i \text{ is minimized and by Lemma \ref{lemma:unbiased_sampling}}\bigr)\\[4pt]
&\le
\mathbb{E}\!\left\|\boldsymbol{\gamma}^{t-1}-\boldsymbol{\lambda}^\ast\right\|^2
+
\frac{2\eta}{n}
\sum_{i=1}^n
\mathbb{E}
\Big[
f_i(\boldsymbol{\phi}_i^\ast,\boldsymbol{\lambda}^\ast)
-
f_i(\tilde{\boldsymbol{\phi}}_i^{t},\boldsymbol{\gamma}^{t-1})
+
\frac{L_i}{2}
\big\|
\tilde{\boldsymbol{\lambda}}_i^{t}
-
\boldsymbol{\gamma}^{t-1}
\big\|^2
 - \frac{\mu}{2}\big\|(\tilde{\boldsymbol{\phi}}_i^t, \tilde{\boldsymbol{\lambda}}_i^t) - (\boldsymbol{\phi}_i^\ast, \boldsymbol{\lambda}_i^\ast) \big\|^2]\\
& \quad +
\mathbb{E}\!\left\|\boldsymbol{\gamma}^t-\boldsymbol{\gamma}^{t-1}\right\|^2, \quad \bigl(\text{by \eqref{eq:coercive-11} and Lemma \ref{lemma:gamma_lambda_strong}}\bigr)\\[4pt]
&\leq
\mathbb{E}\!\left\|\boldsymbol{\gamma}^{t-1}-\boldsymbol{\lambda}^\ast\right\|^2
-
2\eta\,
\mathbb{E}\!\left[
f(\tilde{\boldsymbol{\phi}}^{t},\boldsymbol{\gamma}^{t-1})
-
f(\boldsymbol{\phi}^\ast,\boldsymbol{\lambda}^\ast)
\right]
+
\eta L\,\epsilon_t
-
\mu \eta\; \mathbb{E}\bigl\|
\boldsymbol{\gamma}^{t}-\boldsymbol{\lambda}^\ast
\bigr\|^2\; 
+
\mathbb{E}\!\left\|\boldsymbol{\gamma}^t-\boldsymbol{\gamma}^{t-1}\right\|^2 .
\end{align*}
Rearranging the terms we have
\begin{align}
  (1+ \mu \eta)\mathbb{E}\bigl\|
\boldsymbol{\gamma}^{t}-\boldsymbol{\lambda}^\ast
  \bigr\|^2 \leq \mathbb{E}\!\left\|\boldsymbol{\gamma}^{t-1}-\boldsymbol{\lambda}^\ast\right\|^2
-
2\eta\,
\mathbb{E}\!\left[
f(\tilde{\boldsymbol{\phi}}^{t},\boldsymbol{\gamma}^{t-1})
-
f(\boldsymbol{\phi}^\ast,\boldsymbol{\lambda}^\ast)
\right]
+
\eta L\,\epsilon_t
+
\mathbb{E}\!\left\|\boldsymbol{\gamma}^t-\boldsymbol{\gamma}^{t-1}\right\|^2.
\end{align}
Combining with Lemma \ref{lemma:gamma_bound}, we obtain
\begin{align}\label{eq:diff_strong}
  (1+ \mu \eta)\mathbb{E}\bigl\|
\boldsymbol{\gamma}^{t}-\boldsymbol{\lambda}^\ast
  \bigr\|^2 \leq \mathbb{E}\!\left\|\boldsymbol{\gamma}^{t-1}-\boldsymbol{\lambda}^\ast\right\|^2
-
2\eta\,
\mathbb{E}\!\left[
f(\tilde{\boldsymbol{\phi}}^{t},\boldsymbol{\gamma}^{t-1})
-
f(\boldsymbol{\phi}^\ast,\boldsymbol{\lambda}^\ast)
\right]
+
(1+\eta L)\,\epsilon_t.
\end{align}
By adding \eqref{eq:diff_strong}, rescaled \eqref{eq:epsilon_bound} and rescaled \eqref{eq:G_bound}, 
we can end up with the desired result.
\end{proof}

\subsection{Nonconvex Analysis}
Define
\begin{align*}
G_t 
:= 
\frac{1}{n}\sum_{i=1}^n 
\mathbb{E}\!\left[
\left\|
\boldsymbol{\lambda}_i^t - \boldsymbol{\gamma}^t
\right\|^2
\right], \quad
\epsilon_t 
:= 
\frac{1}{n}\sum_{i=1}^n 
\mathbb{E}\!\left[
\left\|
\tilde{\boldsymbol{\lambda}}_i^t - \boldsymbol{\gamma}^{t-1}
\right\|^2
\right], \quad
\delta_t 
:= 
\frac{1}{n}\sum_{i=1}^n 
\mathbb{E}\!\left[
\left\|
\boldsymbol{\phi}_i^{t-1} - \tilde{\boldsymbol{\phi}}_i^{t}
\right\|^2
\right].
\end{align*}

In the nonconvex case, we make the following two assumptions. 
\begin{assumption} \label{ass:lsmoot_nonconvex}
  Each $\nabla f_i(\cdot)$ is $L$-Lipschitz with $L\in (0,\infty), i.e., $ for $\forall \, i \in [n]$, 
  \begin{align*} 
      \|\nabla f_i(\boldsymbol{x}) - \nabla f_i(\boldsymbol{y})\| \leq L\|\boldsymbol{x}-\boldsymbol{y}\|, \, \forall \, x,\, y \in \mathbb{R}^d
  \end{align*}
\end{assumption}

\begin{assumption} 
For each $i\in[n]$, the function $f_i(\boldsymbol{\phi},\boldsymbol{\lambda})$ is twice continuously differentiable
and there exists
a constant $\mu\in(0,\infty)$ such that for all $i\in[n]$ and all $(\boldsymbol{\phi},\boldsymbol{\lambda})\in\mathbb{R}^{d_\phi}\times\mathbb{R}^{d_\lambda}$,
\[
\nabla^2_{\boldsymbol{\phi}\boldsymbol{\phi}} f_i(\boldsymbol{\phi},\boldsymbol{\lambda}) \succeq \mu \mathbf{I}_{d_\phi}.
\]
\end{assumption}


\begin{theorem}
\label{thm:nonconvex_main}
Let $\omega:=\frac{m}{n}\in(0,1)$ and $\kappa:=\frac{L^2}{\mu^2}$.
Suppose all assumptions in Lemma~\ref{lemma:main_nonconvex} hold.
Assume the step size satisfies
\(
\eta \le \min\Biggl\{
\frac{1}{2\sqrt{2}\,L\sqrt{1+2\kappa}},
\;
\frac{\omega}{256\,L\,(1+\kappa)}
\Biggr\}
\).
Then the constants $c_1$ and $c_2$ defined in Lemma~\ref{lemma:main_nonconvex}
are positive. Moreover, for the iterates
$\{(\boldsymbol{\phi}^t,\boldsymbol{\gamma}^t)\}_{t=0}^T$
generated by the algorithm, we have
\[
\mathbb{E}\left[\frac{1}{T}\sum_{t=1}^T
\Bigl\|
\nabla_{\boldsymbol{\lambda}} f(\boldsymbol{\phi}^{t-1},\boldsymbol{\gamma}^{t-1})
\Bigr\|^2\right]
\;=\;
O\!\left(\frac{1}{T}\frac{1}{\eta}\left(
\mathbb{E}[f(\boldsymbol{\phi}^{0},\boldsymbol{\gamma}^{0})]-f^*
\;+\;
nL\left(1+\frac{L^2}{\mu^2}\right)G_0
\right)\right).
\]
\end{theorem}
\begin{proof}
Lemma~\ref{lemma:main_nonconvex} can be telescoped as,
\begin{align}
& c_2\,
\mathbb{E}\Bigl\|
\nabla_{\boldsymbol{\lambda}} f(\boldsymbol{\phi}^{t-1},\boldsymbol{\gamma}^{t-1})
\Bigr\|^2 \leq \left(\mathbb{E}\!\left[f(\boldsymbol{\phi}^{t-1},\boldsymbol{\gamma}^{t-1})\right] - f^*
+c_1\,G_{t-1}\right) - \left(\mathbb{E}\!\left[f(\boldsymbol{\phi}^t,\boldsymbol{\gamma}^t)\right] - f^*
+c_1\,G_t \right), \\
& c_2\,\sum_{t=1}^T
\mathbb{E}\Bigl\|
\nabla_{\boldsymbol{\lambda}} f(\boldsymbol{\phi}^{t-1},\boldsymbol{\gamma}^{t-1})
\Bigr\|^2 \leq \left(\mathbb{E}\!\left[f(\boldsymbol{\phi}^{0},\boldsymbol{\gamma}^{0})\right] - f^*
+c_1\,G_{0}\right) - \left(\mathbb{E}\!\left[f(\boldsymbol{\phi}^T,\boldsymbol{\gamma}^T)\right] - f^*
+c_1\,G_T \right).
\end{align}
If $\eta$ has a reasonable range, we have $c_1$ and $c_2$ being positive. So we have
\begin{align}
\mathbb{E}\left[\frac{1}{T}\sum_{t=1}^T
\Bigl\|
\nabla_{\boldsymbol{\lambda}} f(\boldsymbol{\phi}^{t-1},\boldsymbol{\gamma}^{t-1})
\Bigr\|^2\right] \leq \frac{1}{T}\frac{1}{c_2}\left(\mathbb{E}\!\left[f(\boldsymbol{\phi}^{0},\boldsymbol{\gamma}^{0})\right] - f^*
+c_1\,G_{0}\right) - \left(\mathbb{E}\!\left[f(\boldsymbol{\phi}^T,\boldsymbol{\gamma}^T)\right] - f^*
+c_1\,G_T \right).
\end{align}
Elimining the negative terms on the RHS, we have
\begin{align}
\mathbb{E}\left[\frac{1}{T}\sum_{t=1}^T
\Bigl\|
\nabla_{\boldsymbol{\lambda}} f(\boldsymbol{\phi}^{t-1},\boldsymbol{\gamma}^{t-1})
\Bigr\|^2\right] \leq \frac{1}{T}\frac{1}{c_2}\left(\mathbb{E}\!\left[f(\boldsymbol{\phi}^{0},\boldsymbol{\gamma}^{0})\right] - f^*
+c_1\,G_{0}\right).
\end{align}
Plug in the value of $c_1$, $c_2$ in Lemma \ref{lemma:main_nonconvex}, and when $\eta \le \min\Biggl\{
\frac{1}{2\sqrt{2}\,L\sqrt{1+2\kappa}},
\;
\frac{\omega}{256\,L\,(1+\kappa)}
\Biggr\},$ we have $c_1 >0, \, c_2 >0$, so
\begin{align}
\mathbb{E}\left[\frac{1}{T}\sum_{t=1}^T
\Bigl\|
\nabla_{\boldsymbol{\lambda}} f(\boldsymbol{\phi}^{t-1},\boldsymbol{\gamma}^{t-1})
\Bigr\|^2\right]
\;=\;
O\!\left(
\frac{\mathbb{E}[f(\boldsymbol{\phi}^{0},\boldsymbol{\gamma}^{0})]-f^*}{\eta T}
\;+\;
\frac{nL\left(1+\frac{L^2}{\mu^2}\right)G_0}{\eta T}
\right).
\end{align}
We complete the proof. 
\end{proof}

\begin{lemma} \label{lemma:sensitivity}
Define
\[
\boldsymbol{\phi}_i^\star(\boldsymbol{\lambda})
:= \argmin_{\boldsymbol{\phi}_i} f_i(\boldsymbol{\phi}_i,\boldsymbol{\lambda}).
\]
With Assumptions \ref{ass:lsmoot_nonconvex} and \ref{ass:phi_strong_convex_nonconvex}, we have
\begin{align}
\bigl\|
\boldsymbol{\phi}_i^\star(\boldsymbol{\lambda}_1)
-
\boldsymbol{\phi}_i^\star(\boldsymbol{\lambda}_2)
\bigr\|
\;\le\;
\frac{L}{\mu}
\bigl\|
\boldsymbol{\lambda}_1-\boldsymbol{\lambda}_2
\bigr\|.
\end{align}
\end{lemma}

\begin{proof}
By first-order optimality,
\begin{align*}
\nabla_{\boldsymbol{\phi}_i} f_i(\boldsymbol{\phi}_i^\star(\boldsymbol{\lambda}_1),\boldsymbol{\lambda}_1)
=
\nabla_{\boldsymbol{\phi}_i} f_i(\boldsymbol{\phi}_i^\star(\boldsymbol{\lambda}_2),\boldsymbol{\lambda}_2)
= \boldsymbol{0}.
\end{align*}
Thus,
\begin{align*}
\boldsymbol{0}
&=
\nabla_{\boldsymbol{\phi}_i} f_i(\boldsymbol{\phi}_i^\star(\boldsymbol{\lambda}_1),\boldsymbol{\lambda}_1)
-
\nabla_{\boldsymbol{\phi}_i} f_i(\boldsymbol{\phi}_i^\star(\boldsymbol{\lambda}_2),\boldsymbol{\lambda}_2) \\
&=
\Bigl[
\nabla_{\boldsymbol{\phi}_i} f_i(\boldsymbol{\phi}_i^\star(\boldsymbol{\lambda}_1),\boldsymbol{\lambda}_1)
-
\nabla_{\boldsymbol{\phi}_i} f_i(\boldsymbol{\phi}_i^\star(\boldsymbol{\lambda}_2),\boldsymbol{\lambda}_1)
\Bigr]
+
\Bigl[
\nabla_{\boldsymbol{\phi}_i} f_i(\boldsymbol{\phi}_i^\star(\boldsymbol{\lambda}_2),\boldsymbol{\lambda}_1)
-
\nabla_{\boldsymbol{\phi}_i} f_i(\boldsymbol{\phi}_i^\star(\boldsymbol{\lambda}_2),\boldsymbol{\lambda}_2)
\Bigr].
\end{align*}

By Asssumption \ref{ass:phi_strong_convex_nonconvex}, we can derivate
\[
\bigl\langle
\nabla_{\boldsymbol{\phi}_i} f_i(\boldsymbol{\phi}_i^\star(\boldsymbol{\lambda}_1),\boldsymbol{\lambda}_1)
-
\nabla_{\boldsymbol{\phi}_i} f_i(\boldsymbol{\phi}_i^\star(\boldsymbol{\lambda}_2),\boldsymbol{\lambda}_1),
\;
\boldsymbol{\phi}_i^\star(\boldsymbol{\lambda}_1)
-
\boldsymbol{\phi}_i^\star(\boldsymbol{\lambda}_2)
\bigr\rangle
\ge
\mu
\bigl\|
\boldsymbol{\phi}_i^\star(\boldsymbol{\lambda}_1)
-
\boldsymbol{\phi}_i^\star(\boldsymbol{\lambda}_2)
\bigr\|^2.
\]

By Cauchy--Schwarz,
\[
\bigl\|
\nabla_{\boldsymbol{\phi}_i} f_i(\boldsymbol{\phi}_i^\star(\boldsymbol{\lambda}_1),\boldsymbol{\lambda}_1)
-
\nabla_{\boldsymbol{\phi}_i} f_i(\boldsymbol{\phi}_i^\star(\boldsymbol{\lambda}_2),\boldsymbol{\lambda}_1)
\bigr\|
\ge
\mu
\bigl\|
\boldsymbol{\phi}_i^\star(\boldsymbol{\lambda}_1)
-
\boldsymbol{\phi}_i^\star(\boldsymbol{\lambda}_2)
\bigr\|.
\]

By \(L\)-smoothness in \(\boldsymbol{\lambda}\),
\[
\bigl\|
\nabla_{\boldsymbol{\phi}_i} f_i(\boldsymbol{\phi}_i^\star(\boldsymbol{\lambda}_2),\boldsymbol{\lambda}_1)
-
\nabla_{\boldsymbol{\phi}_i} f_i(\boldsymbol{\phi}_i^\star(\boldsymbol{\lambda}_2),\boldsymbol{\lambda}_2)
\bigr\|
\le
L
\bigl\|
\boldsymbol{\lambda}_1-\boldsymbol{\lambda}_2
\bigr\|.
\]

Combining the above inequalities yields
\[
\bigl\|
\boldsymbol{\phi}_i^\star(\boldsymbol{\lambda}_1)
-
\boldsymbol{\phi}_i^\star(\boldsymbol{\lambda}_2)
\bigr\|
\le
\frac{L}{\mu}
\bigl\|
\boldsymbol{\lambda}_1-\boldsymbol{\lambda}_2
\bigr\|.
\]
\end{proof}

\begin{lemma} \label{lemma:nonconvex_delta_bound}
We have
\begin{align}
\delta_t
\;\le\;
\frac{2L^2}{\mu^2}\bigl(G_{t-1}+\epsilon_t\bigr).
\end{align}
\end{lemma}

\begin{proof}
By definition,
\[
\delta_t
=
\frac{1}{n}\sum_{i=1}^n
\mathbb{E}\bigl\|
\boldsymbol{\phi}_i^{t-1}-\tilde{\boldsymbol{\phi}}_i^{\,t}
\bigr\|^2.
\]
Using Lemma \ref{lemma:sensitivity},
\[
\delta_t
\le
\frac{1}{n}\sum_{i=1}^n
\frac{L^2}{\mu^2}\,
\mathbb{E}\bigl\|
\boldsymbol{\lambda}_i^{t-1}-\tilde{\boldsymbol{\lambda}}_i^{\,t}
\bigr\|^2.
\]
Insert and subtract \(\boldsymbol{\gamma}^{t-1}\):
\begin{align*}
\delta_t
&\le
\frac{1}{n}\sum_{i=1}^n
\frac{L^2}{\mu^2}\,
\mathbb{E}\bigl\|
\boldsymbol{\lambda}_i^{t-1}-\boldsymbol{\gamma}^{t-1}
+
\boldsymbol{\gamma}^{t-1}-\tilde{\boldsymbol{\lambda}}_i^{\,t}
\bigr\|^2 \\
&\le
\frac{2L^2}{n\mu^2}
\sum_{i=1}^n
\Bigl(
\mathbb{E}\bigl\|
\boldsymbol{\lambda}_i^{t-1}-\boldsymbol{\gamma}^{t-1}
\bigr\|^2
+
\mathbb{E}\bigl\|
\boldsymbol{\gamma}^{t-1}-\tilde{\boldsymbol{\lambda}}_i^{\,t}
\bigr\|^2
\Bigr).
\end{align*}
By the definitions of \(G_{t-1}\) and \(\epsilon_t\), this yields
\[
\delta_t
\le
\frac{2L^2}{\mu^2}\bigl(G_{t-1}+\epsilon_t\bigr).
\]
\end{proof}


\begin{lemma}\label{lemma:epsilon_nonconvex}
\begin{align}
(1-4\eta^2 L^2(1+\frac{2L^2}{\mu^2}))\,\epsilon_t
&\le
8\eta^2 L^2(1+\frac{L^2}{\mu^2})\,G_{t-1}
+
4\eta^2\,
\mathbb{E}\!\left\|
\nabla_{\boldsymbol{\lambda}_i} f(\boldsymbol{\phi}^{t-1},\boldsymbol{\gamma}^{t-1})
\right\|^2 .
\end{align}
\begin{proof}
\begin{align*}
\epsilon_t
&=
\frac{1}{n}\sum_{i=1}^n
\mathbb{E}\!\left\|
\tilde{\boldsymbol{\lambda}}_i^{t}
-
\boldsymbol{\gamma}^{t-1}
\right\|^2,
\\
&=
\frac{1}{n}\sum_{i=1}^n
\mathbb{E}\!\left\|
\tilde{\boldsymbol{\lambda}}_i^{t}
-
\boldsymbol{\lambda}_0^{t-1}
+
\boldsymbol{h}^{t-1}
\right\|^2,
\\
&=
\frac{1}{n}\sum_{i=1}^n
\mathbb{E}\!\left\|
\eta\!\left(
\nabla_{\boldsymbol{\lambda}_i} f_i(\boldsymbol{\phi}_i^{t-1},\boldsymbol{\lambda}_i^{t-1})
-
\nabla_{\boldsymbol{\lambda}_i} f_i(\tilde{\boldsymbol{\phi}}_i^{t},\tilde{\boldsymbol{\lambda}}_i^{t})
\right)
+
\boldsymbol{h}^{t-1}
\right\|^2, \quad \left(\text{by \eqref{eq:true_equal_virtual}} \right)
\\
&=
\frac{\eta^2}{n}\sum_{i=1}^n
\mathbb{E}\!\Big\|
\nabla_{\boldsymbol{\lambda}_i} f_i(\boldsymbol{\phi}_i^{t-1},\boldsymbol{\lambda}_i^{t-1})
-
\nabla_{\boldsymbol{\lambda}_i} f_i(\boldsymbol{\phi}_i^{t-1},\boldsymbol{\gamma}^{\,t-1})
+
\nabla_{\boldsymbol{\lambda}_i} f_i(\boldsymbol{\phi}_i^{t-1},\boldsymbol{\gamma}^{\,t-1})
\\
&\qquad\qquad
-
\nabla_{\boldsymbol{\lambda}_i} f(\tilde{\boldsymbol{\phi}}_i^{t},\tilde{\boldsymbol{\lambda}}_i^{t})
-
\nabla_{\boldsymbol{\lambda}_i} f(\boldsymbol{\phi}^{t-1},\boldsymbol{\gamma}^{t-1})
+
\nabla_{\boldsymbol{\lambda}_i} f_i(\boldsymbol{\phi}^{t-1},\boldsymbol{\gamma}^{t-1})
+ \frac{1}{\eta} \boldsymbol{h}^{t-1}
\Big\|^2, \\
&\le
\frac{4\eta^2}{n}\sum_{i=1}^n
\mathbb{E}\!\left\|
\nabla_{\boldsymbol{\lambda}_i} f_i(\boldsymbol{\phi}_i^{t-1},\boldsymbol{\lambda}_i^{t-1})
-
\nabla_{\boldsymbol{\lambda}_i} f_i(\boldsymbol{\phi}_i^{t-1},\boldsymbol{\gamma}^{\,t-1})
\right\|^2
+
\frac{4\eta^2}{n}\sum_{i=1}^n
\mathbb{E}\!\left\|
\nabla_{\boldsymbol{\lambda}_i} f_i(\boldsymbol{\phi}_i^{t-1},\boldsymbol{\gamma}^{\,t-1})
-
\nabla_{\boldsymbol{\lambda}_i} f_i(\tilde{\boldsymbol{\phi}}_i^{t},\tilde{\boldsymbol{\lambda}}_i^{t})
\right\|^2
\\
&\quad+
4\eta^2\,
\mathbb{E}\!\left\|
\nabla_{\boldsymbol{\lambda}_i} f(\boldsymbol{\phi}^{t-1},\boldsymbol{\gamma}^{t-1})
\right\|^2
+
4\eta^2\,
\mathbb{E}\!\left\|
\nabla_{\boldsymbol{\lambda}_i} f(\boldsymbol{\phi}^{t-1},\boldsymbol{\gamma}^{t-1})
- \frac{1}{n}\sum_{i=1}^n \nabla_{\boldsymbol{\lambda}_i} f_i(\boldsymbol{\phi}_i^{t-1},\boldsymbol{\lambda}_i^{t-1})
\right\|^2,
\\
&\le
8\eta^2 L^2\,G_{t-1}
+
4\eta^2 L^2\,(\epsilon_t+\delta_t)
+
4\eta^2\,
\mathbb{E}\!\left\|
\nabla_{\boldsymbol{\lambda}_i} f(\boldsymbol{\phi}^{t-1},\boldsymbol{\gamma}^{\,t-1})
\right\|^2.
\end{align*}
So we have
\begin{align*}
(1-4\eta^2 L^2)\,\epsilon_t
&\le
8\eta^2 L^2\,G_{t-1}
+
4\eta^2 L^2\,\delta_t
+
4\eta^2\,
\mathbb{E}\!\left\|
\nabla_{\boldsymbol{\lambda}_i} f(\boldsymbol{\phi}^{t-1},\boldsymbol{\gamma}^{t-1})
\right\|^2 .
\end{align*}
According to Lemma \ref{lemma:nonconvex_delta_bound}, we have
\begin{align*}
(1-4\eta^2 L^2(1+\frac{2L^2}{\mu^2}))\,\epsilon_t
&\le
8\eta^2 L^2(1+\frac{L^2}{\mu^2})\,G_{t-1}
+
4\eta^2\,
\mathbb{E}\!\left\|
\nabla_{\boldsymbol{\lambda}_i} f(\boldsymbol{\phi}^{t-1},\boldsymbol{\gamma}^{t-1})
\right\|^2 .
\end{align*}
\end{proof}
\end{lemma}

\begin{lemma}
\begin{align}
G_t
&\le
\frac{2(n-m)}{2n-m}\,G_{t-1}
+
\left(
\frac{2m}{2n-m}
+
\frac{2n}{m}
\right)\epsilon_t .
\end{align}
\end{lemma}

\begin{proof}
Write
\[
\boldsymbol{\lambda}_i^{t}-\boldsymbol{\gamma}^{t}
=
\bigl(\boldsymbol{\lambda}_i^{t}-\boldsymbol{\gamma}^{t-1}\bigr)
+
\bigl(\boldsymbol{\gamma}^{t-1}-\boldsymbol{\gamma}^{t}\bigr).
\]
Using the inequality
\[
\|a+b\|_2^2
\le
(1+z)\|a\|_2^2
+
\Bigl(1+\frac{1}{z}\Bigr)\|b\|_2^2,
\qquad \forall z>0,
\]
with \(z=\frac{m}{2n-m}\), we obtain
\begin{align*}
G_t
&=
\frac{1}{n}\sum_{i=1}^n
\mathbb{E}\bigl\|\boldsymbol{\lambda}_i^{t}-\boldsymbol{\gamma}^{t}\bigr\|_2^2, \\
&\le
\Bigl(1+\frac{m}{2n-m}\Bigr)
\frac{1}{n}\sum_{i=1}^n
\mathbb{E}\bigl\|
\boldsymbol{\lambda}_i^{t}-\boldsymbol{\gamma}^{t-1}
\bigr\|_2^2
+
\Bigl(1+\frac{2n-m}{m}\Bigr)
\mathbb{E}\bigl\|
\boldsymbol{\gamma}^{t-1}-\boldsymbol{\gamma}^{t}
\bigr\|_2^2, \\
&\le
\frac{m}{n}\Bigl(1+\frac{m}{2n-m}\Bigr)
\frac{1}{n}\sum_{i=1}^n
\mathbb{E}\bigl\|
\boldsymbol{\tilde{\lambda}}_i^{t}-\boldsymbol{\gamma}^{t-1}
\bigr\|_2^2
+ \left(1-\frac{m}{n}\right)\Bigl(1+\frac{m}{2n-m}\Bigr)
\frac{1}{n}\sum_{i=1}^n
\mathbb{E}\bigl\|
\boldsymbol{\lambda}_i^{t-1}-\boldsymbol{\gamma}^{t-1}
\bigr\|_2^2\\
& \quad +
\Bigl(1+\frac{2n-m}{m}\Bigr)
\mathbb{E}\bigl\|
\boldsymbol{\gamma}^{t-1}-\boldsymbol{\gamma}^{t}
\bigr\|_2^2, \\
&=
\frac{2m}{2n-m}\,\epsilon_t
+
\frac{2(n-m)}{2n-m}\,G_{t-1}
+
\frac{2n}{m}\,
\mathbb{E}\bigl\|
\boldsymbol{\gamma}^{t-1}-\boldsymbol{\gamma}^{t}
\bigr\|_2^2.
\end{align*}
By Lemma \ref{lemma:gamma_bound}, we have
\begin{align*}
G_t
&\le
\frac{2(n-m)}{2n-m}\,G_{t-1}
+
\left(
\frac{2m}{2n-m}
+
\frac{2n}{m}
\right)\epsilon_t .
\end{align*}
\end{proof}

\begin{lemma} \label{lemma:combine_nonconvex}
\begin{align} \nonumber
\mathbb{E}\!\left[f(\boldsymbol{\phi}^t,\boldsymbol{\gamma}^t)\right]
-
\mathbb{E}\!\left[f(\boldsymbol{\phi}^{t-1},\boldsymbol{\gamma}^{t-1})\right]
\le\;
&\Biggl(\frac{L}{2} +
\frac{\eta L^2}{2}
+
\frac{L^2}{\mu^2}\bigl(mL+\eta L^2\bigr)
\Biggr)\epsilon_t  +
\Biggl(
\frac{Lm}{2}
+
\frac{L^2}{\mu^2}\bigl(mL+\eta L^2\bigr)
\Biggr) G_{t-1}
\\ 
&
-
\frac{\eta}{2}\,
\mathbb{E}\Bigl\|
\nabla_{\boldsymbol{\lambda}} f(\boldsymbol{\phi}^{t-1},\boldsymbol{\gamma}^{t-1})
\Bigr\|^2.
\end{align}
\begin{proof}
By \(L\)-smoothness of \(f\),
\begin{align*}
\mathbb{E}\!\left[f(\boldsymbol{\phi}^t,\boldsymbol{\gamma}^t)\right]
-
\mathbb{E}\!\left[f(\boldsymbol{\phi}^{t-1},\boldsymbol{\gamma}^{t-1})\right]
&\le
\mathbb{E}\!\left[
\left\langle
\nabla_\lambda f(\boldsymbol{\phi}^{t-1},\boldsymbol{\gamma}^{t-1}),
\boldsymbol{\gamma}^t-\boldsymbol{\gamma}^{t-1}
\right\rangle
\right] \\
&\quad+
\mathbb{E}\!\left[
\left\langle
\nabla_\phi f(\boldsymbol{\phi}^{t-1},\boldsymbol{\gamma}^{t-1}),
\boldsymbol{\phi}^t-\boldsymbol{\phi}^{t-1}
\right\rangle
\right] \\
&\quad+
\frac{L}{2}\mathbb{E}\|\boldsymbol{\gamma}^t-\boldsymbol{\gamma}^{t-1}\|^2
+
\frac{L}{2}\mathbb{E}\|\boldsymbol{\phi}^t-\boldsymbol{\phi}^{t-1}\|^2 .
\end{align*}

Since
\[
\mathbb{E}\|\boldsymbol{\phi}^t-\boldsymbol{\phi}^{t-1}\|^2
=
\sum_{i=1}^n \frac{m}{n}\,
\mathbb{E}\|\tilde{\boldsymbol{\phi}}_i^t-\boldsymbol{\phi}_i^{t-1}\|^2
=
m\,\delta_t,
\]
we obtain
\begin{align*}
& \quad \;\mathbb{E}\!\left[f(\boldsymbol{\phi}^t,\boldsymbol{\gamma}^t)\right]
-
\mathbb{E}\!\left[f(\boldsymbol{\phi}^{t-1},\boldsymbol{\gamma}^{t-1})\right]
-
\frac{L}{2}\mathbb{E}\|\boldsymbol{\gamma}^t-\boldsymbol{\gamma}^{t-1}\|^2
-
\frac{mL}{2}\delta_t\\
&\le
\mathbb{E}\!\left[
\left\langle
\nabla_\lambda f(\boldsymbol{\phi}^{t-1},\boldsymbol{\gamma}^{t-1}),
\boldsymbol{\gamma}^t-\boldsymbol{\gamma}^{t-1}
\right\rangle
\right] +
\mathbb{E}\!\left[
\left\langle
\nabla_\phi f(\boldsymbol{\phi}^{t-1},\boldsymbol{\gamma}^{t-1}),
\boldsymbol{\phi}^t-\boldsymbol{\phi}^{t-1}
\right\rangle
\right],\\
&\le
\eta\,\mathbb{E}\!\left[
\left\langle
\nabla_\lambda f(\boldsymbol{\phi}^{t-1},\boldsymbol{\gamma}^{t-1}),
\frac{1}{n}\sum_{i=1}^n
\bigl(-\nabla_{\boldsymbol{\lambda}_i} f_i(\tilde{\boldsymbol{\phi}}_i^t,\tilde{\boldsymbol{\lambda}}_i^t)\bigr)
\right\rangle
\right]
+
\frac{m}{n}\sum_{i=1}^n
\mathbb{E}\!\left[
\left\langle
\nabla_{\boldsymbol{\phi}_i} f_i(\boldsymbol{\phi}_i^{t-1},\boldsymbol{\gamma}^{t-1}),
\tilde{\boldsymbol{\phi}}_i^t-\boldsymbol{\phi}_i^{t-1}
\right\rangle
\right]. \quad \left(\text{by Lemma \ref{lemma:gamma_expectation}} \right)
\end{align*}

Since 
\begin{align*}
& \quad \eta\,\mathbb{E}\!\left[
\left\langle
\nabla_{\boldsymbol{\lambda}} f(\boldsymbol{\phi}^{t-1},\boldsymbol{\gamma}^{t-1}),
\frac{1}{n}\sum_{i=1}^n
\bigl(-\nabla_{\boldsymbol{\lambda}_i} f_i(\tilde{\boldsymbol{\phi}}_i^t,\tilde{\boldsymbol{\lambda}}_i^t)\bigr)
\right\rangle
\right] \\
& \leq 
\eta\,\mathbb{E}\!\left[
\left\langle
\nabla_{\boldsymbol{\lambda}} f(\boldsymbol{\phi}^{t-1},\boldsymbol{\gamma}^{t-1}),
\nabla_{\boldsymbol{\lambda}} f(\boldsymbol{\phi}^{t-1},\boldsymbol{\gamma}^{t-1})
-
\frac{1}{n}\sum_{i=1}^n
\nabla_{\boldsymbol{\lambda}_i} f_i(\tilde{\boldsymbol{\phi}}_i^t,\tilde{\boldsymbol{\lambda}}_i^t)
\right\rangle
\right]
-
\eta\,\mathbb{E}\!\left\|
\nabla_{\boldsymbol{\lambda}} f(\boldsymbol{\phi}^{t-1},\boldsymbol{\gamma}^{t-1})
\right\|^2 \\
& \leq 
\frac{\eta}{2}\,
\mathbb{E}\!\left[
\left\|
\nabla_{\boldsymbol{\lambda}} f(\boldsymbol{\phi}^{t-1},\boldsymbol{\gamma}^{t-1})
\right\|^2
+
\left\|
\frac{1}{n}\sum_{i=1}^n
\nabla_{\boldsymbol{\lambda}_i} f_i(\tilde{\boldsymbol{\phi}}_i^t,\tilde{\boldsymbol{\lambda}}_i^t)
\right\|^2
\right]
-
\eta\,\mathbb{E}\!\left\|
\nabla_{\boldsymbol{\lambda}} f(\boldsymbol{\phi}^{t-1},\boldsymbol{\gamma}^{t-1})
\right\|^2 \\
& \leq 
\frac{\eta}{2n}\sum_{i=1}^n
\mathbb{E}\Bigl\|
\nabla_{\boldsymbol{\lambda}_i} f_i(\boldsymbol{\phi}_i^{t-1},\boldsymbol{\gamma}^{t-1})
-
\nabla_{\boldsymbol{\lambda}_i} f_i(\tilde{\boldsymbol{\phi}}_i^{t},\tilde{\boldsymbol{\lambda}}_i^{t})
\Bigr\|^2
-
\frac{\eta}{2}\,
\mathbb{E}\Bigl\|
\nabla_{\boldsymbol{\lambda}} f(\boldsymbol{\phi}^{t-1},\boldsymbol{\gamma}^{t-1})
\Bigr\|^2,\\
& \leq 
\frac{\eta L^2}{2n}\sum_{i=1}^n
\mathbb{E}\left[\|\boldsymbol{\phi}_i^{t-1} - \tilde{\boldsymbol{\phi}}_i^{t} \|^2 +  \|\boldsymbol{\gamma}^{t-1} - \tilde{\boldsymbol{\lambda}}_i^{t} \|^2\right]
-
\frac{\eta}{2}\,
\mathbb{E}\Bigl\|
\nabla_{\boldsymbol{\lambda}} f(\boldsymbol{\phi}^{t-1},\boldsymbol{\gamma}^{t-1})
\Bigr\|^2, \quad \left(\text{by $L$-smoothness} \right)\\
& \leq \frac{\eta L^2}{2} \delta_t + \frac{\eta L^2}{2} \epsilon_t - \frac{\eta}{2}\,
\mathbb{E}\Bigl\|
\nabla_{\boldsymbol{\lambda}} f(\boldsymbol{\phi}^{t-1},\boldsymbol{\gamma}^{t-1})
\Bigr\|^2,
\end{align*}
And 
\begin{align*}
\frac{m}{n}\sum_{i=1}^n
\mathbb{E}\!\left[
\left\langle
\nabla_{\boldsymbol{\phi}_i} f_i(\boldsymbol{\phi}_i^{t-1},\boldsymbol{\gamma}^{t-1}),
\tilde{\boldsymbol{\phi}}_i^t-\boldsymbol{\phi}_i^{t-1}
\right\rangle
\right] & \leq \frac{m}{n}\sum_{i=1}^n
\mathbb{E}\!\left[
\left\|
\nabla_{\boldsymbol{\phi}_i} f_i(\boldsymbol{\phi}_i^{t-1},\boldsymbol{\gamma}^{t-1})\right\| \;
\left\| \tilde{\boldsymbol{\phi}}_i^t-\boldsymbol{\phi}_i^{t-1}
\right\|
\right],\\
& \leq \frac{m}{n}\sum_{i=1}^n
\mathbb{E}\!\left[
\left\|
\nabla_{\boldsymbol{\phi}_i} f_i(\boldsymbol{\phi}_i^{t-1},\boldsymbol{\gamma}^{t-1}) - \nabla_{\boldsymbol{\phi}_i} f_i(\boldsymbol{\phi}_i^{t-1},\boldsymbol{\lambda}_i^{t-1})\right\| \;
\left\| \tilde{\boldsymbol{\phi}}_i^t-\boldsymbol{\phi}_i^{t-1}
\right\|
\right],\\
& \leq \frac{Lm}{n}\sum_{i=1}^n
\mathbb{E}\!\left[
\left\|
\boldsymbol{\gamma}^{t-1} - \boldsymbol{\lambda}_i^{t-1}\right\| \;
\left\| \tilde{\boldsymbol{\phi}}_i^t-\boldsymbol{\phi}_i^{t-1}
\right\|
\right], \quad \left(\text{by $L$-smoothness} \right)\\
& \leq \frac{Lm}{2n}\sum_{i=1}^n
\mathbb{E}\!\left[
\left\|
\boldsymbol{\gamma}^{t-1} - \boldsymbol{\lambda}_i^{t-1}\right\|^2 + 
\left\| \tilde{\boldsymbol{\phi}}_i^t-\boldsymbol{\phi}_i^{t-1}
\right\|^2
\right],\\
& = \frac{Lm}{2}(G_{t-1} + \delta_t).
\end{align*}
Combining all of these, we have
\begin{align*}
    \mathbb{E}\!\left[f(\boldsymbol{\phi}^t,\boldsymbol{\gamma}^t)\right]
-
\mathbb{E}\!\left[f(\boldsymbol{\phi}^{t-1},\boldsymbol{\gamma}^{t-1})\right]
-
\frac{L}{2}\mathbb{E}\|\boldsymbol{\gamma}^t-\boldsymbol{\gamma}^{t-1}\|^2
-
\frac{mL}{2}\delta_t \leq \frac{\eta L^2}{2} \delta_t + \frac{\eta L^2}{2} \epsilon_t - \frac{\eta}{2}\,
\mathbb{E}\Bigl\|
\nabla_{\boldsymbol{\lambda}} f(\boldsymbol{\phi}^{t-1},\boldsymbol{\gamma}^{t-1})
\Bigr\|^2 + \frac{Lm}{2}(G_{t-1} + \delta_t).
\end{align*}
Reorganize it, we have
\begin{align*}
\mathbb{E}\!\left[f(\boldsymbol{\phi}^t,\boldsymbol{\gamma}^t)\right]
-
\mathbb{E}\!\left[f(\boldsymbol{\phi}^{t-1},\boldsymbol{\gamma}^{t-1})\right]
&\le
\frac{L}{2}\mathbb{E}\|\boldsymbol{\gamma}^t-\boldsymbol{\gamma}^{t-1}\|^2
+
\Bigl(\frac{mL}{2}+\frac{\eta L^2}{2}\Bigr)\delta_t
\\
& \; +
\frac{\eta L^2}{2}\epsilon_t
+
\frac{Lm}{2}G_{t-1}
-
\frac{\eta}{2}\,
\mathbb{E}\Bigl\|
\nabla_{\boldsymbol{\lambda}} f(\boldsymbol{\phi}^{t-1},\boldsymbol{\gamma}^{t-1})
\Bigr\|^2,\\
& \leq \frac{L}{2}\mathbb{E}\|\boldsymbol{\gamma}^t-\boldsymbol{\gamma}^{t-1}\|^2 +
\Biggl(
\frac{Lm}{2}
+
\frac{L^2}{\mu^2}\bigl(mL+\eta L^2\bigr)
\Biggr) G_{t-1}
+
\Biggl(
\frac{\eta L^2}{2}
+
\frac{L^2}{\mu^2}\bigl(mL+\eta L^2\bigr)
\Biggr)\epsilon_t
\\
&-
\frac{\eta}{2}\,
\mathbb{E}\Bigl\|
\nabla_{\boldsymbol{\lambda}} f(\boldsymbol{\phi}^{t-1},\boldsymbol{\gamma}^{t-1})
\Bigr\|^2,  \quad \left(\text{by Lemma \ref{lemma:nonconvex_delta_bound}} \right)
\end{align*}
By Lemma \ref{lemma:gamma_bound}, we have
\begin{align*}
\mathbb{E}\!\left[f(\boldsymbol{\phi}^t,\boldsymbol{\gamma}^t)\right]
-
\mathbb{E}\!\left[f(\boldsymbol{\phi}^{t-1},\boldsymbol{\gamma}^{t-1})\right]
\le\;
&\Biggl(\frac{L}{2} +
\frac{\eta L^2}{2}
+
\frac{L^2}{\mu^2}\bigl(mL+\eta L^2\bigr)
\Biggr)\epsilon_t  +
\Biggl(
\frac{Lm}{2}
+
\frac{L^2}{\mu^2}\bigl(mL+\eta L^2\bigr)
\Biggr) G_{t-1}
\\
&
-
\frac{\eta}{2}\,
\mathbb{E}\Bigl\|
\nabla_{\boldsymbol{\lambda}} f(\boldsymbol{\phi}^{t-1},\boldsymbol{\gamma}^{t-1})
\Bigr\|^2.
\end{align*}

\end{proof}
\end{lemma}


\begin{lemma}\label{lemma:main_nonconvex}
Assume
\[
1-4\eta^2 L^2\Bigl(1+\frac{2L^2}{\mu^2}\Bigr)>0, \quad c_1 >0, \quad c_2 >0
\]
where $c_1$ and $c_2$ are defined below. Moreover, a sufficient condition for these assumptions is
\(
\eta \le \min\Biggl\{
\frac{1}{2\sqrt{2}\,L\sqrt{1+2\kappa}},
\;
\frac{\omega}{256\,L\,(1+\kappa)}
\Biggr\},
\qquad
\omega:=\frac{m}{n}\in(0,1),\ \kappa:=\frac{L^2}{\mu^2}.
\)
Define
\[
\rho:=\frac{2(n-m)}{2n-m},
\qquad
\xi:=\frac{2m}{2n-m}+\frac{2n}{m},
\]
and
\[
A:=
\frac{L}{2}+\frac{\eta L^2}{2}+\frac{L^2}{\mu^2}(mL+\eta L^2),
\qquad
B:=
\frac{Lm}{2}+\frac{L^2}{\mu^2}(mL+\eta L^2).
\]
Let
\[
p:=
\frac{8\eta^2 L^2\bigl(1+\frac{L^2}{\mu^2}\bigr)}
{1-4\eta^2 L^2\bigl(1+\frac{2L^2}{\mu^2}\bigr)},
\qquad
q:=
\frac{4\eta^2}
{1-4\eta^2 L^2\bigl(1+\frac{2L^2}{\mu^2}\bigr)}.
\]
Choose
\[
c_1
:=
-\frac{Ap+B}{\xi p+(\rho-1)}
=
-\frac{Ap+B}{\xi p-\frac{m}{2n-m}},
\]
and define
\[
c_2
:=
\frac{\eta}{2}-(A+c_1 \xi)q.
\]
Then the iterates satisfy
\begin{align}
&\mathbb{E}\!\left[f(\boldsymbol{\phi}^t,\boldsymbol{\gamma}^t)\right]
+c_1\,G_t \le
\mathbb{E}\!\left[f(\boldsymbol{\phi}^{t-1},\boldsymbol{\gamma}^{t-1})\right]
+c_1\,G_{t-1}
-
c_2\,
\mathbb{E}\Bigl\|
\nabla_{\boldsymbol{\lambda}} f(\boldsymbol{\phi}^{t-1},\boldsymbol{\gamma}^{t-1})
\Bigr\|^2 .
\end{align}
\end{lemma}

\begin{proof}
Start from the three inequalities:
\begin{align}
(1-\alpha_\eta)\epsilon_t
&\le
\beta_\eta G_{t-1}
+
4\eta^2\,
\mathbb{E}\Bigl\|
\nabla_{\boldsymbol{\lambda}} f(\boldsymbol{\phi}^{t-1},\boldsymbol{\gamma}^{t-1})
\Bigr\|^2,
\label{eq:eps-bound}
\\
G_t
&\le
\rho\,G_{t-1}+\xi\,\epsilon_t,
\label{eq:G-rec}
\\
\mathbb{E}\!\left[f(\boldsymbol{\phi}^t,\boldsymbol{\gamma}^t)\right]
-
\mathbb{E}\!\left[f(\boldsymbol{\phi}^{t-1},\boldsymbol{\gamma}^{t-1})\right]
&\le
A\,\epsilon_t+B\,G_{t-1}
-\frac{\eta}{2}\,
\mathbb{E}\Bigl\|
\nabla_{\boldsymbol{\lambda}} f(\boldsymbol{\phi}^{t-1},\boldsymbol{\gamma}^{t-1})
\Bigr\|^2,
\label{eq:f-descent}
\end{align}
where
\[
\alpha_\eta:=4\eta^2 L^2\Bigl(1+\frac{2L^2}{\mu^2}\Bigr),
\qquad
\beta_\eta:=8\eta^2 L^2\Bigl(1+\frac{L^2}{\mu^2}\Bigr).
\]
Since \(1-\alpha_\eta>0\), \eqref{eq:eps-bound} implies
\[
\epsilon_t \le p\,G_{t-1}+q\,
\mathbb{E}\Bigl\|
\nabla_{\boldsymbol{\lambda}} f(\boldsymbol{\phi}^{t-1},\boldsymbol{\gamma}^{t-1})
\Bigr\|^2.
\]

Multiply \eqref{eq:G-rec} by \(c_1\) and add it to \eqref{eq:f-descent} to obtain
\begin{align*}
&\mathbb{E}\!\left[f(\boldsymbol{\phi}^t,\boldsymbol{\gamma}^t)\right]+c_1 G_t
-
\Bigl(
\mathbb{E}\!\left[f(\boldsymbol{\phi}^{t-1},\boldsymbol{\gamma}^{t-1})\right]+c_1 G_{t-1}
\Bigr)
\\
&\qquad\le
(A+c_1 \xi)\epsilon_t+\bigl(B+c_1(\rho-1)\bigr)G_{t-1}
-\frac{\eta}{2}\,
\mathbb{E}\Bigl\|
\nabla_{\boldsymbol{\lambda}} f(\boldsymbol{\phi}^{t-1},\boldsymbol{\gamma}^{t-1})
\Bigr\|^2.
\end{align*}
Substituting the bound on \(\epsilon_t\) yields
\begin{align*}
&\mathbb{E}\!\left[f(\boldsymbol{\phi}^t,\boldsymbol{\gamma}^t)\right]+c_1 G_t
-
\Bigl(
\mathbb{E}\!\left[f(\boldsymbol{\phi}^{t-1},\boldsymbol{\gamma}^{t-1})\right]+c_1 G_{t-1}
\Bigr)
\\
&\qquad\le
\Bigl((A+c_1 \xi)p+B+c_1(\rho-1)\Bigr)G_{t-1}
+
\Bigl((A+c_1 \xi)q-\frac{\eta}{2}\Bigr)
\mathbb{E}\Bigl\|
\nabla_{\boldsymbol{\lambda}} f(\boldsymbol{\phi}^{t-1},\boldsymbol{\gamma}^{t-1})
\Bigr\|^2.
\end{align*}
With the choice
\(
c_1=-(Ap+B)/(\xi p+(\rho-1))
\),
the coefficient of \(G_{t-1}\) vanishes, leaving
\[
\mathbb{E}\!\left[f(\boldsymbol{\phi}^t,\boldsymbol{\gamma}^t)\right]+c_1 G_t
\le
\mathbb{E}\!\left[f(\boldsymbol{\phi}^{t-1},\boldsymbol{\gamma}^{t-1})\right]+c_1 G_{t-1}
-
c_2
\mathbb{E}\Bigl\|
\nabla_{\boldsymbol{\lambda}} f(\boldsymbol{\phi}^{t-1},\boldsymbol{\gamma}^{t-1})
\Bigr\|^2,
\]
where \(c_2:=\frac{\eta}{2}-(A+c_1 \xi)q\). This proves the claim.

Next, we will calculate the range for $\eta$.
Let $\omega:=m/n\in(0,1)$ and $\kappa:=L^2/\mu^2$.
First, $1-\alpha_\eta>0$ holds if $\alpha_\eta \le \tfrac12$, i.e.,
\[
4\eta^2L^2(1+2\kappa)\le \frac12
\quad\Longleftrightarrow\quad
\eta \le \frac{1}{2\sqrt{2}\,L\sqrt{1+2\kappa}}.
\]
Under $\alpha_\eta\le \tfrac12$, we have $1-\alpha_\eta\ge \tfrac12$ and hence
\[
p \le 16\eta^2L^2(1+\kappa),
\qquad
q \le 8\eta^2.
\]
Moreover, since $\rho-1=-\frac{m}{2n-m}$, we have
\[
c_1>0
\ \Longleftrightarrow\
\xi p+(\rho-1)<0
\ \Longleftrightarrow\
\xi p < \frac{m}{2n-m}.
\]
Using $\frac{m}{2n-m}=\frac{\omega}{2-\omega}\ge \frac{\omega}{2}$ and
$\xi=\frac{2m}{2n-m}+\frac{2n}{m}=\frac{2\omega}{2-\omega}+\frac{2}{\omega}\le \frac{4}{\omega}$,
a sufficient condition is
\[
\frac{4}{\omega}\cdot 16\eta^2L^2(1+\kappa)\le \frac{\omega}{2},
\]
which is implied by $\eta \le \frac{\omega}{256\,L(1+\kappa)}$.
Under the same step size bound, $q=O(\eta^2)$ and $p=O(\eta^2L^2(1+\kappa))$
are small enough so that $(A+c_1\xi)q\le \eta/2$, hence $c_2>0$.
Therefore,
\[
\eta \le \min\Biggl\{
\frac{1}{2\sqrt{2}\,L\sqrt{1+2\kappa}},
\;
\frac{\omega}{256\,L\,(1+\kappa)}
\Biggr\}
\]
is sufficient for $1-\alpha_\eta>0$, $c_1>0$, and $c_2>0$.
\end{proof}

\subsection{Accelerating Algorithm: Theoretical Analysis}
In the accelerated algorithm, we assign distinct constant step size $\eta$ to different blocks of the global variable $\boldsymbol{\lambda}$. 
Assume that $\boldsymbol{\lambda}$ is partitioned into $B$ blocks, and that all coordinates within the same block share the same step size. 
We use $(\boldsymbol{\lambda}_i)_j$ to denote the $j$-th block of the global variable $\boldsymbol{\lambda}$ for function $i$.
Define
\begin{align*}
G_t 
&:= 
\frac{1}{n}\sum_{i=1}^n 
\mathbb{E}\!\left[
\left\|
\boldsymbol{\lambda}_i^t - \boldsymbol{\gamma}^t
\right\|^2
\right], 
\qquad
\delta_t 
:= 
\frac{1}{n}\sum_{i=1}^n 
\mathbb{E}\!\left[
\left\|
\boldsymbol{\phi}_i^{t-1} - \tilde{\boldsymbol{\phi}}_i^{t}
\right\|^2
\right].
\end{align*}
Moreover, define
\begin{align*}
\epsilon_t 
&:= 
\frac{1}{n}\sum_{i=1}^n 
\mathbb{E}\!\left[
\left\|
\tilde{\boldsymbol{\lambda}}_i^t - \boldsymbol{\gamma}^{t-1}
\right\|^2
\right], 
\qquad
(\epsilon_t)_j 
:=
\frac{1}{n}\sum_{i=1}^n 
\mathbb{E}\!\left[
\left\|
(\tilde{\boldsymbol{\lambda}}_i^t)_j - (\boldsymbol{\gamma}^{t-1})_j
\right\|^2
\right].
\end{align*}
If $\|\cdot\|$ denotes the (block-separable) Euclidean norm, then
\[
G_t=\sum_{j=1}^B (G_t)_j,
\qquad
\epsilon_t=\sum_{j=1}^B (\epsilon_t)_j.
\]
So we have $\epsilon_t = \sum_{j=1}^B (\epsilon_t)_j$.
In the nonconvex case, we make the following two assumptions. 
\begin{assumption}[Block-wise smoothness] 
For each $i\in[n]$, the function $f_i(\boldsymbol{\phi},\boldsymbol{\lambda})$ is continuously differentiable.
Moreover, there exist constants $L_{\phi}\in(0,\infty)$ and $\{L_{\lambda,j}\}_{j=1}^B \subset (0,\infty)$ such that for all
$(\boldsymbol{\phi}_1,\boldsymbol{\lambda}_1),(\boldsymbol{\phi}_2,\boldsymbol{\lambda}_2)\in\mathbb{R}^{d_\phi}\times\mathbb{R}^{d_\lambda}$,
\begin{align*}
\left\|
\nabla_{\boldsymbol{\phi}} f_i(\boldsymbol{\phi}_1,\boldsymbol{\lambda}_1)
-
\nabla_{\boldsymbol{\phi}} f_i(\boldsymbol{\phi}_2,\boldsymbol{\lambda}_2)
\right\|
&\le
L_{\phi}\,
\left\|
(\boldsymbol{\phi}_1,\boldsymbol{\lambda}_1)
-
(\boldsymbol{\phi}_2,\boldsymbol{\lambda}_2)
\right\|,\\[3pt]
\left\|
\bigl(\nabla_{\boldsymbol{\lambda}} f_i(\boldsymbol{\phi}_1,\boldsymbol{\lambda}_1)\bigr)_j
-
\bigl(\nabla_{\boldsymbol{\lambda}} f_i(\boldsymbol{\phi}_2,\boldsymbol{\lambda}_2)\bigr)_j
\right\|
&\le
L_{\lambda,j}\,
\left\|
(\boldsymbol{\phi}_1,\boldsymbol{\lambda}_1)
-
(\boldsymbol{\phi}_2,\boldsymbol{\lambda}_2)
\right\|,
\qquad \forall j\in[B].
\end{align*}
Here, $(\cdot)_j$ denotes the $j$-th block of a vector under the partition of $\boldsymbol{\lambda}$.
\end{assumption}

\begin{assumption} 
For each $i\in[n]$, the function $f_i(\boldsymbol{\phi},\boldsymbol{\lambda})$ is twice continuously differentiable
and there exists
a constant $\mu\in(0,\infty)$ such that for all $i\in[n]$ and all $(\boldsymbol{\phi},\boldsymbol{\lambda})\in\mathbb{R}^{d_\phi}\times\mathbb{R}^{d_\lambda}$,
\[
\nabla^2_{\boldsymbol{\phi}\boldsymbol{\phi}} f_i(\boldsymbol{\phi},\boldsymbol{\lambda}) \succeq \mu \mathbf{I}_{d_\phi}.
\]
\end{assumption}
\begin{theorem}[Nonconvex convergence with block-wise step size]
\label{thm:nonconvex_main_block}
Let $r:=\frac{m}{n}\in(0,1)$ and $\kappa_\phi:=\frac{L_{\boldsymbol{\phi}}^2}{\mu^2}$.
Suppose all assumptions in Lemma~\ref{lemma:main_nonconvex_block} hold.
Define
\[
\eta_{\min}:=\min_{j\in[B]}\eta_j,
\qquad
\eta_{\max}:=\max_{j\in[B]}\eta_j,
\qquad
S:=\sum_{j=1}^B \eta_jL_j^2,
\quad \rho:=\frac{2(n-m)}{2n-m}\in(0,1),
\qquad
\xi:=\frac{2m}{2n-m}+\frac{2n}{m}.
\]
Let $c_1,c_2$ be the constants defined in Lemma~\ref{lemma:main_nonconvex_block}.
Assume the step size satisfies
\begin{align*}
\sum_{j=1}^B \eta_j^2 L_j^2
<
\frac{1}{4(1+2\kappa_\phi)}, \quad
8\xi(1+\kappa_\phi)\sum_{j=1}^B \eta_j^2 L_j^2
<
(1-\rho)\,D, \quad
8(E+c_1\xi)\,\eta_{\max}
<
D.
\end{align*}
Then for the iterates $\{(\boldsymbol{\phi}^t,\boldsymbol{\gamma}^t)\}_{t=0}^T$
generated by the algorithm, we have
\begin{align}
\mathbb{E}\!\left[\frac{1}{T}\sum_{t=1}^T
\Bigl\|
\nabla_{\boldsymbol{\lambda}} f(\boldsymbol{\phi}^{t-1},\boldsymbol{\gamma}^{t-1})
\Bigr\|_{\mathbf{H}}^2\right]
\;\le\;
\frac{1}{T}\frac{1}{c_2}\Bigl(
\mathbb{E}\!\left[f(\boldsymbol{\phi}^{0},\boldsymbol{\gamma}^{0})\right]-f^*
+c_1\,G_0
\Bigr).
\label{eq:block_weighted_rate}
\end{align}
Moreover, this can be further extended to
\begin{align}
\sum_{j=1}^B \eta_j \mathbb{E}\!\left[\frac{1}{T}\sum_{t=1}^T
\Bigl\|
(\nabla_{\boldsymbol{\lambda}} f(\boldsymbol{\phi}^{t-1},\boldsymbol{\gamma}^{t-1}))_j
\Bigr\|^2\right]
\;\le\;
\frac{1}{T}\frac{1}{c_2}\Bigl(
\mathbb{E}\!\left[f(\boldsymbol{\phi}^{0},\boldsymbol{\gamma}^{0})\right]-f^*
+c_1\,G_0
\Bigr).
\label{eq:block_euclidean_rate}
\end{align}
In particular, under any step size regime for which $c_2=\Theta(1)$ and
$c_1=O\!\bigl(nL(1+\kappa_\phi)\bigr)$, we have
\[
\sum_{j=1}^B \eta_j \mathbb{E}\!\left[\frac{1}{T}\sum_{t=1}^T
\Bigl\|
(\nabla_{\boldsymbol{\lambda}} f(\boldsymbol{\phi}^{t-1},\boldsymbol{\gamma}^{t-1}))_j
\Bigr\|^2\right]
\;=\;
O\!\left(
\frac{1}{T}\Bigl(
\mathbb{E}[f(\boldsymbol{\phi}^{0},\boldsymbol{\gamma}^{0})]-f^*
+
nL(1+\kappa_\phi)G_0
\Bigr)\right).
\]
\end{theorem}

\begin{proof}
Lemma~\ref{lemma:main_nonconvex_block} gives:
\begin{align}
\mathbb{E}\!\left[f(\boldsymbol{\phi}^t,\boldsymbol{\gamma}^t)\right]
+c_1G_t
\le
\mathbb{E}\!\left[f(\boldsymbol{\phi}^{t-1},\boldsymbol{\gamma}^{t-1})\right]
+c_1G_{t-1}
-
c_2\,
\mathbb{E}\Bigl\|
\nabla_{\boldsymbol{\lambda}} f(\boldsymbol{\phi}^{t-1},\boldsymbol{\gamma}^{t-1})
\Bigr\|_{\mathbf{H}}^2 .
\label{eq:lyapunov_step}
\end{align}
Rearranging \eqref{eq:lyapunov_step} yields
\begin{align}
c_2\,
\mathbb{E}\Bigl\|
\nabla_{\boldsymbol{\lambda}} f(\boldsymbol{\phi}^{t-1},\boldsymbol{\gamma}^{t-1})
\Bigr\|_{\mathbf{H}}^2
\le
\Bigl(
\mathbb{E}\!\left[f(\boldsymbol{\phi}^{t-1},\boldsymbol{\gamma}^{t-1})\right]-f^*
+c_1G_{t-1}
\Bigr)
-
\Bigl(
\mathbb{E}\!\left[f(\boldsymbol{\phi}^{t},\boldsymbol{\gamma}^{t})\right]-f^*
+c_1G_{t}
\Bigr).
\label{eq:lyapunov_rearrange}
\end{align}
Summing \eqref{eq:lyapunov_rearrange} over $t=1,\ldots,T$ gives a telescoping series:
\begin{align}
c_2\sum_{t=1}^T
\mathbb{E}\Bigl\|
\nabla_{\boldsymbol{\lambda}} f(\boldsymbol{\phi}^{t-1},\boldsymbol{\gamma}^{t-1})
\Bigr\|_{\mathbf{H}}^2
\le
\Bigl(
\mathbb{E}\!\left[f(\boldsymbol{\phi}^{0},\boldsymbol{\gamma}^{0})\right]-f^*
+c_1G_{0}
\Bigr)
-
\Bigl(
\mathbb{E}\!\left[f(\boldsymbol{\phi}^{T},\boldsymbol{\gamma}^{T})\right]-f^*
+c_1G_{T}
\Bigr).
\end{align}
Dropping the non-positive term on the right-hand side yields
\[
c_2\sum_{t=1}^T
\mathbb{E}\Bigl\|
\nabla_{\boldsymbol{\lambda}} f(\boldsymbol{\phi}^{t-1},\boldsymbol{\gamma}^{t-1})
\Bigr\|_{\mathbf{H}}^2
\le
\mathbb{E}\!\left[f(\boldsymbol{\phi}^{0},\boldsymbol{\gamma}^{0})\right]-f^*
+c_1G_{0}.
\]
Dividing both sides by $Tc_2$ proves \eqref{eq:block_weighted_rate}.

Finally, since $\mathbf{H}\succeq \eta_{\min}\mathbf{I}$, we have
$\|\boldsymbol{v}\|_{\mathbf{H}}^2 \ge \eta_{\min}\|\boldsymbol{v}\|^2$.
Applying this to \eqref{eq:block_weighted_rate} immediately yields
\eqref{eq:block_euclidean_rate}.
\end{proof}
\begin{lemma}[From block-wise smoothness to global smoothness] \label{lemma:block_to_global_smooth}
Suppose Assumption~\ref{ass:block_smooth_nonconvex} holds. 
Define the concatenated variable $\boldsymbol{z}:=(\boldsymbol{\phi},\boldsymbol{\lambda})\in\mathbb{R}^{d_\phi+d_\lambda}$ and equip $\mathbb{R}^{d_\phi+d_\lambda}$ with the Euclidean norm
\[
\|\boldsymbol{z}\|^2=\|\boldsymbol{\phi}\|^2+\|\boldsymbol{\lambda}\|^2
=\|\boldsymbol{\phi}\|^2+\sum_{j=1}^B \|(\boldsymbol{\lambda})_j\|^2.
\]
Then for each $i\in[n]$, the full gradient mapping $\nabla f_i(\boldsymbol{z})$ is globally Lipschitz continuous, i.e., for all 
$\boldsymbol{z}_1=(\boldsymbol{\phi}_1,\boldsymbol{\lambda}_1)$ and 
$\boldsymbol{z}_2=(\boldsymbol{\phi}_2,\boldsymbol{\lambda}_2)$,
\[
\bigl\|
\nabla f_i(\boldsymbol{z}_1)-\nabla f_i(\boldsymbol{z}_2)
\bigr\|
\;\le\;
L\,\|\boldsymbol{z}_1-\boldsymbol{z}_2\|,
\qquad
L:=\sqrt{L_\phi^2+\sum_{j=1}^B L_{\lambda,j}^2}.
\]
\end{lemma}

\begin{proof}
Let $\boldsymbol{z}_1=(\boldsymbol{\phi}_1,\boldsymbol{\lambda}_1)$ and $\boldsymbol{z}_2=(\boldsymbol{\phi}_2,\boldsymbol{\lambda}_2)$ be arbitrary.
By the definition of the full gradient and the block partition of $\boldsymbol{\lambda}$, we have
\[
\nabla f_i(\boldsymbol{z})
=
\begin{pmatrix}
\nabla_{\boldsymbol{\phi}} f_i(\boldsymbol{\phi},\boldsymbol{\lambda})\\
\nabla_{\boldsymbol{\lambda}} f_i(\boldsymbol{\phi},\boldsymbol{\lambda})
\end{pmatrix},
\qquad
\text{and}\qquad
\|\nabla_{\boldsymbol{\lambda}} f_i(\cdot)\|^2
=
\sum_{j=1}^B
\bigl\|(\nabla_{\boldsymbol{\lambda}} f_i(\cdot))_j\bigr\|^2.
\]
Therefore,
\begin{align*}
\bigl\|
\nabla f_i(\boldsymbol{z}_1)-\nabla f_i(\boldsymbol{z}_2)
\bigr\|^2
&=
\bigl\|
\nabla_{\boldsymbol{\phi}} f_i(\boldsymbol{z}_1)-\nabla_{\boldsymbol{\phi}} f_i(\boldsymbol{z}_2)
\bigr\|^2
+
\bigl\|
\nabla_{\boldsymbol{\lambda}} f_i(\boldsymbol{z}_1)-\nabla_{\boldsymbol{\lambda}} f_i(\boldsymbol{z}_2)
\bigr\|^2\\
&=
\bigl\|
\nabla_{\boldsymbol{\phi}} f_i(\boldsymbol{z}_1)-\nabla_{\boldsymbol{\phi}} f_i(\boldsymbol{z}_2)
\bigr\|^2
+
\sum_{j=1}^B
\bigl\|
(\nabla_{\boldsymbol{\lambda}} f_i(\boldsymbol{z}_1))_j-(\nabla_{\boldsymbol{\lambda}} f_i(\boldsymbol{z}_2))_j
\bigr\|^2.
\end{align*}
Applying Assumption~\ref{ass:block_smooth_nonconvex} yields
\[
\bigl\|
\nabla_{\boldsymbol{\phi}} f_i(\boldsymbol{z}_1)-\nabla_{\boldsymbol{\phi}} f_i(\boldsymbol{z}_2)
\bigr\|
\le
L_\phi \|\boldsymbol{z}_1-\boldsymbol{z}_2\|,
\qquad
\bigl\|
(\nabla_{\boldsymbol{\lambda}} f_i(\boldsymbol{z}_1))_j-(\nabla_{\boldsymbol{\lambda}} f_i(\boldsymbol{z}_2))_j
\bigr\|
\le
L_{\lambda,j}\|\boldsymbol{z}_1-\boldsymbol{z}_2\|.
\]
Substituting the above bounds gives
\[
\bigl\|
\nabla f_i(\boldsymbol{z}_1)-\nabla f_i(\boldsymbol{z}_2)
\bigr\|^2
\le
\left(
L_\phi^2+\sum_{j=1}^B L_{\lambda,j}^2
\right)
\|\boldsymbol{z}_1-\boldsymbol{z}_2\|^2.
\]
Taking square roots on both sides completes the proof.
\end{proof}

\begin{lemma}
\begin{align*}
(\epsilon_t)_j
&\le
4\eta_j^2 L_j^2\!\left(
\Bigl(2+\frac{2L_{\boldsymbol{\phi}}^2}{\mu^2}\Bigr)G_{t-1}
+
\Bigl(1+\frac{2L_{\boldsymbol{\phi}}^2}{\mu^2}\Bigr)\epsilon_t
\right)
+
4\eta_j^2\,
\mathbb{E}\!\left\|
(\nabla_{\boldsymbol{\lambda}_i} f(\boldsymbol{\phi}^{t-1},\boldsymbol{\gamma}^{\,t-1}))_j
\right\|^2 .
\end{align*}
\begin{proof}
For each $j \in [B]$, we can derive an analogous result to Lemma \ref{lemma:epsilon_nonconvex}. 
\begin{align*}
(\epsilon_t)_j
&\le
\frac{8\eta_j^2}{n}\sum_{i=1}^n
\mathbb{E}\!\left\|
(\nabla_{\boldsymbol{\lambda}_i} f_i(\boldsymbol{\phi}_i^{t-1},\boldsymbol{\lambda}_i^{t-1}))_j
-
(\nabla_{\boldsymbol{\lambda}_i} f_i(\boldsymbol{\phi}_i^{t-1},\boldsymbol{\gamma}^{\,t-1}))_j
\right\|^2
+
\frac{4\eta_j^2}{n}\sum_{i=1}^n
\mathbb{E}\!\left\|
(\nabla_{\boldsymbol{\lambda}_i} f_i(\boldsymbol{\phi}_i^{t-1},\boldsymbol{\gamma}^{\,t-1}))_j
-
(\nabla_{\boldsymbol{\lambda}_i} f_i(\tilde{\boldsymbol{\phi}}_i^{t},\tilde{\boldsymbol{\lambda}}_i^{t}))_j
\right\|^2.
\\
&\quad+
4\eta_j^2\,
\mathbb{E}\!\left\|
(\nabla_{\boldsymbol{\lambda}_i} f(\boldsymbol{\phi}^{t-1},\boldsymbol{\gamma}^{t-1}))_j
\right\|^2.
\end{align*}
Since we have 
\begin{align*}
&\left\|
(\nabla_{\boldsymbol{\lambda}_i} f_i(\boldsymbol{\phi}_i^{t-1},\boldsymbol{\lambda}_i^{t-1}))_j
-
(\nabla_{\boldsymbol{\lambda}_i} f_i(\boldsymbol{\phi}_i^{t-1},\boldsymbol{\gamma}^{\,t-1}))_j
\right\|^2 \leq L_j^2 \|\boldsymbol{\lambda}_i^{t-1} - \boldsymbol{\gamma}^{\,t-1}\|^2,\\
&\left\|
(\nabla_{\boldsymbol{\lambda}_i} f_i(\boldsymbol{\phi}_i^{t-1},\boldsymbol{\gamma}^{\,t-1}))_j
-
(\nabla_{\boldsymbol{\lambda}_i} f_i(\tilde{\boldsymbol{\phi}}_i^{t},\tilde{\boldsymbol{\lambda}}_i^{t}))_j
\right\|^2 \leq  L_j^2 \left(\|\tilde{\boldsymbol{\lambda}}_i^{t} - \boldsymbol{\gamma}^{\,t-1}\|^2 + \|\tilde{\boldsymbol{\phi}}_i^{t} - \boldsymbol{\phi}^{\,t-1}\|^2 \right),
\end{align*}
we can obtain
\begin{align*}
(\epsilon_t)_j
&\le
8\eta_j^2 L_j^2\,G_{t-1}
+
4\eta_j^2 L_j^2\,(\epsilon_t+\delta_t)
+
4\eta_j^2\,
\mathbb{E}\!\left\|
(\nabla_{\boldsymbol{\lambda}_i} f(\boldsymbol{\phi}^{t-1},\boldsymbol{\gamma}^{\,t-1}))_j
\right\|^2.
\end{align*}
We can extend Lemma~\ref{lemma:sensitivity} and Lemma~\ref{lemma:nonconvex_delta_bound} by replacing $L$ with $L_{\boldsymbol{\phi}}$, so we have
\begin{align}\label{eq:delta_bound}
\delta_t
\;\le\;
\frac{2L_{\boldsymbol{\phi}}^2}{\mu^2}\bigl(G_{t-1}+\epsilon_t\bigr).
\end{align}
Substituting \eqref{eq:delta_bound} into the previous bound and rearranging terms yield
\begin{align*}
(\epsilon_t)_j
&\le
8\eta_j^2 L_j^2\,G_{t-1}
+
4\eta_j^2 L_j^2\!\left(
\epsilon_t
+
\frac{2L_{\boldsymbol{\phi}}^2}{\mu^2}\bigl(G_{t-1}+\epsilon_t\bigr)
\right)
+
4\eta_j^2\,
\mathbb{E}\!\left\|
(\nabla_{\boldsymbol{\lambda}_i} f(\boldsymbol{\phi}^{t-1},\boldsymbol{\gamma}^{\,t-1}))_j
\right\|^2 \\[3pt]
&=
4\eta_j^2 L_j^2\!\left(
\Bigl(2+\frac{2L_{\boldsymbol{\phi}}^2}{\mu^2}\Bigr)G_{t-1}
+
\Bigl(1+\frac{2L_{\boldsymbol{\phi}}^2}{\mu^2}\Bigr)\epsilon_t
\right)
+
4\eta_j^2\,
\mathbb{E}\!\left\|
(\nabla_{\boldsymbol{\lambda}_i} f(\boldsymbol{\phi}^{t-1},\boldsymbol{\gamma}^{\,t-1}))_j
\right\|^2 .
\end{align*}
\end{proof}
\end{lemma}

\begin{lemma}
Let
\(
\mathbf{H}
:=
\mathrm{blkdiag}\bigl(\eta_1\mathbf{I}_{d_1},\ldots,\eta_B\mathbf{I}_{d_B}\bigr),
S
:=
\sum_{j=1}^B \eta_j L_j^2,
\)
and define the $\mathbf{H}$-weighted norm by
\(
\|\boldsymbol{v}\|_{\mathbf{H}}^2
:=
\boldsymbol{v}^\top \mathbf{H}\boldsymbol{v}.
\)
Then we have
\begin{align} \nonumber
\mathbb{E}\!\left[f(\boldsymbol{\phi}^t,\boldsymbol{\gamma}^t)\right]
-
\mathbb{E}\!\left[f(\boldsymbol{\phi}^{t-1},\boldsymbol{\gamma}^{t-1})\right]
\le\;
&\Biggl(
\frac{L}{2}
+
\frac{S}{2}
+
\frac{L_{\boldsymbol{\phi}}^2}{\mu^2}\bigl(mL+S\bigr)
\Biggr)\epsilon_t
+
\Biggl(
\frac{Lm}{2}
+
\frac{L_{\boldsymbol{\phi}}^2}{\mu^2}\bigl(mL+S\bigr)
\Biggr) G_{t-1}
\\
&\;-\frac{1}{2}\,
\mathbb{E}\Bigl\|
\nabla_{\boldsymbol{\lambda}} f(\boldsymbol{\phi}^{t-1},\boldsymbol{\gamma}^{t-1})
\Bigr\|_{\mathbf{H}}^2.
\end{align}
\begin{proof}
By \(L\)-smoothness of \(f\), where $L:=\sqrt{L_\phi^2+\sum_{j=1}^B L_{\lambda,j}^2}$, we have
\begin{align*}
\mathbb{E}\!\left[f(\boldsymbol{\phi}^t,\boldsymbol{\gamma}^t)\right]
-
\mathbb{E}\!\left[f(\boldsymbol{\phi}^{t-1},\boldsymbol{\gamma}^{t-1})\right]
&\le
\mathbb{E}\!\left[
\left\langle
\nabla_\lambda f(\boldsymbol{\phi}^{t-1},\boldsymbol{\gamma}^{t-1}),
\boldsymbol{\gamma}^t-\boldsymbol{\gamma}^{t-1}
\right\rangle
\right]+
\mathbb{E}\!\left[
\left\langle
\nabla_\phi f(\boldsymbol{\phi}^{t-1},\boldsymbol{\gamma}^{t-1}),
\boldsymbol{\phi}^t-\boldsymbol{\phi}^{t-1}
\right\rangle
\right] \\
&\quad+
\frac{L}{2}\mathbb{E}\|\boldsymbol{\gamma}^t-\boldsymbol{\gamma}^{t-1}\|^2
+
\frac{L}{2}\mathbb{E}\|\boldsymbol{\phi}^t-\boldsymbol{\phi}^{t-1}\|^2 .
\end{align*}

Since
\[
\mathbb{E}\|\boldsymbol{\phi}^t-\boldsymbol{\phi}^{t-1}\|^2
=
\sum_{i=1}^n \frac{m}{n}\,
\mathbb{E}\|\tilde{\boldsymbol{\phi}}_i^t-\boldsymbol{\phi}_i^{t-1}\|^2
=
m\,\delta_t,
\]
we obtain
\begin{align*}
& \quad \;\mathbb{E}\!\left[f(\boldsymbol{\phi}^t,\boldsymbol{\gamma}^t)\right]
-
\mathbb{E}\!\left[f(\boldsymbol{\phi}^{t-1},\boldsymbol{\gamma}^{t-1})\right]
-
\frac{L}{2}\mathbb{E}\|\boldsymbol{\gamma}^t-\boldsymbol{\gamma}^{t-1}\|^2
-
\frac{mL}{2}\delta_t\\
&\le
\mathbb{E}\!\left[
\left\langle
\nabla_\lambda f(\boldsymbol{\phi}^{t-1},\boldsymbol{\gamma}^{t-1}),
\boldsymbol{\gamma}^t-\boldsymbol{\gamma}^{t-1}
\right\rangle
\right] +
\mathbb{E}\!\left[
\left\langle
\nabla_\phi f(\boldsymbol{\phi}^{t-1},\boldsymbol{\gamma}^{t-1}),
\boldsymbol{\phi}^t-\boldsymbol{\phi}^{t-1}
\right\rangle
\right],\\
&\le
\eta\,\mathbb{E}\!\left[
\left\langle
\nabla_\lambda f(\boldsymbol{\phi}^{t-1},\boldsymbol{\gamma}^{t-1}),
\frac{1}{n}\sum_{i=1}^n
\bigl(-\nabla_{\boldsymbol{\lambda}_i} f_i(\tilde{\boldsymbol{\phi}}_i^t,\tilde{\boldsymbol{\lambda}}_i^t)\bigr)
\right\rangle
\right]
+
\frac{m}{n}\sum_{i=1}^n
\mathbb{E}\!\left[
\left\langle
\nabla_{\boldsymbol{\phi}_i} f_i(\boldsymbol{\phi}_i^{t-1},\boldsymbol{\gamma}^{t-1}),
\tilde{\boldsymbol{\phi}}_i^t-\boldsymbol{\phi}_i^{t-1}
\right\rangle
\right]. \quad \left(\text{by Lemma \ref{lemma:gamma_expectation}} \right)
\end{align*}
Define the block-wise step size matrix
\[
\mathbf{H}
:=
\mathrm{blkdiag}\bigl(\eta_1\mathbf{I}_{d_1},\ldots,\eta_B\mathbf{I}_{d_B}\bigr),
\qquad
\|\boldsymbol{v}\|_{\mathbf{H}}^2 := \boldsymbol{v}^\top \mathbf{H}\boldsymbol{v}
= \sum_{j=1}^B \eta_j \|(\boldsymbol{v})_j\|^2 .
\]

By extending Lemma~\ref{lemma:gamma_expectation}, we have
\begin{align}
\mathbb{E}\!\left[\boldsymbol{\gamma}^{t}-\boldsymbol{\gamma}^{t-1}\right]
&=
-\frac{1}{n}\sum_{i=1}^n
\mathbb{E}\!\left[
\mathbf{H} \,\nabla_{\boldsymbol{\lambda}_i}
f_i\!\left(\tilde{\boldsymbol{\phi}}_i^{t}, \tilde{\boldsymbol{\lambda}}_i^{t}\right)
\right].
\end{align}
Equivalently, for each $j\in[B]$,
\begin{align*}
\mathbb{E}\!\left[(\boldsymbol{\gamma}^{t})_j-(\boldsymbol{\gamma}^{t-1})_j\right]
&=
-\frac{\eta_j}{n}\sum_{i=1}^n
\mathbb{E}\!\left[
\bigl(\nabla_{\!\boldsymbol{\lambda}_i}
f_i\!(\tilde{\boldsymbol{\phi}}_i^{t}, \tilde{\boldsymbol{\lambda}}_i^{t})\bigr)_j
\right].
\end{align*}

As a result,
\begin{align*}
&\quad
\mathbb{E}\!\left[f(\boldsymbol{\phi}^t,\boldsymbol{\gamma}^t)\right]
-
\mathbb{E}\!\left[f(\boldsymbol{\phi}^{t-1},\boldsymbol{\gamma}^{t-1})\right]
-
\frac{L}{2}\mathbb{E}\|\boldsymbol{\gamma}^t-\boldsymbol{\gamma}^{t-1}\|^2
-
\frac{mL}{2}\delta_t \\
&\le
\mathbb{E}\!\left[
\left\langle
\nabla_{\boldsymbol{\lambda}} f(\boldsymbol{\phi}^{t-1},\boldsymbol{\gamma}^{t-1}),
-\frac{1}{n}\sum_{i=1}^n
\mathbf{H}\,\nabla_{\boldsymbol{\lambda}_i} f_i(\tilde{\boldsymbol{\phi}}_i^t,\tilde{\boldsymbol{\lambda}}_i^t)
\right\rangle
\right]
+
\frac{m}{n}\sum_{i=1}^n
\mathbb{E}\!\left[
\left\langle
\nabla_{\boldsymbol{\phi}_i} f_i(\boldsymbol{\phi}_i^{t-1},\boldsymbol{\gamma}^{t-1}),
\tilde{\boldsymbol{\phi}}_i^t-\boldsymbol{\phi}_i^{t-1}
\right\rangle
\right].
\end{align*}

Moreover,
\begin{align*}
&\quad \mathbb{E}\!\left[
\left\langle
\nabla_{\boldsymbol{\lambda}} f(\boldsymbol{\phi}^{t-1},\boldsymbol{\gamma}^{t-1}),
\frac{1}{n}\sum_{i=1}^n
\bigl(-\mathbf{H}\nabla_{\boldsymbol{\lambda}_i} f_i(\tilde{\boldsymbol{\phi}}_i^t,\tilde{\boldsymbol{\lambda}}_i^t)\bigr)
\right\rangle
\right] \\
&=
\mathbb{E}\!\left[
\left\langle
\nabla_{\boldsymbol{\lambda}} f(\boldsymbol{\phi}^{t-1},\boldsymbol{\gamma}^{t-1}),
\mathbf{H}\!\left(
\nabla_{\boldsymbol{\lambda}} f(\boldsymbol{\phi}^{t-1},\boldsymbol{\gamma}^{t-1})
-
\frac{1}{n}\sum_{i=1}^n
\nabla_{\boldsymbol{\lambda}_i} f_i(\tilde{\boldsymbol{\phi}}_i^t,\tilde{\boldsymbol{\lambda}}_i^t)
\right)
\right\rangle
\right]
-
\mathbb{E}\!\left\|
\nabla_{\boldsymbol{\lambda}} f(\boldsymbol{\phi}^{t-1},\boldsymbol{\gamma}^{t-1})
\right\|_{\mathbf{H}}^2 \\
&\le
\frac{1}{2}\,
\mathbb{E}\!\left[
\left\|
\nabla_{\boldsymbol{\lambda}} f(\boldsymbol{\phi}^{t-1},\boldsymbol{\gamma}^{t-1})
\right\|_{\mathbf{H}}^2
+
\left\|
\nabla_{\boldsymbol{\lambda}} f(\boldsymbol{\phi}^{t-1},\boldsymbol{\gamma}^{t-1})
-
\frac{1}{n}\sum_{i=1}^n
\nabla_{\boldsymbol{\lambda}_i} f_i(\tilde{\boldsymbol{\phi}}_i^t,\tilde{\boldsymbol{\lambda}}_i^t)
\right\|_{\mathbf{H}}^2
\right]
-
\mathbb{E}\!\left\|
\nabla_{\boldsymbol{\lambda}} f(\boldsymbol{\phi}^{t-1},\boldsymbol{\gamma}^{t-1})
\right\|_{\mathbf{H}}^2 \\
&=
\frac{1}{2}\,
\mathbb{E}\!\left[
\left\|
\nabla_{\boldsymbol{\lambda}} f(\boldsymbol{\phi}^{t-1},\boldsymbol{\gamma}^{t-1})
-
\frac{1}{n}\sum_{i=1}^n
\nabla_{\boldsymbol{\lambda}_i} f_i(\tilde{\boldsymbol{\phi}}_i^t,\tilde{\boldsymbol{\lambda}}_i^t)
\right\|_{\mathbf{H}}^2
\right]
-
\frac{1}{2}\,
\mathbb{E}\!\left\|
\nabla_{\boldsymbol{\lambda}} f(\boldsymbol{\phi}^{t-1},\boldsymbol{\gamma}^{t-1})
\right\|_{\mathbf{H}}^2 \\
&\le
\frac{1}{2n}\sum_{i=1}^n
\mathbb{E}\!\left\|
\mathbf{H}^{1/2}\!\left(
\nabla_{\boldsymbol{\lambda}_i} f_i(\boldsymbol{\phi}_i^{t-1},\boldsymbol{\gamma}^{t-1})
-
\nabla_{\boldsymbol{\lambda}_i} f_i(\tilde{\boldsymbol{\phi}}_i^{t},\tilde{\boldsymbol{\lambda}}_i^{t})
\right)
\right\|^2
-
\frac{1}{2}\,
\mathbb{E}\!\left\|
\nabla_{\boldsymbol{\lambda}} f(\boldsymbol{\phi}^{t-1},\boldsymbol{\gamma}^{t-1})
\right\|_{\mathbf{H}}^2 \\
&=
\frac{1}{2n}\sum_{i=1}^n
\sum_{j=1}^B \eta_j\,
\mathbb{E}\!\left\|
\bigl(\nabla_{\boldsymbol{\lambda}_i} f_i(\boldsymbol{\phi}_i^{t-1},\boldsymbol{\gamma}^{t-1})\bigr)_j
-
\bigl(\nabla_{\boldsymbol{\lambda}_i} f_i(\tilde{\boldsymbol{\phi}}_i^{t},\tilde{\boldsymbol{\lambda}}_i^{t})\bigr)_j
\right\|^2
-
\frac{1}{2}\,
\mathbb{E}\!\left\|
\nabla_{\boldsymbol{\lambda}} f(\boldsymbol{\phi}^{t-1},\boldsymbol{\gamma}^{t-1})
\right\|_{\mathbf{H}}^2 \\
&\le
\frac{1}{2n}\sum_{i=1}^n
\sum_{j=1}^B \eta_j L_j^2\,
\mathbb{E}\!\left[
\|\boldsymbol{\phi}_i^{t-1} - \tilde{\boldsymbol{\phi}}_i^{t}\|^2
+
\|\boldsymbol{\gamma}^{t-1} - \tilde{\boldsymbol{\lambda}}_i^{t}\|^2
\right]
-
\frac{1}{2}\,
\mathbb{E}\!\left\|
\nabla_{\boldsymbol{\lambda}} f(\boldsymbol{\phi}^{t-1},\boldsymbol{\gamma}^{t-1})
\right\|_{\mathbf{H}}^2 ,
\qquad (\text{by block-wise smoothness})\\
& = \frac{1}{2}\left(\sum_{j=1}^B \eta_j L_j^2 \right)(\delta_t + \epsilon_t) - \frac{1}{2}\,
\mathbb{E}\!\left\|
\nabla_{\boldsymbol{\lambda}} f(\boldsymbol{\phi}^{t-1},\boldsymbol{\gamma}^{t-1})
\right\|_{\mathbf{H}}^2.
\end{align*}
And the same as Lemma~\ref{lemma:combine_nonconvex}, we have
\begin{align*}
\frac{m}{n}\sum_{i=1}^n
\mathbb{E}\!\left[
\left\langle
\nabla_{\boldsymbol{\phi}_i} f_i(\boldsymbol{\phi}_i^{t-1},\boldsymbol{\gamma}^{t-1}),
\tilde{\boldsymbol{\phi}}_i^t-\boldsymbol{\phi}_i^{t-1}
\right\rangle
\right] & \leq \frac{Lm}{2}(G_{t-1} + \delta_t).
\end{align*}
Combining all of these, and letting
\(
S \;:=\; \sum_{j=1}^B \eta_j L_j^2,
\)
we obtain
\begin{align*}
&\quad
\mathbb{E}\!\left[f(\boldsymbol{\phi}^t,\boldsymbol{\gamma}^t)\right]
-
\mathbb{E}\!\left[f(\boldsymbol{\phi}^{t-1},\boldsymbol{\gamma}^{t-1})\right]
-
\frac{L}{2}\mathbb{E}\|\boldsymbol{\gamma}^t-\boldsymbol{\gamma}^{t-1}\|^2
-
\frac{mL}{2}\delta_t \\
&\le
\frac{1}{2}\left(\sum_{j=1}^B \eta_j L_j^2 \right)(\delta_t + \epsilon_t)
-
\frac{1}{2}\,
\mathbb{E}\!\left\|
\nabla_{\boldsymbol{\lambda}} f(\boldsymbol{\phi}^{t-1},\boldsymbol{\gamma}^{t-1})
\right\|_{\mathbf{H}}^2
+
\frac{Lm}{2}(G_{t-1} + \delta_t).
\end{align*}
Rearranging terms yields
\begin{align*}
\mathbb{E}\!\left[f(\boldsymbol{\phi}^t,\boldsymbol{\gamma}^t)\right]
-
\mathbb{E}\!\left[f(\boldsymbol{\phi}^{t-1},\boldsymbol{\gamma}^{t-1})\right]
&\le
\frac{L}{2}\mathbb{E}\|\boldsymbol{\gamma}^t-\boldsymbol{\gamma}^{t-1}\|^2
+
\Bigl(\frac{mL}{2}+\frac{S}{2}\Bigr)\delta_t
+
\frac{S}{2}\epsilon_t
+
\frac{Lm}{2}G_{t-1} \\
&\quad
-
\frac{1}{2}\,
\mathbb{E}\!\left\|
\nabla_{\boldsymbol{\lambda}} f(\boldsymbol{\phi}^{t-1},\boldsymbol{\gamma}^{t-1})
\right\|_{\mathbf{H}}^2.
\end{align*}
By \eqref{eq:delta_bound}, we have
\(
\delta_t
\;\le\;
\frac{2L_{\boldsymbol{\phi}}^2}{\mu^2}\bigl(G_{t-1}+\epsilon_t\bigr),
\)
and hence
\begin{align*}
& \quad \mathbb{E}\!\left[f(\boldsymbol{\phi}^t,\boldsymbol{\gamma}^t)\right]
-
\mathbb{E}\!\left[f(\boldsymbol{\phi}^{t-1},\boldsymbol{\gamma}^{t-1})\right]\\
&\le
\frac{L}{2}\mathbb{E}\|\boldsymbol{\gamma}^t-\boldsymbol{\gamma}^{t-1}\|^2
+
\Biggl(
\frac{Lm}{2}
+
\frac{L_{\boldsymbol{\phi}}^2}{\mu^2}\bigl(mL+S\bigr)
\Biggr)G_{t-1} +
\Biggl(
\frac{S}{2}
+
\frac{L_{\boldsymbol{\phi}}^2}{\mu^2}\bigl(mL+S\bigr)
\Biggr)\epsilon_t
-
\frac{1}{2}\,
\mathbb{E}\!\left\|
\nabla_{\boldsymbol{\lambda}} f(\boldsymbol{\phi}^{t-1},\boldsymbol{\gamma}^{t-1})
\right\|_{\mathbf{H}}^2, \\
&\le\;
\Biggl(
\frac{L}{2}
+
\frac{S}{2}
+
\frac{L_{\boldsymbol{\phi}}^2}{\mu^2}\bigl(mL+S\bigr)
\Biggr)\epsilon_t +
\Biggl(
\frac{Lm}{2}
+
\frac{L_{\boldsymbol{\phi}}^2}{\mu^2}\bigl(mL+S\bigr)
\Biggr)G_{t-1}
-
\frac{1}{2}\,
\mathbb{E}\!\left\|
\nabla_{\boldsymbol{\lambda}} f(\boldsymbol{\phi}^{t-1},\boldsymbol{\gamma}^{t-1})
\right\|_{\mathbf{H}}^2.
\end{align*}
\end{proof}
\end{lemma}

\begin{lemma}\label{lemma:main_nonconvex_block}
Let $\kappa_\phi := \frac{L_{\boldsymbol{\phi}}^2}{\mu^2}$ and define
\(
\mathbf{H}
:=
\mathrm{blkdiag}\bigl(\eta_1\mathbf{I}_{d_1},\ldots,\eta_B\mathbf{I}_{d_B}\bigr),
\qquad
\|\boldsymbol{v}\|_{\mathbf{H}}^2
:=
\sum_{j=1}^B \eta_j\|(\boldsymbol{v})_j\|^2 .
\)
Let
\(
\rho:=\frac{2(n-m)}{2n-m},
\qquad
\xi:=\frac{2m}{2n-m}+\frac{2n}{m},
\qquad
\eta_{\max}:=\max_{j\in[B]}\eta_j,
\)
\(
S:=\sum_{j=1}^B \eta_j L_j^2,
\qquad
E:=\frac{L}{2}+\frac{S}{2}+\kappa_\phi(mL+S),
\qquad
F:=\frac{Lm}{2}+\kappa_\phi(mL+S),
\)
where $L := \sqrt{L_\phi^2+\sum_{j=1}^B L_j^2}$.
Define
\(
D:=1-4(1+2\kappa_\phi)\sum_{j=1}^B \eta_j^2 L_j^2,
\qquad
p:=\frac{8(1+\kappa_\phi)\sum_{j=1}^B \eta_j^2 L_j^2}{D},
\qquad
q:=\frac{4\eta_{\max}}{D},
\)
and choose
\(
c_1
:=
-\frac{Ep+F}{\xi p+(\rho-1)},
\qquad
c_2
:=
\frac{1}{2}-(E+c_1\xi)\,q.
\)

Assume the step size satisfy
\begin{align}
\sum_{j=1}^B \eta_j^2 L_j^2
&<
\frac{1}{4(1+2\kappa_\phi)},\quad
8\xi(1+\kappa_\phi)\sum_{j=1}^B \eta_j^2 L_j^2
<
(1-\rho)\,D,\quad
8(E+c_1\xi)\,\eta_{\max}
<
D.
\label{eq:eta_conditions_all}
\end{align}
Then $D>0$, $c_1>0$, and $c_2>0$, and the iterates satisfy
\begin{align}
\mathbb{E}\!\left[f(\boldsymbol{\phi}^t,\boldsymbol{\gamma}^t)\right]
+c_1\,G_t
\le
\mathbb{E}\!\left[f(\boldsymbol{\phi}^{t-1},\boldsymbol{\gamma}^{t-1})\right]
+c_1\,G_{t-1}
-
c_2\,
\mathbb{E}\Bigl\|
\nabla_{\boldsymbol{\lambda}} f(\boldsymbol{\phi}^{t-1},\boldsymbol{\gamma}^{t-1})
\Bigr\|_{\mathbf{H}}^2 .
\label{eq:blockwise_descent}
\end{align}
\end{lemma}

\begin{proof}
Start from the three inequalities:
\begin{align}
&(\epsilon_t)_j
\le
4\eta_j^2 L_j^2\!\left(
2(1+\kappa_\phi)\,G_{t-1}
+(1+2\kappa_\phi)\,\epsilon_t
\right)
+
4\eta_j^2\,
\mathbb{E}\!\left\|
(\nabla_{\boldsymbol{\lambda}} f(\boldsymbol{\phi}^{t-1},\boldsymbol{\gamma}^{t-1}))_j
\right\|^2,
\qquad \forall j\in[B],
\label{eq:epsj-bound}
\\
&G_t
\le
\rho\,G_{t-1}+\xi\,\epsilon_t,
\label{eq:G-rec-block}
\\
&\mathbb{E}\!\left[f(\boldsymbol{\phi}^t,\boldsymbol{\gamma}^t)\right]
-
\mathbb{E}\!\left[f(\boldsymbol{\phi}^{t-1},\boldsymbol{\gamma}^{t-1})\right]
\le
E\,\epsilon_t+F\,G_{t-1}
-\frac{1}{2}\,
\mathbb{E}\Bigl\|
\nabla_{\boldsymbol{\lambda}} f(\boldsymbol{\phi}^{t-1},\boldsymbol{\gamma}^{t-1})
\Bigr\|_{\mathbf{H}}^2.
\label{eq:f-descent-block}
\end{align}

Summing over $j$ and using $\eta_j^2\le\eta_{\max}\eta_j$ yields
\[
\epsilon_t
\le
p\,G_{t-1}
+
q\,\mathbb{E}\|\nabla_{\boldsymbol{\lambda}} f(\boldsymbol{\phi}^{t-1},\boldsymbol{\gamma}^{t-1})\|_{\mathbf H}^2 ,
\]
where $D>0$.
Multiply \eqref{eq:G-rec-block} by $c_1$ and add it to \eqref{eq:f-descent-block}:
\begin{align*}
&\mathbb{E}\!\left[f(\boldsymbol{\phi}^t,\boldsymbol{\gamma}^t)\right]+c_1G_t
-
\Bigl(\mathbb{E}\!\left[f(\boldsymbol{\phi}^{t-1},\boldsymbol{\gamma}^{t-1})\right]+c_1G_{t-1}\Bigr)
\le
(A+c_1\xi)\epsilon_t
+\bigl(B+c_1(\rho-1)\bigr)G_{t-1}
-\frac{1}{2}\,
\mathbb{E}\Bigl\|
\nabla_{\boldsymbol{\lambda}} f(\boldsymbol{\phi}^{t-1},\boldsymbol{\gamma}^{t-1})
\Bigr\|_{\mathbf{H}}^2.
\end{align*}
Substituting the bound on $\epsilon_t$ gives
\begin{align*}
&\mathbb{E}\!\left[f(\boldsymbol{\phi}^t,\boldsymbol{\gamma}^t)\right]+c_1G_t
-
\Bigl(\mathbb{E}\!\left[f(\boldsymbol{\phi}^{t-1},\boldsymbol{\gamma}^{t-1})\right]+c_1G_{t-1}\Bigr)
\\
\le
&\Bigl((A+c_1\xi)p+B+c_1(\rho-1)\Bigr)G_{t-1}
+\Bigl((A+c_1\xi)q-\frac{1}{2}\Bigr)
\mathbb{E}\Bigl\|
\nabla_{\boldsymbol{\lambda}} f(\boldsymbol{\phi}^{t-1},\boldsymbol{\gamma}^{t-1})
\Bigr\|_{\mathbf{H}}^2.
\end{align*}
With the choice
\(
c_1=-(Ap+B)/(\xi p+(\rho-1)),
\)
the coefficient of $G_{t-1}$ vanishes. Defining
\(
c_2 := \frac{1}{2}-(A+c_1\xi)q,
\)
we obtain \eqref{eq:blockwise_descent}.
Moreover, $D>0$ holds whenever
\begin{equation}\label{eq:Dpos_simple}
\sum_{j=1}^B \eta_j^2 L_j^2
<
\frac{1}{4(1+2\kappa_\phi)}.
\end{equation}
Since
\(
c_1=-(Ap+\mathcal{B})/(\xi p+(\rho-1)),
\)
a sufficient condition for $c_1>0$ is
\begin{equation}\label{eq:c1pos_simple}
\xi p<1-\rho
\;\;\Longleftrightarrow\;\;
8\xi(1+\kappa_\phi)\sum_{j=1}^B \eta_j^2 L_j^2
<
(1-\rho)\,D.
\end{equation}
Similarly, since
\(
c_2=\frac{1}{2}-(A+c_1\xi)q,
\)
a sufficient condition for $c_2>0$ is
\begin{equation}\label{eq:c2pos_simple}
(A+c_1\xi)q<\frac{1}{2}
\;\;\Longleftrightarrow\;\;
8(A+c_1\xi)\,\eta_{\max}<D.
\end{equation}
Therefore, any step size $\{\eta_j\}_{j=1}^B$ satisfying
\eqref{eq:Dpos_simple}--\eqref{eq:c2pos_simple}
guarantee $D>0$, $c_1>0$, and $c_2>0$.
\end{proof}


\section{Details in Numerical Analysis and Additional Numerical Results} \label{app:experiments} 
\subsection{Experiments on Spatial Transcriptomics Data} \label{app:experiments_STD}
\paragraph{Full hierarchical model.}
Let $\mathbf{x}:=\{\mathbf{x}_1,\ldots,\mathbf{x}_n\}$ denote the observed data, where each
$\mathbf{x}_i\in\mathbb{R}^{d_x}$ for all $i\in[n]$.
Let $\boldsymbol{\zeta}:=\{\zeta_1,\ldots,\zeta_n\}$ denote the latent cluster assignments, where
$\zeta_i\in[K]:=\{1,\ldots,K\}$ indicates the cluster membership of $\mathbf{x}_i$.

We assume that, conditioned on the assignment $\zeta_i=k$, the observation $\mathbf{x}_i$
follows a Gaussian distribution with cluster-specific mean $\mathbf{c}_k$ and precision matrix $\mathbf{\Sigma}_0$, i.e., $\mathbf{x}_i\mid(\zeta_i=k,\mathbf{c}_k,\mathbf{\Sigma}_0)\sim\mathcal{N}(\mathbf{c}_k,\mathbf{\Sigma}_0)$.We assume that each cluster mean $\mathbf{c}_k$ follows a Gaussian prior, with $\mathbf{\Sigma}_0$ treated as a fixed hyperparameter.

In addition to the expression measurements, we incorporate spatial information.
For each data point $i$, let $l_i\in\mathbb{R}^2$ denote its spatial coordinates, and
collect these into $L=(l_1,\ldots,l_n)$.
Moreover, let $g_i\in\mathbb{R}^2$ denote a precomputed flow direction associated with point $i$,
and collect these into $G=(g_1,\ldots,g_n)$.
Both $L$ and $G$ are treated as observed and fixed throughout inference.

We define a neighborhood graph in the spatial domain.
For each $i\in[n]$, let $\mathcal{N}(i)$ denote the neighbor set of node $i$, e.g., obtained
via $k$-nearest neighbors using the coordinates $\{l_i\}_{i=1}^n$.
We then define the (undirected) edge set
\[
E=\{(i,j): j\in\mathcal{N}(i)\ \text{and}\ i\in\mathcal{N}(j)\}.
\]

The resulting hierarchical model is given by
\begin{align}
\nonumber p(\mathbf{x}\mid \boldsymbol{\zeta},\mathbf{c})
&= \prod_{i=1}^n
\mathcal{N}\!\left(\mathbf{x}_i \,\middle|\, \mathbf{c}_{\zeta_i},\mathbf{\Sigma}_0\right), \\[4pt]
\nonumber \mathbf{c}_k 
&\sim
\mathcal{N}\!\left(\boldsymbol{\xi},\mathbf{\Sigma}_1\right),
\qquad \forall \, k=1,\ldots,K, \\[4pt]
p(\boldsymbol{\zeta}\mid L,G)
&\propto
\exp\!\left(
\sum_{(i,j)\in E} r_{ij}\,\mathbb{I}(\zeta_i=\zeta_j)
\right), \label{eq:potts_model}
\end{align}
where the edge weight $r_{ij}$ for $(i,j)\in E$ is defined as
\[
r_{ij}
=
\frac{\bigl|g_i^\top (l_j-l_i)\bigr|}{\|g_i\|_2\,\|l_j-l_i\|_2}
+
\tau\!\left(
\frac{\mathbf{x}_i^\top \mathbf{x}_j}{\|\mathbf{x}_i\|_2\,\|\mathbf{x}_j\|_2}
+1
\right),
\qquad j\in\mathcal{N}(i),
\]
with $\tau\ge 0$, $\boldsymbol{\xi}$, $\mathbf{\Sigma}_0$, $\mathbf{\Sigma}_1$ fixed as hyperparameters.

\paragraph{Variational parameters}
Define 
\begin{align*}
q(\boldsymbol{\zeta}, \mathbf{c})
=
\Bigl(\prod_{i=1}^n q(\zeta_i)\Bigr)
\Bigl(\prod_{k=1}^K q(\mathbf{c_k})\Bigr),
\end{align*}
where
\[
q(\zeta_i) = \text{Categorical}(\phi_{i1}, \ldots, \phi_{iK}),
\]
and
\[
q(c_{kj})
=
\mathcal{N}\!\bigl(c_{kj} \mid m_{kj},s_{kj}^2\bigr), \; \forall\; k \in [K], j \in [d_x]
\]
The variational parameters are
\(
\{\phi_{ik}\}_{i,k}, \{c_{kj}\}_{k,j}, \{s_{kj}\}_{k,j}
\)
which are optimized to maximize the ELBO.

\paragraph{Evidence lower bound} We seek variational parameters such that the distribution $q$ approximates the true posterior as closely as possible. Accordingly, we aim to minimize the Kullback--Leibler divergence
\[
\mathrm{KL}\!\left(
q(\boldsymbol{\zeta},\mathbf{c})
\,\|\, 
p(\boldsymbol{\zeta},\mathbf{c} \mid \mathbf{x}, L, G)
\right).
\]
For notational simplicity, we omit the conditioning on $(L,G)$ in the remainder of the paper. Minimizing this KL divergence is equivalent to maximizing the evidence lower bound (ELBO). The ELBO is given by
\begin{align*}
\text{ELBO}(q)
&=
\mathbb{E}_q\!\left[\log p(\mathbf{x}, \boldsymbol{\zeta}, \mathbf{c})\right]
-
\mathbb{E}_q\!\left[\log q(\boldsymbol{\zeta}, \mathbf{c})\right].
\end{align*}

By computation, we have
\begin{align*}
- \text{ELBO}(q) = 
\underbrace{\sum_{i,k}\phi_{ik}\frac12\sum_j\Big[\log(2\pi\sigma_{0,j}^2)+\frac{(x_{ij}-m_{kj})^2+s_{kj}^2}{\sigma_{0,j}^2}\Big]}_{\text{(A) likelihood}}
+
\underbrace{\sum_k\frac12\sum_j\Big[\log(2\pi\sigma_{1,j}^2)+\frac{(m_{kj}-\xi_{j})^2+s_{kj}^2}{\sigma_{1,j}^2}\Big]}_{\text{(B) prior on }c}\\
\quad
+
\underbrace{\Big(-\sum_{(i,j)\in E} r_{ij}\sum_k\phi_{ik}\phi_{jk}+\log Z\Big)}_{\text{(C) Potts term}}
+
\underbrace{\sum_{i,k}\phi_{ik}\log\phi_{ik}}_{\text{(D) }q(\zeta)}
+
\underbrace{\sum_{k,j}\Big(-\frac12\log(2\pi s_{kj}^2)-\frac12\Big)}_{\text{(E) }q(c)},
\end{align*}
where $\boldsymbol{\Sigma}_0=\mathrm{diag}(\sigma_{0,1}^2,\dots,\sigma_{0,d_x}^2),\quad
\boldsymbol{\Sigma}_1=\mathrm{diag}(\sigma_{1,1}^2,\dots,\sigma_{1,d_x}^2),$  $\boldsymbol{\xi}_{k}=(\xi_{1},\dots,\xi_{d_x})$, and $Z$ is the partition function (normalizing constant) associated with the Potts model.

Our objective is equivalent to a finite-sum function:
\begin{align} \label{eq:spatial_objective}
\nonumber \min_{\{\phi_{ik}\}_{i,k},\,\{m_{kj}\}_{k,j},\,\{s_{kj}\}_{k,j}}
\;&
\sum_{i=1}^n
\Bigg[
\sum_{k=1}^K
\phi_{ik}
\Bigg(
\log \phi_{ik}
+
\frac{1}{2}\sum_{j=1}^{d_x}
\frac{m_{kj}^2+s_{kj}^2 - 2x_{ij} m_{kj}}{\sigma_{0,j}^2}
\Bigg)
\Bigg] \\[2pt]
\nonumber &\qquad
+\frac{1}{2}\sum_{k=1}^K\sum_{j=1}^{d_x}
\Bigg(
\frac{m_{kj}^2+s_{kj}^2 - 2\xi_{j} m_{kj}}{\sigma_{1,j}^2}
-\log s_{kj}^2
\Bigg)
-\sum_{(i,j)\in E} r_{ij}\sum_{k=1}^K \phi_{ik}\phi_{jk} \\[4pt]
&\qquad
+
\underbrace{
\frac{1}{2}\sum_{i=1}^n \sum_{j=1}^{d_x}
\Bigl(
\log(2\pi\sigma_{0,j}^2)
+\frac{x_{ij}^2}{\sigma_{0,j}^2}
\Bigr)
+
\frac{1}{2}\sum_{k=1}^K\sum_{j=1}^{d_x}\bigl(\log\sigma_{1,j}^2+\frac{\xi_{j}^2}{\sigma_{1,j}^2} -1\bigr)
+\log Z
}_{\text{constant}} .
\end{align}
Notice that the optimization variables are subject to implicit constraints. 
In particular, for each data point $i$, the variational weights $\{\phi_{ik}\}_{k=1}^K$ lie on the probability simplex, i.e.,
$\sum_{k=1}^K \phi_{ik} = 1$ and $\phi_{ik} \ge 0$. 
Moreover, the scale parameters satisfy $s_{kj} > 0$ for all $k$ and $j$.

Admittedly, the negative ELBO is not fully decomposable due to the coupling terms induced by the Potts model in \eqref{eq:potts_model}.  However, to handle the large-scale problems, we partition the full dataset into a collection of small patches and restrict the neighborhood graph to within-patch connections, ignoring interactions across different patches.
Under this patch-wise approximation, the objective becomes batch-decomposable, and can therefore be rewritten as
\begin{align*}
\min_{\{\phi_{ik}\},\{m_{kj}\},\{s_{kj}\}}
\;&
\sum_{b=1}^B
\sum_{i\in S_b}
\Bigg[
\sum_{k=1}^K
\Bigg(
\phi_{ik}
\Bigg(
\log \phi_{ik}
+
\frac{1}{2}\sum_{j=1}^{d_x}
\frac{m_{kj}^2+s_{kj}^2 - 2x_{ij} m_{kj}}{\sigma_{0,j}^2}
\Bigg)
+
\frac{1}{2n}\sum_{j=1}^{d_x}
\Bigg(
\frac{m_{kj}^2+s_{kj}^2 - 2\xi_{j} m_{kj}}{\sigma_{1,j}^2}
-\log s_{kj}^2
\Bigg) \\
&\qquad
-
\frac{1}{2}\sum_{\ell \in \mathcal{N}(i)}
r_{i\ell}\phi_{ik}\phi_{\ell k} \Bigg)
\Bigg] .
\end{align*}
Therefore, for the spatial transcriptomics dataset, we divide the problem into approximately 30 batches. This choice reflects a trade-off between computational efficiency and modeling accuracy: using more batches increases the approximation error due to neglected interactions across patch boundaries, whereas using fewer batches raises the computational cost within each batch.

\paragraph{Augmented Lagrangian subproblem}
In both \name\ and \pname, an augmented Lagrangian (AL) subproblem is solved at each iteration. Below, we present the batch-wise closed-form expression of the corresponding AL objective.  
We define $\rho_k = \log s_k^2$ and $\phi_{ik} = \frac{\exp(\alpha_{ik})}{\sum_{j=1}^K \exp(\alpha_{ij})}$, to change the problem into an unconstrained optimization problem.  
According to Algorithm~\ref{alg:meta}, at each iteration we solve the following augmented Lagrangian problem:


\begin{align*}
\mathcal{L}_b (\alpha_{ik}, m_{kj}^b, \rho_{kj}^b) = & 
\sum_{i\in S_b}
\Bigg[
\sum_{k=1}^K
\Bigg(
\frac{\exp{(\alpha_{ik})}}{\sum_{j =1}^K\exp{(\alpha_{ij})}}
\Bigg(
\log \frac{\exp{(\alpha_{ik})}}{\sum_{j =1}^K\exp{(\alpha_{ij})}}
+
\frac{1}{2}\sum_{j=1}^{d_x}
\frac{(m_{kj}^b)^2+ \exp{(\rho_{kj}^b)} - 2x_{ij} m_{kj}^b}{\sigma_{0,j}^2}
\Bigg)
\\
&\qquad +
\frac{1}{2n}\sum_{j=1}^{d_x}
\Bigg(
\frac{(m_{kj}^b)^2+\exp(\rho_{kj}^b) - 2\xi_{j} m_{kj}^b}{\sigma_{1,j}^2}
-\rho_{kj}^b
\Bigg) 
-
\frac{1}{2}\sum_{\ell \in \mathcal{N}(i)}
r_{i\ell}
\frac{\exp{(\alpha_{ik})}}{\sum_{j =1}^K\exp{(\alpha_{ij})}}
\frac{\exp{(\alpha_{\ell k})}}{\sum_{j =1}^K\exp{(\alpha_{\ell j})}}
\Bigg) 
\Bigg]\\
& \qquad + \mu_{kj}^b (m_{kj}^b - m_{kj}^0) + \gamma_{kj}^b (\rho_{kj}^b - \rho_{kj}^0) + \frac{1}{2\eta} [(m_{kj}^b - m_{kj}^0)^2 + (\rho_{kj}^b - \rho_{kj}^0)^2.
\end{align*}

\paragraph{Data preprocessing} For the spatial transcriptomics dataset, we apply standard preprocessing steps to the raw gene expression counts prior to model training. Specifically, we first identify a subset of highly variable genes and retain the top 5{,}000 genes to reduce dimensionality and noise. 
Next, the gene expression counts are normalized to account for differences in sequencing depth across spatial locations, followed by a logarithmic transformation. Finally, the normalized data are scaled by their standard deviations and clipped to a reasonable range to prevent extreme values from dominating the optimization and to stabilize training. The resulting preprocessed dataset is used for all subsequent experiments.

\paragraph{Bayesian model settings}
The hyperparameters of the Bayesian model are chosen as follows. We set $\boldsymbol{\xi}$ to the empirical mean vector of the data. At initialization, we perform $k$-means clustering on the dataset and use the resulting cluster assignments as the initial labels. The average within-cluster variance is used to define $\boldsymbol{\Sigma}_0$, while the variance of the cluster means is used to define $\boldsymbol{\Sigma}_1$.

\paragraph{Experimental Settings}
When implementing the different algorithms, we carefully tune the learning rates and other hyperparameters. We first conduct a coarse-grained search over a broad range of values (e.g., $[10^{-6}, 10^{-4}, 10^{-2}, 1, 10^{2}, 10^{4}]$) to identify a reasonable operating region, and then refine the hyperparameters through a finer-grained search within this region.

\subsection{Experiments on Synthetic Data} \label{app:experiment_syn}
For the synthetic experiments, we generate data from a Gaussian mixture model (GMM) with $K=5$ clusters in $d=10$ dimensions and equal mixture weights. The component parameters are first sampled from a Gaussian prior, after which data points are generated from the resulting mixture distribution.

The GMM model is
\begin{align*}
&\nonumber p(\mathbf{x}\mid \boldsymbol{\zeta},\mathbf{c})
= \prod_{i=1}^n
\mathcal{N}\!\left(\mathbf{x}_i \,\middle|\, \mathbf{c}_{\zeta_i},\mathbf{\Sigma}_0\right), \\[4pt]
&\nonumber \mathbf{c}_k 
\sim
\mathcal{N}\!\left(\boldsymbol{\xi},\mathbf{\Sigma}_1\right),
\qquad \forall \, k=1,\ldots,K, \\[4pt]
&\zeta_i \sim \text{Categorical}(1/K, \ldots, 1/K), \quad \forall \; i = 1, \ldots, n.
\label{eq:potts_model}
\end{align*}
Here $\boldsymbol{\xi}$, $\mathbf{\Sigma}_0$, and $\mathbf{\Sigma}_1$ are set in an empirical-Bayes manner:
$\boldsymbol{\xi}$ is the empirical mean of the data,
$\mathbf{\Sigma}_0$ is determined by the average within-cluster variance from an initial $k$-means clustering,
and $\mathbf{\Sigma}_1$ is determined by the variance of the cluster means.

\paragraph{Experimental Settings}
In this experiment, we use 100{,}000 data points and adopt a mini-batch training scheme with a batch size of 1{,}000. 
The mini-batches are intentionally constructed to be biased rather than uniformly sampled: each mini-batch predominantly contains samples from a single cluster, resulting in non-i.i.d.\ and cluster-imbalanced batches. 
This setting is designed to stress-test the stability of the proposed algorithm under heterogeneous and biased sampling conditions.

\subsection{Experiments on Quadratic Synthetic Data} \label{app:exp_quadratic}
To evaluate convergence under strong convexity, we compare \name, \pname, and the baselines on the following quadratic program:
\[
\begin{aligned}
\min_{\{\boldsymbol{\phi}_i,\boldsymbol{\lambda}_i\}_{i=1}^n, \boldsymbol{\lambda}_0}\;
&\frac{1}{n}\sum_{i=1}^n 
\begin{pmatrix}
\boldsymbol{\phi}_i^\top ,
\boldsymbol{\lambda}_i^\top
\end{pmatrix}
\boldsymbol{Q}_i
\begin{pmatrix}
\boldsymbol{\phi}_i \\
\boldsymbol{\lambda}_i
\end{pmatrix}
\;+\;
\boldsymbol{v}_i^{\top}
\begin{pmatrix}
\boldsymbol{\phi}_i \\
\boldsymbol{\lambda}_i
\end{pmatrix}
\\
\text{s.t.}\quad
&\boldsymbol{\lambda}_i = \boldsymbol{\lambda}_0
\end{aligned}
\]
where $\boldsymbol{Q}_i \in \mathbb{R}^{10 \times 10}, \boldsymbol{Q}_i \in \mathbb{R}^{++}_{10}$, and $\boldsymbol{v}_i \in \mathbb{R}^{10}$ are the local information for each $i$. The condition number for each $\boldsymbol{Q}_i$ is 1000. We set $n=10{,}000$, with decision variables $\boldsymbol{\lambda}_i \in \mathbb{R}^{5} $ and $\boldsymbol{\phi}_i \in \mathbb{R}^{5}$. In our experiments, we use $\boldsymbol{v}_i=\mathbf{0}$, so the optimal objective value is zero, which makes convergence straightforward to assess.

\end{document}